\documentclass{article}

\usepackage{microtype}
\usepackage{graphicx}
\usepackage{subfigure}
\usepackage{booktabs} % for professional tables

\usepackage{hyperref}

\usepackage[accepted]{icml2025}

\usepackage{amsmath}
\usepackage{amssymb}
\usepackage{mathtools}
\usepackage{amsthm}

\usepackage{enumerate}
\usepackage{url}
\usepackage{comment}
\usepackage{colortbl}
\usepackage{tablefootnote}
\usepackage{makecell}
\usepackage{pifont}
\usepackage{bbding}

\usepackage[capitalize,noabbrev]{cleveref}

\theoremstyle{plain}
\newtheorem{theorem}{Theorem}[section]
\newtheorem{proposition}[theorem]{Proposition}
\newtheorem{lemma}[theorem]{Lemma}

\theoremstyle{definition}
\newtheorem{definition}[theorem]{Definition}
\newtheorem{assumption}[theorem]{Assumption}
\theoremstyle{remark}

\usepackage{amsmath,amsfonts,bm}

\def\1{\bm{1}}

\def\gdphi{{\nabla\Phi}}
\def\hatphi{{\hat{\nabla}\phi}}
\def\gdx{{\nabla_{x}}}
\def\gdy{{\nabla_{y}}}
\def\gdxy{{\nabla_{xx}^2}}
\def\gdxy{{\nabla_{xy}^2}}
\def\gdyy{{\nabla_{yy}^2}}

\def\gdf{{\bar{\nabla}f}}

\def\hm{{\hat{m}}}
\def\hv{{\hat{v}}}
\def\hu{{\hat{u}}}

\def\hl{{\hat{L}}}

\def\tx{{\tilde{x}}}
\def\ty{{\tilde{y}}}

\def\tm{{\tilde{m}}}
\def\tu{{\tilde{u}}}

\def\tg{{\tilde{G}}}
\def\tv{{\tilde{V}}}
\def\tl{{\tilde{L}}}

\DeclareMathOperator{\sq}{sq}
\newcommand{\alphasq}{\alpha^{\sq}}
\newcommand{\alphasqt}{\alpha_t^{\sq}}
\newcommand{\betasq}{\beta_{\sq}}

\def\pr{{\mathrm{Pr}}}

\def\init{{\text{init}}}
\def\diag{{\text{diag}}}

\def\vw{{\bm{w}}}
\def\vx{{\bm{x}}}

\def\vz{{\bm{z}}}

\DeclareMathAlphabet{\mathsfit}{\encodingdefault}{\sfdefault}{m}{sl}
\SetMathAlphabet{\mathsfit}{bold}{\encodingdefault}{\sfdefault}{bx}{n}

\def\gD{{\mathcal{D}}}
\def\gE{{\mathcal{E}}}
\def\gF{{\mathcal{F}}}

\def\gH{{\mathcal{H}}}

\newcommand{\E}{\mathbb{E}}

\newcommand{\R}{\mathbb{R}}

\DeclareMathOperator*{\argmin}{arg\,min}

\crefname{assumption}{Assumption}{Assumptions}

\newcommand{\crefdefpart}[2]{%
  \hyperref[#2]{\namecref{#1}~\labelcref*{#1}(\ref*{#2})}%
}

\newcommand{\refdefpart}[2]{%
  \hyperref[#2]{\labelcref*{#1}(\ref*{#2})}%
}

\newcommand{\halfcheck}{\ding{51}\textsuperscript{\kern-0.55em\ding{55}}}

\usepackage[textsize=tiny]{todonotes}

\icmltitlerunning{On the Convergence of Adam-Type Algorithm for Bilevel Optimization under Unbounded Smoothness}

\begin{document}

\twocolumn[
\icmltitle{On the Convergence of Adam-Type Algorithm for Bilevel Optimization \\ under Unbounded Smoothness}



\icmlsetsymbol{equal}{*}

\begin{icmlauthorlist}
\icmlauthor{Xiaochuan Gong}{yyy}
\icmlauthor{Jie Hao}{yyy}
\icmlauthor{Mingrui Liu}{yyy}
\end{icmlauthorlist}

\icmlaffiliation{yyy}{Department of Computer Science, George Mason University}

\icmlcorrespondingauthor{Mingrui Liu}{mingruil@gmu.edu}

\icmlkeywords{Machine Learning, ICML}

\vskip 0.3in
]



\printAffiliationsAndNotice{}  % leave blank if no need to mention equal contribution

\begin{abstract}

Adam has become one of the most popular optimizers for training modern deep neural networks, such as transformers. However, its applicability is largely restricted to single-level optimization problems. In this paper, we aim to extend vanilla Adam to tackle bilevel optimization problems, which have important applications in machine learning, such as meta-learning. In particular, we study stochastic bilevel optimization problems where the lower-level function is strongly convex and the upper-level objective is nonconvex with potentially unbounded smoothness. This unbounded smooth objective function covers a broad class of neural networks, including transformers, which may exhibit non-Lipschitz gradients. In this work, we introduce AdamBO, a single-loop Adam-type method that achieves $\widetilde{O}(\epsilon^{-4})$ oracle complexity to find $\epsilon$-stationary points, where the oracle calls involve stochastic gradient or Hessian/Jacobian-vector product evaluations. The key to our analysis is a novel randomness decoupling lemma that provides refined control over the lower-level variable. We conduct extensive experiments on various machine learning tasks involving bilevel formulations with recurrent neural networks (RNNs) and transformers, demonstrating the effectiveness of our proposed Adam-type algorithm.

\end{abstract}

\section{Introduction}

The Adam algorithm~\citep{kingma2014adam} is one of the most popular optimizers for training modern deep neural networks due to their computational efficiency and minimal need for hyperparameter tuning. For example, Adam has become the default choice for training transformers~\citep{vaswani2017attention,devlin2018bert} and vision transformers (ViT)~\citep{dosovitskiy2021an}. Practitioners favor Adam and adaptive gradient methods in general because they significantly outperform stochastic gradient descent (SGD) for certain models, such as transformers~\citep{zhang2019adaptive,crawshaw2022robustness,kunstner2023noise,ahn2023linear}. Recently, there is a line of work analyzing the convergence of Adam under various assumptions~\citep{guo2021novel,defossez2020simple,wang2022provable,zhang2022adam,li2023convergence}.

Despite the empirical and theoretical advances of Adam, it is only applicable for single-level optimization problems such as the empirical risk minimization. However, there is a huge class of machine learning problems which are inherently bilevel optimization problems~\citep{bracken1973mathematical,dempe2002foundations}, including meta-learning~\citep{franceschi2018bilevel,rajeswaran2019meta}, reinforcement learning~\citep{konda2000actor}, hyperparameter optimization~\citep{franceschi2018bilevel,feurer2019hyperparameter} and continual learning~\citep{borsos2020coresets,hao2023bilevelcoreset}. Therefore, an important question arises: \textbf{How can we extend the applicability of vanilla Adam to solve bilevel optimization problems, while ensuring both provable theoretical convergence guarantees and strong empirical performance for machine learning applications?} 

In this paper, we provide a positive answer to this question, under the setting of bilevel optimization under unbounded smoothness~\citep{hao2024bilevel,gong2024a}. In particular, the bilevel optimization in this setting has the following form:
% \begin{equation}\label{eq:eq1}
% \min_{x\in\mathbb{R}^{d_x}} \Phi(x):=f(x,y^*(x)), \ \ \text{  s.t.,  } \ \  y^*(x) \in \mathop{\arg \min}_{y\in \mathbb{R}^{d_y}} g(x, y),
% \end{equation}
\begin{equation} \label{eq:eq1}
    \begin{aligned}
        &\min_{x\in\mathbb{R}^{d_x}} \Phi(x):=f(x,y^*(x)), \\
        &\text{s.t.,} \ \  y^*(x) \in \mathop{\arg \min}_{y\in \mathbb{R}^{d_y}} g(x, y),
    \end{aligned}
\end{equation}
where $f$ and $g$ are upper- and lower-level functions respectively, and $f$ satisfies a unbounded smoothness condition (see \cref{ass:relax-smooth}) and $g$ is a strongly-convex function in $y$. One example satisfying this particular setting is meta-learning~\citep{finn2017model,franceschi2018bilevel} with certain machine learning models such as RNNs~\citep{elman1990finding} or transformers~\citep{vaswani2017attention}, where $x$ represents all layers except for the prediction head, $y$ represents the prediction head, and the goal is to learn the shared model parameter $x$ to find a common representation such that it can quickly adapt to various tasks by simply adjusting the task-specific prediction head $y$. The unbounded smoothness condition for the upper-level function $f$ is particularly relevant in this paper for two main reasons. First, recent studies have demonstrated that the gradient's Lipschitz constant (i.e., the smoothness constant) is unbounded in various modern neural networks, including RNNs and transformers~\citep{zhang2019gradient,crawshaw2022robustness,hao2024bilevel}. Second, Adam is empirically successful on training these neural networks~\citep{vaswani2017attention,kunstner2023noise} and its convergence under unbounded smoothness was recently proved within the single-level optimization framework~\citep{li2023convergence}. Therefore it is natural and imperative to design new Adam-type algorithms, building on the vanilla Adam approach, to solve bilevel optimization problems in the unbounded smoothness setting. 

We introduce an Adam-type algorithm for such bilevel optimization problems with provable convergence guarantees. Our algorithm is called Adam for Bilevel Optimization (AdamBO). AdamBO begins by running a few iterations of SGD to warm-start the lower-level variable, after which it simultaneously applies vanilla Adam updates to the upper-level variable and SGD updates to the lower-level variable.  The primary challenge for the convergence analysis of AdamBO is tackling the complicated dependency between the upper-level hypergradient bias and the lower-level estimation error when the upper-level performs the vanilla Adam update. The convergence analysis of AdamBO for unbounded smooth upper-level functions builds upon the insight of regarding bilevel optimization as a stochastic optimization problem under distributional drift~\citep{gong2024a}, but with a few important differences. First, our analysis incorporates a novel randomness decoupling lemma for lower-level error control, which arises from using Adam updates for the upper-level variable. Second,  unlike~\citep{hao2024bilevel,gong2024a}, the lower-level error in our setting is not necessarily small across iterations, requiring a more refined analysis to handle the hypergradient bias and establish convergence guarantees. 
% In addition, we also introduce another Adam-type algorithm, namely VR-AdamBO, by incorporating the variance reduction techniques~\citep{cutkosky2019momentum} along with Adam for the upper-level variable and the lower-level acceleration techniques~\citep{gong2024accelerated} to further improve the convergence rate. The analysis for VR-AdamBO relies on a novel stopping-time analysis in the context of bilevel optimization and a careful treatment for the lower-level error, which is different from the techniques used in single-level variance-reduced Adam~\citep{li2023convergence}. 
Our main contributions are summarized as follows.

\begin{itemize}

\item We design a variant of Adam, called AdamBO, for solving bilevel optimization problems under the unbounded smoothness setting. We prove that AdamBO converges to $\epsilon$-stationary points with $\widetilde{O}(\epsilon^{-4})$ oracle complexity. 
% To achieve this result, we develop a novel randomness decoupling lemma for lower-level error control and a refined analysis for the hypergradient bias, which are of independent interest and could be applied to analyzing the convergence of other adaptive optimizers in bilevel optimization.

% \item We propose a variance-reduced variant of AdamBO, named VR-AdamBO, with an improved oracle complexity of $\widetilde{O}(\epsilon^{-3})$. The proof relies on a novel stopping time analysis in the context of bilevel optimization and a careful treatment for the lower-level error.

\item We develop a novel randomness decoupling lemma for lower-level error control and a refined analysis for the hypergradient bias, which are of independent interest and could be applied to analyzing the convergence of other adaptive optimizers in bilevel optimization.

\item We conduct experiments on meta-learning and deep AUC maximization for text classification tasks with RNNs and transformers to verify the effectiveness of the proposed Adam-type algorithms. We show that AdamBO consistently outperforms other bilevel algorithms during the training process. Notably, for the transformer model, they improve the training (testing) AUC by at least 14\% (7\%) over other baselines. The running time results indicate that our algorithms converge much faster than baselines.

\end{itemize}

\section{Related Work}

\textbf{Convergence Analysis of Adam.} Adam was proposed by~\citep{kingma2014adam} and the convergence guarantee was established under the framework of online convex optimization.~\cite{reddi2019convergence} identified a divergence example of Adam under fixed hyperparameters and designed new variants to fix the divergence issue of Adam. Recently, there is a line of work analyzing the convergence of Adam under various assumptions and problem-dependent hyperparameter choices~\citep{zhou2018adashift,guo2021novel,defossez2020simple,wang2022provable,zhang2022adam,li2023convergence}. The most related work to our paper is~\citep{li2023convergence}, which studied the convergence of Adam under relaxed assumptions (i.e., generalized smoothness as defined by~\citep{li2023convergence}). However, all of these works only consider Adam within the single-level optimization framework and are not applicable for bilevel optimization problems.

\textbf{Bilevel Optimization.} Bilevel optimization was extensively studied in the literature, most of which focus on asymptotic convergence guarantees~\citep{bracken1973mathematical,vicente1994descent,anandalingam1990solution,white1993penalty}.~\cite{ghadimi2018approximation} studied bilevel optimization algorithms with non-asymptotic convergence guarantees when the lower-level function is strongly convex. The complexity results were later improved by a series of work~\citep{hong2023two,ji2021bilevel,chen2021single,dagreou2022framework,kwon2023fully,chen2023optimal}. When each realization of the functions has a Lipschitz stochastic gradient, several works incorporate momentum-based variance reduction techniques~\citep{cutkosky2019momentum} to further improve the convergence rate~\citep{khanduri2021near,guo2021randomized,yang2021provably}. Recently,~\citep{hao2024bilevel,gong2024a,gong2024accelerated} considered bilevel optimization with unbounded smoothness for the upper-level function and designed stochastic algorithms with convergence guarantees. However, none of these works use the Adam update under the bilevel optimization setting.

\textbf{Relaxed Smoothness.}~\cite{zhang2019gradient} initiated the convergence analysis of the gradient clipping algorithms under the relaxed smoothness condition, which was motivated by the loss landscape of RNNs and LSTMs. The work of~\citep{zhang2019gradient} inspired a line of work focusing on designing various algorithms under the relaxed smoothness condition~\citep{zhang2020improved,jin2021non,liu2022communication,crawshaw2023episode,crawshaw2023federated,faw2023beyond,wang2023convergence,li2023convergence,li2023convex}, some of them achieved improved convergence rates~\citep{liu2023near,reisizadeh2023variance,li2023convergence}. Several variants of relaxed smoothness were considered in~\citep{crawshaw2022robustness,chen2023generalized,hao2024bilevel,gong2024a,gong2024accelerated}. This work considered the same problem setting as in~\citep{hao2024bilevel,gong2024a,gong2024accelerated}, focusing on designing Adam-type algorithms for bilevel optimization with unbounded smooth upper-level functions.

\section{Preliminaries, Notations and Problem Setup}
\vspace*{-0.05in}

Denote $\langle\cdot,\cdot\rangle$ and $\|\cdot\|$ as the inner product and Euclidean norm of a vector or spectral norm of a matrix. For any vectors $x$ and $y$, denote $x^2,\sqrt{x}, |x|, x\odot y, x/y$ as the coordinate-wise square, square root, absolute value, product and quotient, respectively.
We write $x\preceq y$ to denote the coordinate-wise inequality between $x$ and $y$. We use $\widetilde{O}(\cdot)$, $\widetilde{\Theta}(\cdot)$, $\widetilde{\Omega}(\cdot)$ to denote asymptotic notations that hide polylogarithmic factors of $1/\epsilon$. Define $f, g:\R^{d_x}\times\R^{d_y}\rightarrow \R$ as the upper- and lower-level functions, where $f(x,y)=\E_{\xi\sim\mathcal{D}_f}[F(x,y;\xi)]$ and $g(x,y)=\E_{\zeta\sim\mathcal{D}_g}[G(x,y;\zeta)]$, with $\gD_f$ and $\gD_g$ being the underlying data distributions, respectively. 
When the lower-level function is strongly convex, the hypergradient has the following form~\citep{ghadimi2018approximation}:
%\begin{small}
\begin{equation*}
    \begin{aligned}
        &\gdphi(x) = \gdx f(x,y^*(x)) \\
        &\quad- \gdxy g(x,y^*(x))[\nabla_{yy}^{2} g(x,y^*(x))]^{-1}\gdy f(x,y^*(x)).
    \end{aligned}
\end{equation*}
%\end{small}%
The goal of this paper is to design Adam-type algorithms that can find $\epsilon$-stationary points of function $\Phi$ (i.e., finding an $x$ such that $\|\gdphi(x)\|\leq \epsilon$). For a given $(x,y)$, we estimate the hypergradient $\gdphi(x)$ using Neumann series approach~\citep{ghadimi2018approximation} with the following formulation:
% \begin{small}
\begin{equation*}
\label{eq:neumannseries}
    \begin{aligned}
        &\hatphi(x,y;\Bar{\xi}) = \gdx F(x,y;\xi) - \gdxy G(x,y;\zeta^{(0)}) \\
        &\left[\frac{1}{l_{g,1}}\sum_{q=0}^{Q-1}\prod_{j=1}^{q}\left(I - \frac{\gdyy G(x,y;\zeta^{(q,j)})}{l_{g,1}}\right)\right]\gdy F(x,y;\xi),
    \end{aligned}
\end{equation*}
% \end{small}%
where $ \Bar{\xi}\coloneqq \{\xi, \zeta^{(0)}, \Bar{\zeta}^{(0)}, \dots, \Bar{\zeta}^{(Q-1)}\}$ and $\Bar{\zeta}^{(q)}\coloneqq \{\zeta^{(q,1)}, \dots, \zeta^{(q,q)}\}$ for $q\geq0$.

% Now we start to state the main assumptions for our analysis.

\begin{definition}[{$(L_{x,0}, L_{x,1}, L_{y,0}, L_{y,1})$-Smoothness~\citep[Assumption 1]{hao2024bilevel}}] \label{ass:relax-smooth}
Let $z=(x,y)$ and $z'=(x',y')$, there exists $L_{x,0}, L_{x,1}, L_{y,0}, L_{y,1} > 0$ such that for all $z,z'$, if $\|z-z'\| \leq 1/\sqrt{L_{x,1}^2+L_{y,1}^2}$, then $ \|\gdx f(z)-\gdx f(z')\| \leq (L_{x,0}+L_{x,1}\|\gdx f(z)\|)\|z-z'\|$ and $\|\gdy f(z)-\gdy f(z')\| \leq (L_{y,0}+L_{y,1}\|\gdy f(z)\|)\|z-z'\|$.
\end{definition}

\textbf{Remark}: This definition characterizes the unbounded smoothness of the upper-level function $f$ and has also been used in previous works~\citep{hao2024bilevel,gong2024a,gong2024accelerated}. It can be regarded as a generalization of the relaxed smooth assumption in~\citep{zhang2019gradient} and the coordinate-wise relaxed smoothness assumption in~\citep{crawshaw2022robustness}. Moreover, it has been empirically verified for bilevel formulations with RNNs~\citep{hao2024bilevel}.

\begin{assumption} \label{ass:bilevel-assumption}
% Suppose the followings hold for functions $f$ and $g$: 
Suppose functions $f$ and $g$ satisfy: 
(i) $f$ is continuously differentiable and $(L_{x,0}, L_{x,1}, L_{y,0}, L_{y,1})$-smooth in $(x, y)$; (ii) For every $x$, $\|\gdy f(x,y^*(x))\| \leq l_{f,0}$; (iii) For every $x$, $g(x, y)$ is $\mu$-strongly convex in $y$ for $\mu>0$; (iv) $g$ is continuously differentiable and $l_{g,1}$-smooth jointly in $(x, y)$; (v) $g$ is twice continuously differentiable, and $\gdxy g, \gdyy g$ are $l_{g,2}$-Lipschitz jointly in $(x, y)$; (vi) Objective function $\Phi$ is bounded from below by $\Phi^*$.
\end{assumption}

\textbf{Remark:} \cref{ass:bilevel-assumption} is standard in the bilevel optimization literature~\citep{kwon2023fully,ghadimi2018approximation,hao2024bilevel}. Under this assumption, the objective function $\Phi$ is $(L_0,L_1)$-smooth, see \cref{lm:Phi-relax-smooth} in \cref{app:tech_lemmas} for definitions of $L_0, L_1$ and more details.

\begin{assumption} \label{ass:noise}
Suppose the following stochastic estimators are unbiased and satisfy: (i) $\|\gdx F(x,y;\xi) - \gdx f(x,y)\| \leq \sigma_f$; (ii) $\|\gdy F(x,y;\xi) - \gdy f(x,y)\| \leq \sigma_f$; (iii) $\pr(\|\gdy G(x,y;\zeta) - \gdy g(x,y)\| \geq s) \leq 2\exp(-2s^2/\sigma_{g,1}^2)$; (iv) $\|\gdxy G(x,y;\zeta) - \gdxy g(x,y)\| \leq \sigma_{g,2}$; (v) $\|\gdyy G(x,y;\zeta) - \gdyy g(x,y)\| \leq \sigma_{g,2}$.
\end{assumption}

\textbf{Remark:} \cref{ass:noise} assumes the noise in the stochastic gradient and Hessian/Jacobian is almost-surely bounded or light-tailed. This is an standard assumption in the literature of optimization for single-level relaxed smooth functions~\citep{zhang2019gradient,zhang2020improved}, as well as for bilevel optimization under unbounded smooth upper-level functions~\citep{hao2024bilevel,gong2024a,gong2024accelerated}.

\begin{assumption} \label{ass:bilevel-additional}
(i) Let $z=(x,y)$ and $z'=(x',y')$, if $\|z-z'\| \leq 1/\sqrt{L_{x,1}^2+L_{y,1}^2}$, then for every $\xi$, $\|\gdy F(z;\xi)-\gdy F(z';\xi)\| \leq (L_{y,0}+L_{y,1}\|\gdy f(z)\|)\|z-z'\|$; (ii) For every $\zeta$, $G(x,y;\zeta)$ satisfy \Cref{ass:bilevel-assumption} (iv) and (v).
\end{assumption}

\textbf{Remark}: \cref{ass:bilevel-additional} (i) requires that certain properties of the second argument (i.e., the lower-level variable $y$) in the upper-level function at the population level also hold almost surely for each random realization. \cref{ass:bilevel-additional} (ii) requires each random realization of the lower-level function satisfies the same property as in the population level. Similar assumptions were made implicitly in the bilevel optimization literature~\citep{ghadimi2018approximation}. Note that this assumption does not assume any properties in terms of the upper-level variable $x$ under each random realization.

\section{AdamBO and Convergence Analysis}

% \subsection{Algorithm Design, Main Challenges, and Technique Overview}
\subsection{Algorithm Design and Technique Overview}

\textbf{Algorithm Design.} Our proposed Adam-type algorithm AdamBO is presented in \cref{alg:bi-adam}. It consists of the following components. First, the algorithm requires several warm-start steps for updating the lower-level variable $y$ for a given initialization of the upper-level variable $x_0$ (line 2), which is designed to obtain a good estimate of the optimal lower-level variable at the very beginning and shares the same spirit of the bilevel algorithms introduced in~\citep{hao2024bilevel,gong2024a,gong2024accelerated}. Second, the algorithm updates both the upper- and lower-level variables simultaneously: the lower-level variable $y$ is updated by SGD, and the upper-level variable $x$ is updated by the vanilla Adam algorithm (lines $3\sim 9$). Therefore, the upper-level update benefits from the coordinate-wise adaptive learning rate. In contrast, the existing bilevel optimization algorithms under the unbounded smoothness setting use normalized SGD with momentum to update the upper-level variable~\citep{hao2024bilevel,gong2024a,gong2024accelerated}, which use a universal learning rate for every coordinate.

\textbf{Main Challenges.} The main challenges for the convergence analysis of AdamBO are listed as follows. First, the analysis of vanilla Adam in the single-level generalized smooth optimization setting~\citep{li2023convergence} is not directly applicable for bilevel problems. This is because the hypergradient estimator in bilevel optimization may have a non-negligible bias due to inaccurate estimation of the lower-level variable, whereas the single-level analysis in~\citep{li2023convergence} does not need to account for this issue. Second, the existing algorithms and analyses for bilevel optimization with unbounded smooth upper-level functions require the lower-level error to be small~\citep{hao2024bilevel,gong2024a,gong2024accelerated}, which may not hold for AdamBO. In particular, the existing analysis crucially relies on a fixed update length for the upper-level variable at every iteration (due to normalization): the analysis in~\citep{hao2024bilevel,gong2024a,gong2024accelerated} views the update of the upper-level variable as a fixed distributional drift for the lower-level function, which is crucial to show that the lower-level error is small and the hypergradient bias is negligible. However, such an argument is not true for AdamBO: the Adam update for the lower-level variable does not have a fixed update size and it depends on randomness from both upper-level and lower-level random variables in the stochastic setting, which make the lower-level error control more challenging.

\textbf{Technique Overview.} To address these challenges, one of our main technical contributions is the introduction of a novel randomness decoupling lemma for controlling the lower-level error when the upper-level variable is updated by Adam, as illustrated in \cref{sec:random-decoupling}. This lemma provide a high probability guarantee for the lower-level error control when the upper-level update rule satisfies certain conditions (which are satisfied by the vanilla Adam update rule for the upper-level variable). The key novelty of this lemma lies in the randomness-decoupling fact: 
the high-probability bound depends solely on the randomness $\{\zeta_t\}_{t=1}^{T}$ from the lower-level random variables, and it holds for any fixed sequence of upper-level variables $\{x_t\}_{t=1}^{T}$ and any fixed upper-level random variables $\{\bar{\xi}_t\}_{t=1}^{T}$ that respect the Adam updates. To describe the condition that Adam satisfies and to prove this lemma, we introduce an auxiliary sequence (defined in~\eqref{eq:virtualsequence}) that separates the randomness in the upper- and lower-level random variables, which is new and has not been leveraged in previous bilevel optimization literature.

\begin{algorithm}[t]
    \caption{\textsc{AdamBO}}\label{alg:bi-adam}
    \begin{algorithmic}[1]
        \STATE \textbf{Input:} $\beta,\betasq,\eta,\gamma,\lambda, T_0, T, x_1, y_0$
        \STATE \textbf{Initialize} $y_1=\texttt{SGD}(x_1, y_0, \gamma, T_0)$, $\hm_1=\hatphi(x_1,y_1;\Bar{\xi}_1)$ and $\hv_1=(\hatphi(x_1,y_1;\Bar{\xi}_1))^2$
	\FOR{$t=1,\dots,T$}
            % \STATE Draw a new sample $\xi_t$ and perform the following updates
            % \STATE $\alpha_t = \frac{\beta}{1-(1-\beta)^t}$, $\alphasqt = \frac{\betasq}{1-(1-\betasq)^t}$
            % \STATE Draw new samples and perform the following updates
            \STATE $y_{t+1} = y_t - \gamma\gdy G(x_t,y_t;\zeta_t)$
            % \STATE $h_t = (1-\gamma)h_{t-1} + \gamma \hatphi(x_t)$
    	\STATE $m_t = (1-\beta)m_{t-1} + \beta \hatphi(x_t,y_t;\Bar{\xi}_t)$
    	\STATE $v_t=(1-\betasq)v_{t-1} + \betasq (\hatphi(x_t,y_t;\Bar{\xi}_t))^2$
            \STATE $\hm_t = \frac{m_t}{1-(1-\beta)^t}$
            \STATE $\hv_t = \frac{v_t}{1-(1-\betasq)^t}$
    	\STATE $x_{t+1} = x_t-\frac{\eta}{\sqrt{\hv_t}+\lambda}\odot \hm_t$
	\ENDFOR
    \end{algorithmic}
\end{algorithm}

\subsection{Main Results}

We first introduce some notations and technical definitions. Denote $\sigma(\cdot)$ as the $\sigma$-algebra generated by the random variables within the argument. Let $\gF_{\init}$ be the filtration for updating $y_1$ (see \cref{alg:sgd}): $\gF_{\init} = \sigma(\pi_0,\dots,\pi_{T_0-1})$. For any $t \geq 2$, define $\gF_t^x, \gF_t^y$ and $\gF_t$ as $\gF_t^x = \sigma(\Bar{\xi}_1,\dots,\Bar{\xi}_{t-1})$, $\gF_t^y = \sigma(\zeta_1,\dots,\zeta_{t-1})$ and $\gF_t = \sigma(\gF_{\init} \cup \gF_t^x \cup \gF_t^y)$. We use $\E_t[\cdot]$ to denote the conditional expectation $\E[\cdot \mid \gF_t]$. We also use $c_1,c_2,c_3$ to denote small enough constants and $C_1,C_2$ to denote large enough constants, all of which are independent of $\epsilon$ and $\delta$, where $\epsilon$ denotes the target gradient norm and $\delta$ denotes the failure probability. The definitions of problem-dependent constants $\sigma_{\phi}, C_{\phi,0}, C_{\phi,1}, \Delta_1, L_0, L_1, L, C_{\beta}$ are comprehensively listed in \cref{sec:adambo-notations}.

\begin{theorem} \label{thm:main}
Suppose \cref{ass:bilevel-assumption,ass:noise,ass:bilevel-additional} hold. Let $G$ be a constant satisfying $G \geq \max\left\{4\lambda, 2\sigma_{\phi}, 4C_{\phi,0}, \frac{C_{\phi,1}}{L_1}, \sqrt{\frac{C_1\Delta_1L_0}{C_L}}, \frac{C_1\Delta_1L_1}{C_L}\right\}$. Given any $\epsilon>0$ and $\delta\in(0,1)$, choose $0\leq \betasq \leq 1$, $\beta = \widetilde{\Theta}(\epsilon^2)$, $\gamma = \widetilde{\Theta}(\epsilon^2)$, $\eta = \widetilde{\Theta}(\epsilon^2)$, $Q = \widetilde{\Theta}(1)$, $T_0 = \widetilde{\Theta}(\epsilon^{-2})$. Run \cref{alg:bi-adam} for $T=\max\left\{\frac{1}{\beta^2}, \frac{C_2\Delta_1G}{\eta\epsilon^2}\right\}=\widetilde{O}(\epsilon^{-4})$ iterations. Then with probability at least $1-\delta$ over the randomness in $\gF_{T+1}$, we have $\frac{1}{T}\sum_{t=1}^{T}\|\gdphi(x_t)\|\leq \epsilon^2$.
% we have $\|\gdphi(x_t)\|\leq G$ for all $t\in[T]$, and $\frac{1}{T}\sum_{t=1}^{T}\|\gdphi(x_t)\|\leq \epsilon^2$.
\end{theorem}

\textbf{Remark 1}: The full statement of \cref{thm:main} with detailed parameter choices is deferred to \cref{thm:main-appendix} in \cref{sec:proof-thm}. \cref{thm:main} provides the convergence guarantee for \cref{alg:bi-adam}: AdamBO converges to $\epsilon$-stationary points with $T_0+QT=\widetilde{O}(\epsilon^{-4})$ oracle complexity. This complexity result matches that of non-adaptive bilevel optimization algorithms in~\citep{hao2024bilevel,gong2024a} when the upper-level function exhibits unbounded smoothness, as well as the complexity of Adam for single-level optimization with generalized smooth functions~\citep{li2023convergence}. It is also worth noting that we choose a larger learning rate $\eta=\widetilde{\Theta}(\epsilon^2)$ for the upper-level updates, compared to $\eta=\widetilde{\Theta}(\epsilon^3)$ used in the SLIP algorithm~\citep{gong2024a}. 
See \cref{tab:bilevel} for a comparison of bilevel optimization algorithms under the unbounded smoothness setting.

\textbf{Remark 2}: In Theorem~\ref{thm:main}, we require the momentum parameter $\beta$ to be small.
% (i.e., $\beta=\widetilde{\Theta}(\epsilon^2)$). 
% which corresponds to $1-\beta_1$ in the original Adam paper~\cite{kingma2014adam}. 
Note that the default choice of $\beta_1$ in \cite{kingma2014adam} is $0.9$, which corresponds to $\beta=0.1$ in our algorithm. This seemingly different choice of $\beta$ (i.e., $\beta=\widetilde{\Theta}(\epsilon^2)$ in Theorem~\ref{thm:main} versus $\beta=0.1$ in~\cite{kingma2014adam}) is due to the problem setting. In practice, Adam is typically used to minimize functions with finite-sum structure~\cite{zhang2022adam}, while our paper considers a more challenging stochastic optimization setting. In stochastic optimization setting, constant $\beta$ makes Adam diverge. For example, \citep[Theorem 3]{reddi2019convergence} has shown that there is a stochastic convex optimization problem for which Adam does not converge for any constant $\beta$. We believe a small $\beta=\widetilde{\Theta}(\epsilon^2)$ is a reasonable surrogate for $\beta=0.1$ under stochastic optimization setting: such a choice of $\beta$ is also used in the analysis of Adam under the single-level stochastic optimization setting with a generalized smooth upper-level function~\cite{li2023convergence}. Moreover, existing convergence analyses of (single-level) Adam that do not need such choice of $\beta$ require other strong assumptions for the objective function, which is incompatible to our setting. Please see discussion and \cref{tab:adam} in \cref{app:table} for details. 
% Moreover, existing convergence analyses of (single-level) Adam that do not need such choice of $\beta$ require other strong assumptions for the objective function, which is incompatible to our setting. They either rely on the bounded gradient assumption \citep{de2018convergence,defossez2020simple}, or they only prove convergence to some neighborhood of stationary points with a constant radius unless assuming the strong growth condition under the finite sum setting \citep{zhang2022adam,wang2022provable}. Please see \cref{tab:adam} in \cref{app:table} for more details. 

\textbf{Remark 3}: One limitation of our complexity dependence on $\lambda$ is $O(\lambda^{-2})$, which can be large since $\lambda$ is typically small in practice. To address this concern, we conduct additional experiments in \cref{fig:lambda_sensitivity} to evaluate the empirical sensitivity of our algorithm to $\lambda$. Although the default choice of $\lambda$ is $10^{-8}$ \citep{kingma2014adam}, increasing it up to $10^{-4}$ only causes minor differences in AUC maximization, and increasing it up to $10^{-3}$ leads to minor changes in hyper-representation performance with BERT \citep{devlin2018bert}. 

% In addition, one limitation of the dependence of our complexity on $\lambda$ is $O(\lambda^{-2})$, which might be large since $\lambda$ is usually small in practice. 
% We have conducted additional experiments in \cref{fig:lambda_sensitivity} to show the empirical performance of our algorithm is not very sensitive to the choice of $\lambda$. Although the default choice of $\lambda$ is $10^{-8}$ \citep{kingma2014adam}, increasing it up to $10^{-4}$ only causes minor differences in AUC maximization, and increasing it up to $10^{-3}$ leads to minor changes in hyper-representation performance with BERT \citep{devlin2018bert}. 

\subsection{Proof Sketch}
In this section, we provide a proof sketch for \cref{thm:main}. The detailed proof can be found in \cref{app:proof_biadam}. Let $y_t^*=y^*(x_t)$. The key idea is to provide a high probability bound of lower-level estimation error $\|y_t-y_t^*\|$ when the upper-level variable $x$ is updated by the vanilla Adam. \cref{lm:event-y} provides such a guarantee: the lower-level error $\|y_t-y_t^*\|$ is bounded by a function of the initial estimation error $\|y_1-y_1^*\|$, the variance term $\sigma_{g,1}^2$, and an auxiliary momentum estimator of the hypergradient $\|\hu_t\|$ (see definition of $\hu_t$ in \eqref{eq:aux}). Based on \cref{lm:event-y}, we introduce \cref{lem:upperlevel} and \ref{lem:biasvirtual}, which incorporate the lower-level error into the upper-level problems and adapt the stopping time technique of Adam~\citep{li2023convergence} to prove the convergence. The proof of \cref{lm:event-y} is a direct application of the randomness decoupling lemma (i.e., \cref{lm:any-sequence-main} in \cref{sec:random-decoupling}). All of the proofs in this section are based on \cref{ass:bilevel-assumption,ass:noise,ass:bilevel-additional}. 
The full statements and proofs of \cref{lm:any-sequence-main,lm:warm-start-main,lm:event-y,lem:upperlevel,lem:biasvirtual} are provided in \cref{sec:proof-randomness-decouple,sec:proof-warm-start,sec:proof-event-y,sec:proof-momentum,sec:proof-biasvirtual}.
% The full statements of the proofs of \cref{lm:any-sequence-main,lm:warm-start-main,lm:event-y,lem:upperlevel,lem:biasvirtual} are deferred to \cref{sec:proof-randomness-decouple,sec:proof-warm-start,sec:proof-event-y,sec:proof-momentum,sec:proof-biasvirtual}.

% \subsubsection{Equivalent Update Rule of AdamBO}
Let $\alpha_t = \frac{\beta}{1-(1-\beta)^t}$ and $\alphasqt = \frac{\betasq}{1-(1-\betasq)^t}$. Inspired by \citep{li2023convergence}, we provide an equivalent yet simpler update rule of lines 5-8 of \cref{alg:bi-adam} (see \cref{prop:bi-adam-equiv} for more details):
\begin{equation*}
    \begin{aligned}
        \hm_t &= (1-\alpha_t)\hm_{t-1} + \alpha_t \hatphi(x_t,y_t;\Bar{\xi}_t), \\
        % \quad
        \hv_t &= (1-\alphasqt)\hv_{t-1}+\alphasqt (\hatphi(x_t,y_t;\Bar{\xi}_t))^2.
    \end{aligned}
\end{equation*}

% \subsubsection{Random Decoupling Lemma for Lower-Level Error Control}
\subsubsection{Randomness Decoupling Lemma}
\label{sec:random-decoupling}
In this section, we introduce the random decoupling lemma (\cref{lm:any-sequence-main}) for the lower-level error control. The rationale is as follows: for any given upper-level variable sequence and any given randomness from the upper-level updates that satisfy certain conditions and are consistent with the AdamBO updates, we can bound the lower-level error with high probability, where the randomness is taken solely from lower-level random variables. Specifically, for any given sequence $\{\tx_t\}$, define $\tilde{\zeta}_t$ and $\hat{\xi}_t$ as the random variables from the lower-level and upper-level, respectively, at the $t$-th iteration (see \eqref{eq:hat-xi-def} for definition). We consider the following update rule for $\{\ty_t\}$, which is exactly SGD and corresponds to line 5 of \Cref{alg:bi-adam}:
\begin{equation} \label{eq:yt-any-update-main}
    \ty_{t+1} = \ty_t - \gamma\gdy G(\tx_t,\ty_t;\tilde{\zeta}_t).
\end{equation}

Let $\ty_t^*=y^*(\tx_t)$ and $\tilde{\gF}_t^y = \sigma(\tilde{\zeta}_1,\dots,\tilde{\zeta}_{t-1})$. Denote $\tg_t \coloneqq \max_{k\leq t}\|\gdphi(\tx_k)\|$, $\tl_t \coloneqq L_0 + L_1 \tilde{G}_t$. We also introduce the following auxiliary sequences $\{\tm_t\}$ and $\{\tu_t\}$ for our analysis:
\begin{equation}
\label{eq:virtualsequence}
    \begin{aligned}
        \tm_t &= (1-\alpha_t)\tm_{t-1} + \alpha_t \hatphi(\tx_t,\ty_t;\hat{\xi}_t), \\
        % \quad
        \tu_t &= (1-\alpha_t)\tu_{t-1} + \alpha_t\hatphi(\tx_t,\ty_t^*;\hat{\xi}_t).
    \end{aligned}
\end{equation}

\begin{lemma}[Randomness Decoupling Lemma] \label{lm:any-sequence-main}
Given any sequence $\{\tx_t\}$ and any randomness $\{\hat{\xi}_t\}$ such that
% \begin{small}
\begin{equation}
\label{eq:adamproperty}
    \begin{aligned}
        \|\tx_{t+1}-\tx_t\|^2
        \leq \frac{2\eta^2}{\lambda^2}\left(\|\tu_t\|^2 + \tl_t^2\sum_{j=1}^{t}d_{t,j}\|\ty_j-\ty_j^*\|^2\right),
    \end{aligned}
\end{equation}
% \end{small}%
where $\{d_{t,j}\}_{j=1}^{t}$ is defined in \eqref{eq:dj-def}. Let $\{\ty_t\}$ be the iterates generated by the update rule \eqref{eq:yt-any-update-main} with $\gamma\leq 1/2l_{g,1}$ and choose $\gamma=2\beta/\mu$. For any given $\delta\in(0,1)$ and all $t\geq1$, the following holds with probability at least $1-\delta$ over the randomness in $\tilde{\gF}_{T+1}^y$:
% \begin{small}
% \begin{equation} \label{eq:any-sequence-main}
%     \begin{aligned}
%         &\|\ty_t-\ty_t^*\|^2 
%         \leq \left(1-\frac{\mu\gamma}{2}\right)^{t-1}\|\ty_1-\ty_1^*\|^2 + \frac{8\gamma\sigma_{g,1}^2}{\mu}\ln\frac{eT}{\delta} \quad\text{(Variance)} \\
%         % &\quad\quad\quad+ \left(\frac{4\eta^2l_{g,1}^2}{\lambda^2\mu^3\gamma}\|\ty_1-\ty_1^*\|^2 + \frac{16\eta^2l_{g,1}^2\sigma_{g,1}^2}{\lambda^2\mu^4}\right)\sum_{i=1}^{t-1}\left(1-\frac{\mu\gamma}{2}\right)^{t-1-i}\tl_i^2 \quad\text{(Drift)} \\
%         % &\quad\quad\quad+ \frac{4\eta^2l_{g,1}^2}{\lambda^2\mu^3\gamma}\sum_{i=1}^{t-1}\left(1-\frac{\mu\gamma}{2}\right)^{t-1-i}\|\tu_i\|^2 + \frac{64\eta^4l_{g,1}^4}{\lambda^4\mu^8\gamma^4}\sum_{i=1}^{t-1}\left(1-\frac{\mu\gamma}{2}\right)^{t-1-i}\alpha_i\tl_i^2\|\tu_i\|^2. \quad\text{(Drift)}
%         &+ \left(\frac{4\eta^2l_{g,1}^2}{\lambda^2\mu^3\gamma}\|\ty_1-\ty_1^*\|^2 + \frac{16\eta^2l_{g,1}^2\sigma_{g,1}^2}{\lambda^2\mu^4}\right)\sum_{i=1}^{t-1}\left(1-\frac{\mu\gamma}{2}\right)^{t-1-i}\tl_i^2 \quad\text{(Drift)} \\
%         &+ \frac{4\eta^2l_{g,1}^2}{\lambda^2\mu^3\gamma}\sum_{i=1}^{t-1}\left(1-\frac{\mu\gamma}{2}\right)^{t-1-i}\|\tu_i\|^2 + \frac{64\eta^4l_{g,1}^4}{\lambda^4\mu^8\gamma^4}\sum_{i=1}^{t-1}\left(1-\frac{\mu\gamma}{2}\right)^{t-1-i}\alpha_i\tl_i^2\|\tu_i\|^2. \quad\text{(Drift)}
%     \end{aligned}
% \end{equation}
\begin{equation} \label{eq:any-sequence-main}
\small
    \begin{aligned}
        &\|\ty_t-\ty_t^*\|^2 
        \leq \left(1-\frac{\mu\gamma}{2}\right)^{t-1}\|\ty_1-\ty_1^*\|^2 + \text{(Variance)} + \text{(Drift)}, \\
        &\text{(Variance)}
        = \frac{8\gamma\sigma_{g,1}^2}{\mu}\ln\frac{eT}{\delta}, \\
        &\text{(Drift)}
        = \frac{4\eta^2l_{g,1}^2}{\lambda^2\mu^3\gamma}\sum_{i=1}^{t-1}\left(1-\frac{\mu\gamma}{2}\right)^{t-1-i}\|\tu_i\|^2 \\
        &+ \frac{64\eta^4l_{g,1}^4}{\lambda^4\mu^8\gamma^4}\sum_{i=1}^{t-1}\left(1-\frac{\mu\gamma}{2}\right)^{t-1-i}\alpha_i\tl_i^2\|\tu_i\|^2 \\
        &+ \frac{\eta^2}{\lambda^2}\left(\frac{4l_{g,1}^2}{\mu^3\gamma}\|\ty_1-\ty_1^*\|^2 + \frac{16l_{g,1}^2\sigma_{g,1}^2}{\mu^4}\right)\sum_{i=1}^{t-1}\left(1-\frac{\mu\gamma}{2}\right)^{t-1-i}\tl_i^2.
    \end{aligned}
\end{equation}
% \end{small}%
\end{lemma}

\textbf{Remark}: \cref{lm:any-sequence-main} shows that, when \eqref{eq:adamproperty} holds for any sequence $\{\tilde{x}_t\}$ and any $\{\hat{\xi}_t\}$ (as satisfied by the vanilla Adam update for the upper-level variable), the lower-level error can be controlled with high probability as in~\eqref{eq:any-sequence-main}. In addition, the high probability is taken over the randomness solely from the lower-level filtration $\tilde{\gF}_{T+1}^y$. This lemma provides a technical tool to control the lower-level error without concerns about the dependency issues from the upper-level randomness. In particular, the right-hand side of~\eqref{eq:any-sequence-main} consists of two parts: the standard variance term, which does not involve the update of $\{\tx_t\}$ over $t$; and the drift terms, which account for the update of $\{\tx_t\}$ over time.

\subsubsection{Applications of the Randomness Decoupling Lemma and Remaining Proof}
% \subsubsection{Applications of \cref{lm:any-sequence-main}}
Given a large enough constant $G$, denote $L = L_0 + L_1G$ and $\psi=C_LG^2/2L$, where $G$ is defined in \cref{thm:main} and $C_L$ is defined in \eqref{eq:Cu-def}. 
Now we formally define the stopping time $\tau$ as 
$\tau \coloneqq \min\{t \mid \Phi(x_t)-\Phi^* > \psi\} \wedge (T+1)$.
% \begin{equation*}
%     \tau \coloneqq \min\{t \mid \Phi(x_t)-\Phi^* > \psi\} \wedge (T+1).
% \end{equation*}
Based on \cref{lm:reverse-PL}, we know that if $t<\tau$, we have both $\Phi(x_t)-\Phi^*\leq \psi$ and $\|\gdphi(x_t)\|\leq G$. 
Similar to \cref{sec:random-decoupling}, we introduce the following auxiliary sequence $\{\hu_t\}$ for our analysis: 
\begin{equation} 
\label{eq:aux}
    \begin{aligned}
        \hu_t 
        = (1-\alpha_t)\hu_{t-1} + \alpha_t\hatphi(x_t,y_t^*;\Bar{\xi}_t) .
        % = \sum_{j=1}^{t}d_{t,j}\hatphi(x_t,y_t^*;\Bar{\xi}_t).
    \end{aligned}
\end{equation}

% The full statements of the proofs of \cref{lm:warm-start-main,lm:event-y,lem:upperlevel,lem:biasvirtual} are deferred to \cref{sec:proof-warm-start,sec:proof-event-y,sec:proof-momentum,sec:proof-biasvirtual}.

\begin{lemma}[Warm-Start] \label{lm:warm-start-main}
Choose $\gamma\leq 1/2l_{g,1}$. With probability at least $1-\delta/4$ over the randomness in $\gF_{\init}$ (denote this event as $\gE_0$) that: 
$\|y_1-y_1^*\|^2 \leq \left(1-\frac{\mu\gamma}{2}\right)^{T_0}\|y_0-y_0^*\|^2 + \frac{8\gamma\sigma_{g,1}^2}{\mu}\ln\frac{4e}{\delta}$.
% \begin{small}
% \begin{equation*}
%     \|y_1-y_1^*\|^2 \leq \left(1-\frac{\mu\gamma}{2}\right)^{T_0}\|y_0-y_0^*\|^2 + \frac{8\gamma\sigma_{g,1}^2}{\mu}\ln\frac{4e}{\delta}.
% \end{equation*}
% \end{small}%
\end{lemma}

\begin{lemma} \label{lm:event-y}
Under the parameter choices in \cref{lm:general-recursion}, apply \cref{lm:any-sequence-main} with $\{\tx_t\}=\{x_t\}$, $\{\ty_t\}=\{y_t\}, \{\tu_t\}=\{\hu_t\}$ and $\{\tl_t\}=\{\hl_t\}$, then \eqref{eq:any-sequence-main} holds with probability at least $1-\delta/4$ over the randomness in $\gF_{T+1}^y$ (denote this event as $\gE_{y}$).
\end{lemma}

\textbf{Remark}: \cref{lm:warm-start-main} and \cref{lm:event-y} together provide a high probability bound for the lower-level error, where the randomness is taken only from the lower-level filtrations $\mathcal{F}_{\init}$ and $\mathcal{F}_{T+1}^y$. \cref{lm:event-y} is a direct application of \cref{lm:any-sequence-main} to the actual sequence $\{x_t\}$ and $\{y_t\}$ in \cref{alg:bi-adam}.

\begin{lemma} 
\label{lem:upperlevel}
If $t<\tau$, we have $\|\gdphi(x_t)\| \leq G$, $\|\hu_t\| \leq C_{u,0}$; under event $\gE_0\cap\gE_{y}$, if $t<\tau$, we have $\|\hm_t\| \leq C_{u,0} + C_{u,1}\varrho$, $\hv_t \preceq (C_{u,0} + C_{u,1}\varrho)^2$, where constants $C_{u,0}, C_{u,1}, \varrho$ are defined in \eqref{eq:Cu-def} and \eqref{eq:recur-2}, respectively.
\end{lemma}

\textbf{Remark}: \cref{lem:upperlevel} generalizes the stopping time analysis from the single-level setting~\citep{li2023convergence} to the bilevel setting and is useful for upper-level analysis. It shows that the momentum estimators of the hypergradient remains bounded when $t<\tau$ and $\gE_0\cap\gE_y$ holds. This implies that $x_{t+1}$ and $x_t$ remains close for small enough $\eta$, allowing us to apply \cref{lm:Phi-relax-smooth,lm:descent-inequality}.

\begin{lemma}
\label{lem:biasvirtual}
Under event $\gE_0\cap\gE_{y}$ and the parameter choices in \cref{lm:general-recursion}, we have 
$\sum_{t=1}^{\tau-1}\|\hm_t-\hu_t\|^2 \leq O(\sqrt{T}) + O(1)\sum_{t=1}^{\tau-1}(\|\epsilon_t\|^2 + \|\gdphi(x_t)\|^2)$.
% \begin{small}
% \begin{equation*}
%     \begin{aligned}
%         &\sum_{t=1}^{\tau-1}\|\hm_t-\hu_t\|^2
%         \leq TL^2\left(\left(1 + \frac{8\eta^2l_{g,1}^2L^2}{\lambda^2\mu^4\gamma^2}\right) \|y_1-y_1^*\|^2 + \left(\frac{8\gamma}{\mu}\ln\frac{4eT}{\delta} + \frac{32\eta^2l_{g,1}^2L^2}{\lambda^2\mu^5\gamma}\right)\sigma_{g,1}^2\right) \\
%         % &+ L^2\left(\frac{8\eta^2l_{g,1}^2}{\lambda^2\mu^4\gamma^2} + \frac{2048\eta^4l_{g,1}^4L^2}{\lambda^4\mu^8\gamma^4}\left(2+\ln\frac{1}{\beta}\right)\right)\sum_{t=1}^{\tau-1}\|\epsilon_t\|^2 + 2\|\gdphi(x_t)\|^2 + 2\|\E_t[\hatphi(x_t,y_t^*;\Bar{\xi}_t)] - \gdphi(x_t)\|^2 \\
%         &+ L^2\left(\frac{8\eta^2l_{g,1}^2}{\lambda^2\mu^4\gamma^2} + \frac{2048\eta^4l_{g,1}^4L^2}{\lambda^4\mu^8\gamma^4}\left(2+\ln\frac{1}{\beta}\right)\right)\sum_{t=1}^{\tau-1}\|\epsilon_t\|^2 + 2\|\gdphi(x_t)\|^2 + \frac{2l_{g,1}^2l_{f,0}^2}{\mu^2}\left(1-\frac{\mu}{l_{g,1}}\right)^{2Q}.
%         % &\quad+ 2L^2\left(\frac{8\eta^2l_{g,1}^2}{\lambda^2\mu^4\gamma^2} + \frac{2048\eta^4l_{g,1}^4L^2}{\lambda^4\mu^8\gamma^4}\left(2+\ln\frac{1}{\beta}\right)\right)\sum_{t=1}^{\tau-1}\|\gdphi(x_t)\|^2.
%     \end{aligned}
% \end{equation*}
% \end{small}%
% \begin{equation*}
%     \begin{aligned}
%         \sum_{t=1}^{\tau-1}\|\hm_t-\hu_t\|^2
%         \leq O(\sqrt{T}) + O(1)\sum_{t=1}^{\tau-1}(\|\epsilon_t\|^2 + \|\gdphi(x_t)\|^2).
%     \end{aligned}
% \end{equation*}
\end{lemma}

\textbf{Remark}: \cref{lem:biasvirtual} provides a bound for the difference between the actual momentum $\hat{m}_t$ versus the virtual momentum $\hat{u}_t$ under the good event $\gE_0 \cap \gE_y$, which is essential for establishing the convergence guarantees for AdamBO.

\vspace*{-0.1in}
\section{Experiments}

% \textcolor{red}{XG: remove VR-AdamBO in the figures.}

\subsection{Hyper-representation learning}

\begin{figure*}[t]
\begin{center}
\subfigure[\scriptsize Train ACC]{\includegraphics[width=0.24\linewidth]{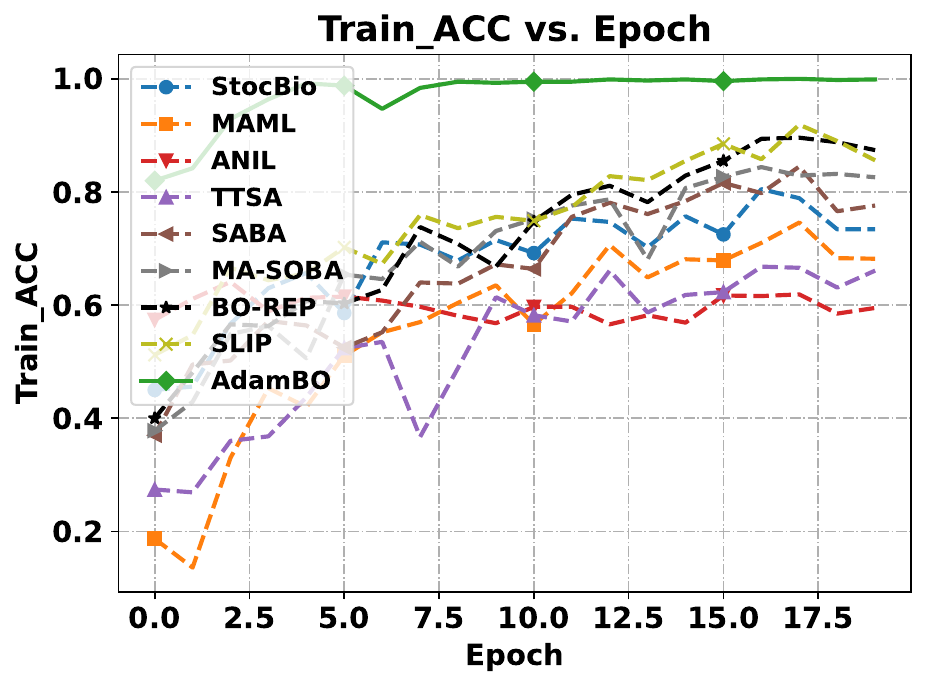}}   \   
\subfigure[\scriptsize Test ACC ]{\includegraphics[width=0.24\linewidth]{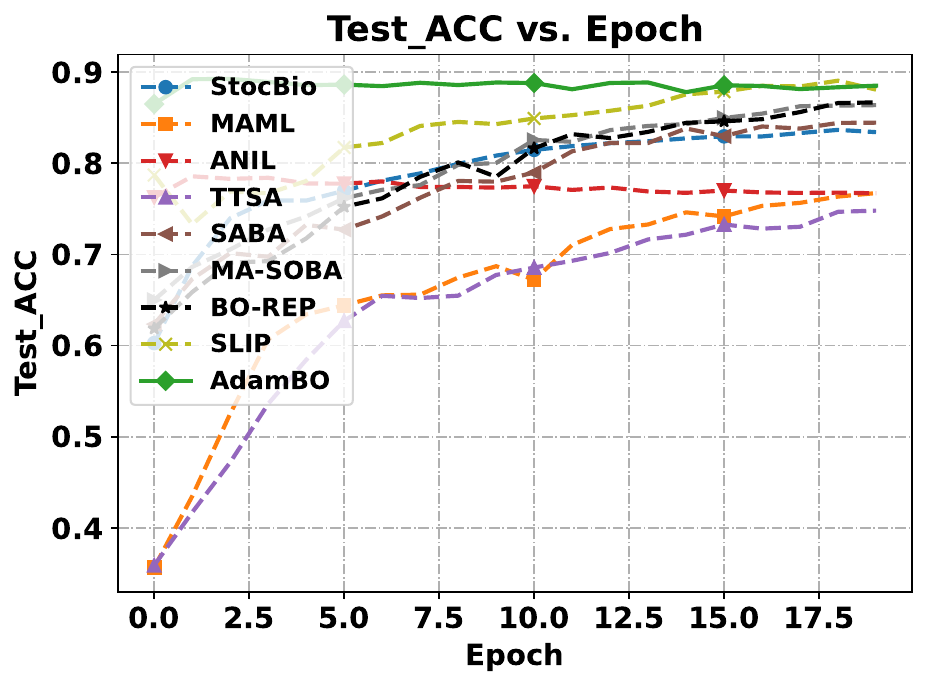}} \ 
\subfigure[\scriptsize Train ACC vs. Time]{\includegraphics[width=0.24\linewidth]{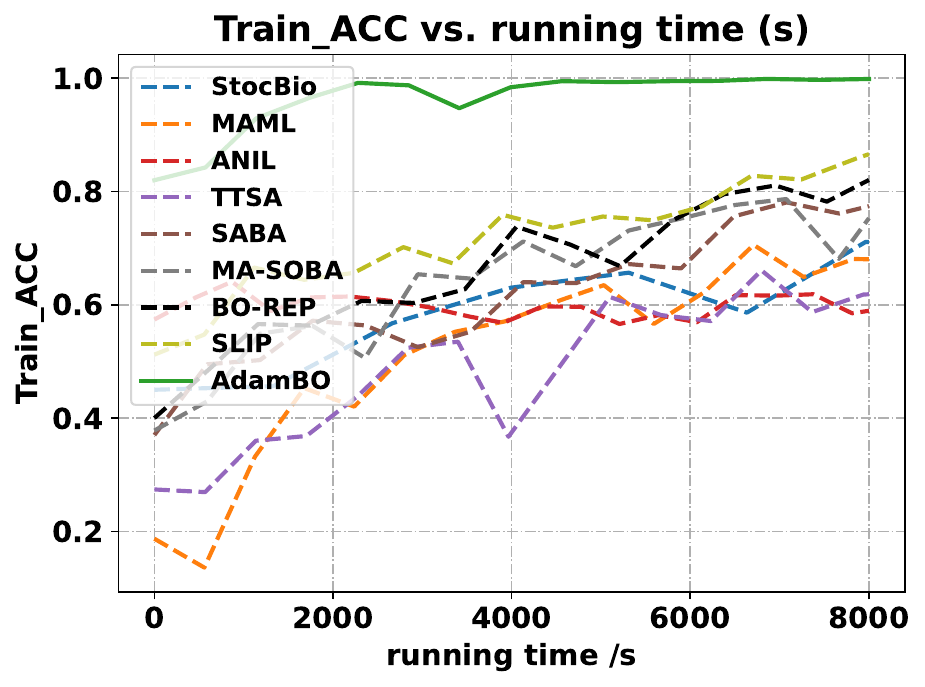}} \
\subfigure[\scriptsize Test ACC vs. Time]{\includegraphics[width=0.24\linewidth]{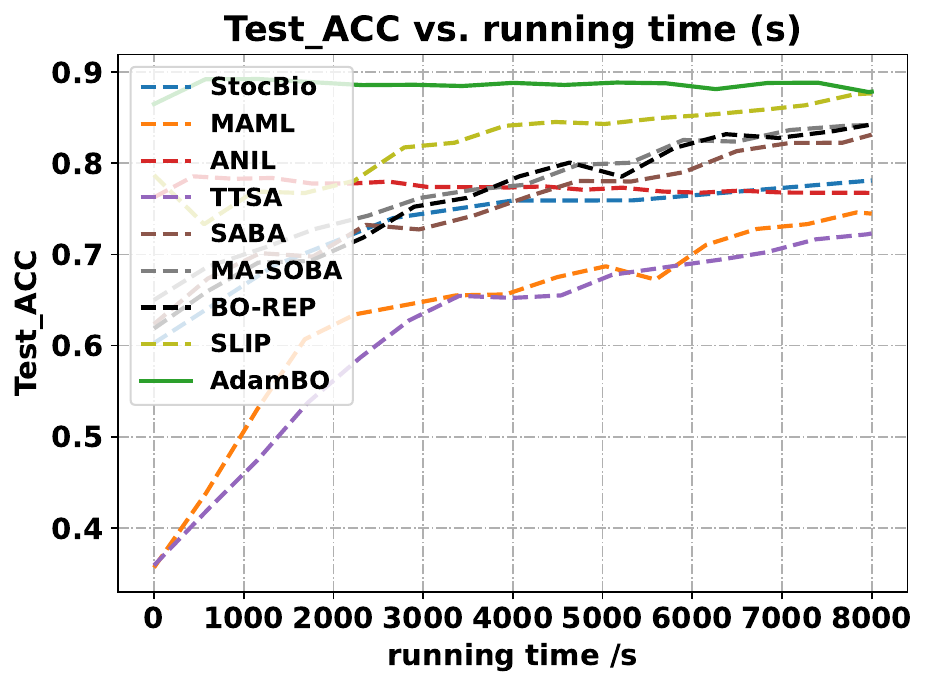}}
\end{center}
\vspace*{-0.15in}
% \caption{Comparison with bilevel optimization baselines on hyper-representation. The experiment is performed on BERT, which contains 8 transformer encoder layers acting as the representation layers and a fully-connected layer acting as the adapter.}
\caption{Comparison with bilevel optimization baselines on BERT for hyper-representation.}
\label{fig:acc_HR_bert}
\end{figure*}

\begin{figure*}[t]
\begin{center}
\subfigure[\scriptsize Train ACC]{\includegraphics[width=0.24\linewidth]{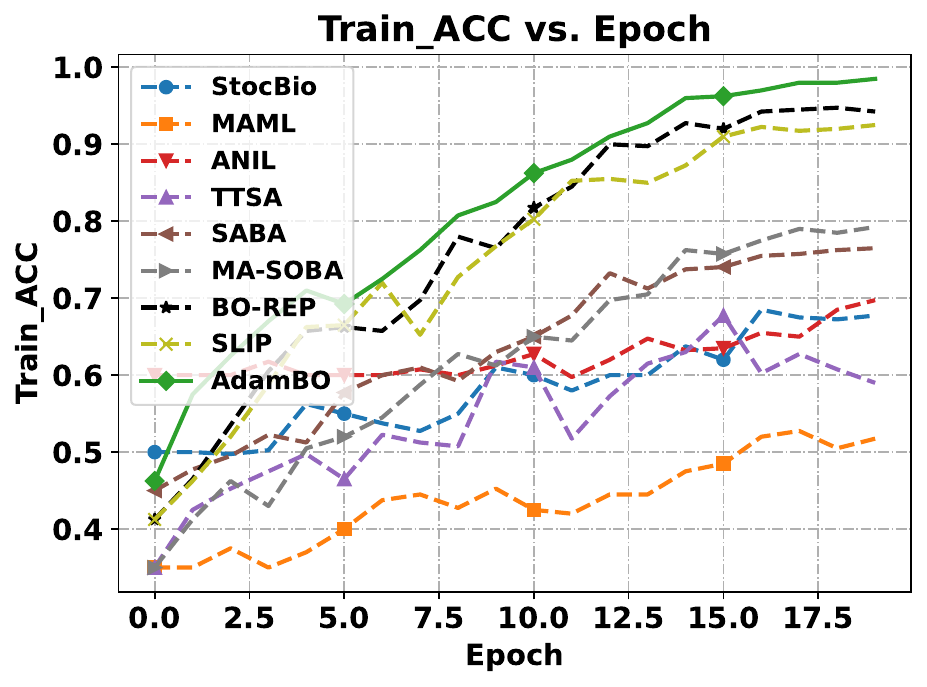}}   \   
\subfigure[\scriptsize Test ACC ]{\includegraphics[width=0.24\linewidth]{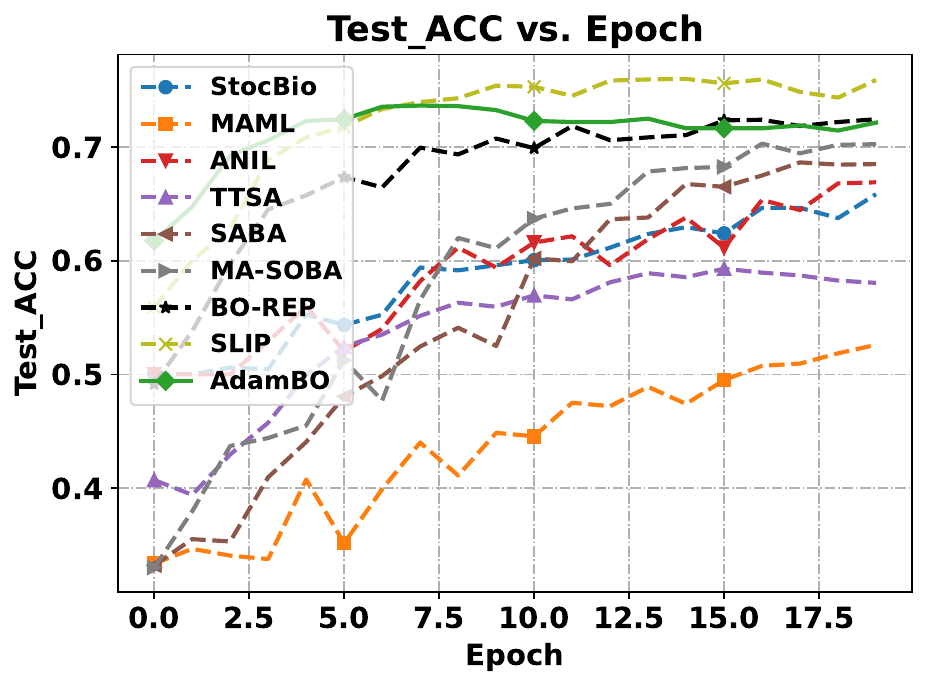}} \ 
\subfigure[\scriptsize Train ACC vs. Time]{\includegraphics[width=0.24\linewidth]{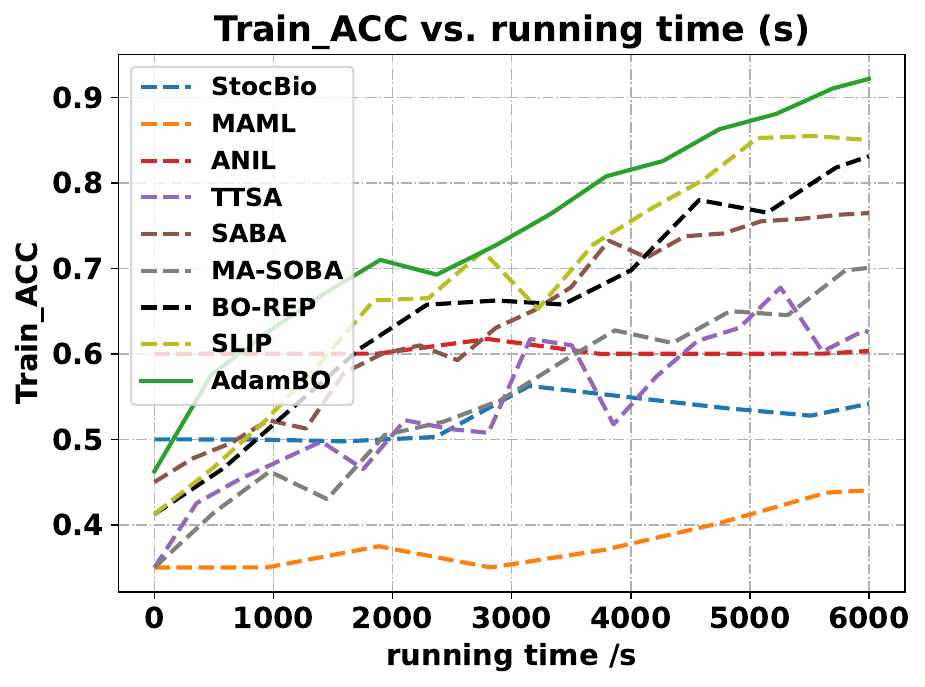}} \
\subfigure[\scriptsize Test ACC vs. Time]{\includegraphics[width=0.24\linewidth]{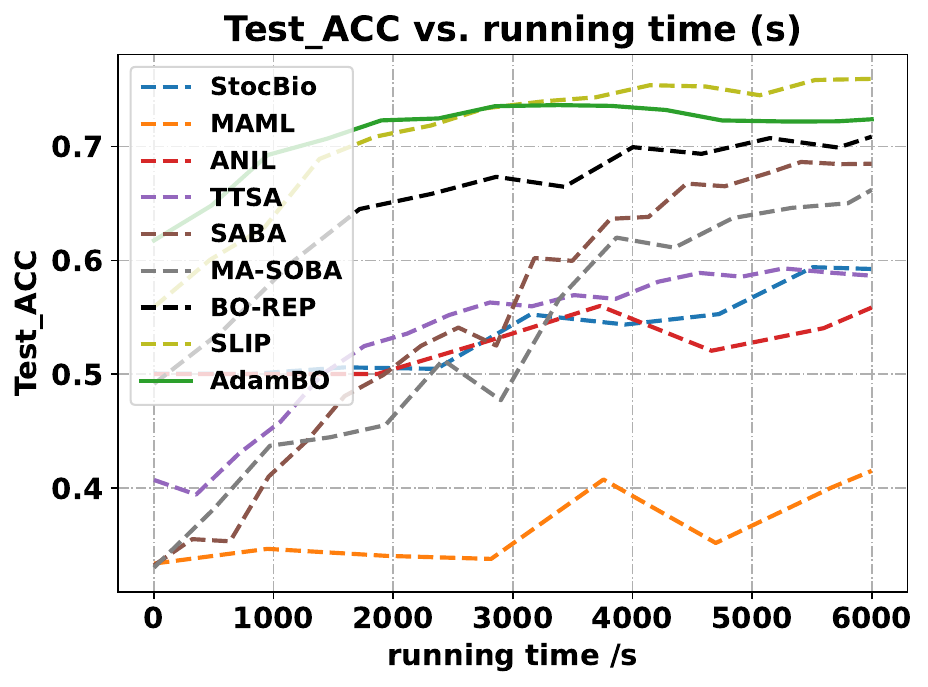}}
\end{center}
\vspace*{-0.15in}
\caption{Comparison with bilevel optimization baselines on RNN for hyper-representation.}
\label{fig:acc_HR}
\end{figure*}

Hyper-representation learning, i.e., meta-learning~\citep{finn2017model}, aims to find a good meta learner parameterized by $x$, such that it can quickly adapt to a new task $i$ by fine-tuning the corresponding adapter $y_i$. Consider a meta-learning task consisting of $K$ tasks with the training set $\{\mathcal{D}_i^{tr}\;|\;i=1,\ldots, K\}$ and validation set  $\{\mathcal{D}_i^{val}\;|\; i=1,\ldots, K\}$. Each task has a loss function  $\mathcal{L}(x, y_i;\xi_i)$ over each sample $\xi_i$. This meta-learning problem can be reformulated as a bilevel optimization, where the lower-level objective function tries to find an optimal task-specific adapter $y_i^*(x)$ on training data $\mathcal{D}_{i}^{tr}$, and the upper-level minimizes the objective function on validation data  $\mathcal{D}_{i}^{val}$ by finding the optimal meta-learner $x$ with a set of adapters $y=\{y_1^*(x), y_2^*(x),\ldots, y_K^*(x)\}$. We have the following formulation:
%\begin{small}
\begin{equation*}\label{eq:hr}
    \begin{aligned}
        &\min_{x}\frac{1}{K}\sum_{i=1}^{K}\frac{1}{|\mathcal{D}^{val}_i|}\sum_{\xi \in \mathcal{D}_{i}^{val}}\mathcal{L}(x, y^*(x); \xi), \\
        &\text{s.t.,}\; 
        \ y^*(x) = \argmin_{y}\frac{1}{K}\sum_{i=1}^{K}\mathcal{L}_{\mathcal{D}_{i}^{tr}}(x, y_i; \zeta) + \frac{\mu}{2}\|y_i\|^2,
    \end{aligned}
\end{equation*}
%\end{small}%
where $\mathcal{L}_{\mathcal{D}_{i}^{tr}}(x, y_i; \zeta) = \frac{1}{|\mathcal{D}_{i}^{tr}|}\sum_{\zeta\in \mathcal{D}^{tr}_i}\mathcal{L}(x, y_i; \zeta)$. The adapter (parameterized by $y_i$) is typically instantiated as the last linear layer, and the meta learner (parameterized by $x$) is the remaining layers of model, which guarantees that the lower-level function to be strongly-convex when $\mu>0$. 

We conduct meta-learning experiments on a larger language model, specifically an 8-layer BERT \citep{devlin2018bert} model. The experiments are performed on a widely-used question classification dataset TREC \citep{li-roth-2002-learning}, which contains 6 coarse-grained categories. To evaluate our approach on meta-learning, we construct $K=500$ meta tasks, where the training data $\mathcal{D}_i^{tr}$ and validation data $\mathcal{D}_i^{val}$ for the $i$-th task are randomly sampled from two disjoint categories, with $5$ examples per category. A BERT model, with 8 self-attention layers and a fully-connected layer, is used in our experiment. The self-attention layers serve as representation layers (with their parameters treated as upper-level variables) and the fully-connected layer (with its parameters treated as lower-level variables) serves as an adapter, where each self-attention layer consists of 8 self-attention heads with the hidden size being 768. The fully-connected layer acts as a classifier, with the input dimension of 768 and the output dimension of 6 (corresponding to the 6 categories). Our bilevel optimization algorithm trains the representation layers and the adapter on the meta tasks ($\mathcal{D}^{tr}$ and $\mathcal{D}^{val}$) from scratch, and then evaluate it on the test set $\mathcal{D}^{te}$. During the evaluation phase, we fix the parameters of representation layers and just fine-tune the adapters. We train the models for 20 epochs and compare it with other bilevel optimization baseline algorithms. 
% The training and testing comparison results are presented in \cref{fig:acc_HR_bert}. As shown, the proposed algorithm AdamBO achieves fast convergence to the best training and test results among all baselines. 

We compare with typical meta-learning algorithms, MAML \citep{rajeswaran2019meta} and ANIL \citep{raghu2019rapid}, and recent bilevel optimization algorithms, StocBio \citep{ji2021bilevel}, TTSA \citep{hong2023two}, SABA \citep{dagreou2022framework}, MA-SOBA \citep{chen2023optimal}, BO-REP~\citep{hao2024bilevel}, SLIP \citep{gong2024a}. The comparison results of training and testing accuracy are shown in \cref{fig:acc_HR_bert}. AdamBO achieves fast convergence to the best training and test results among all baselines. One can refer to \cref{app:hyper_param} for detailed hyper-parameter choices and experimental settings. All the experiments are run on an single NVIDIA A6000 (48GB memory) GPU and a AMD EPYC 7513 32-Core CPU.

\begin{figure*}[!t]
\begin{center}
\subfigure[\scriptsize Train AUC]{\includegraphics[width=0.24\linewidth]{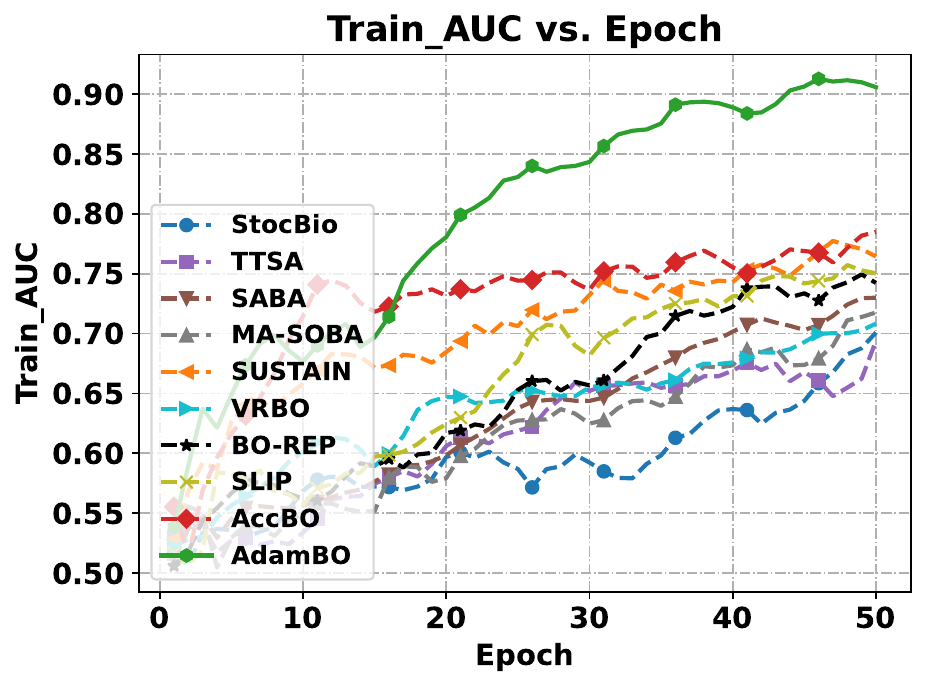}}   
\subfigure[\scriptsize Test AUC]{\includegraphics[width=0.24\linewidth]{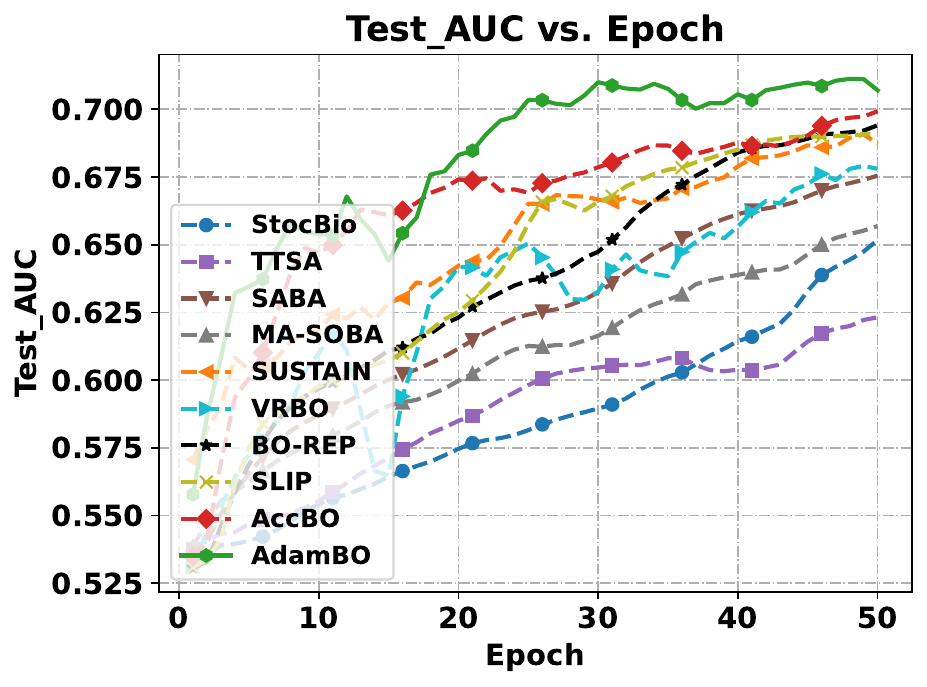}}   
\subfigure[\scriptsize Train AUC vs. Time]{\includegraphics[width=0.24\linewidth]{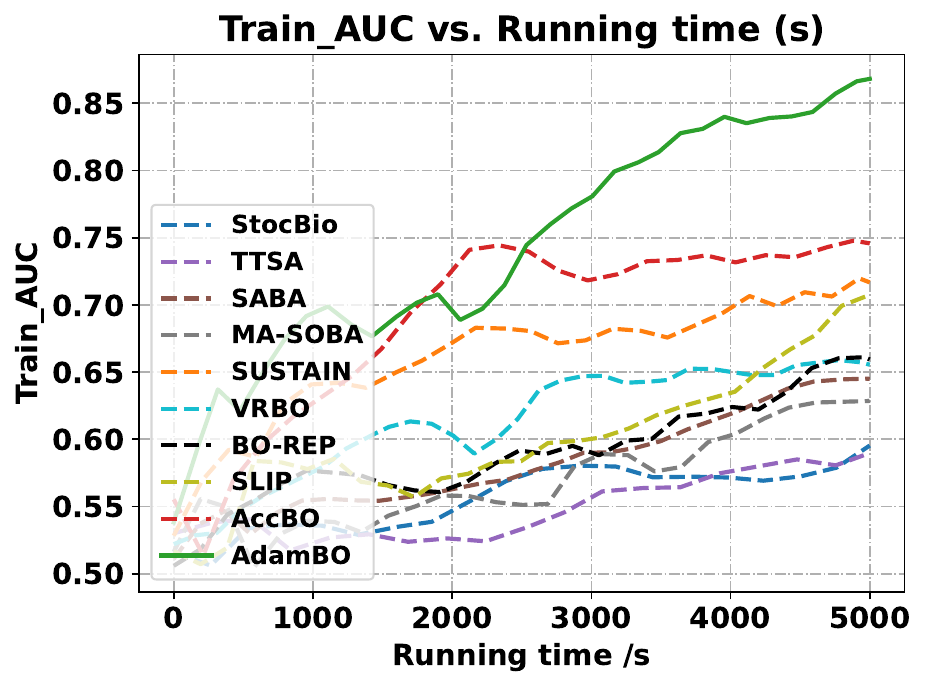}}  
\subfigure[\scriptsize Test AUC vs. Time]{\includegraphics[width=0.24\linewidth]{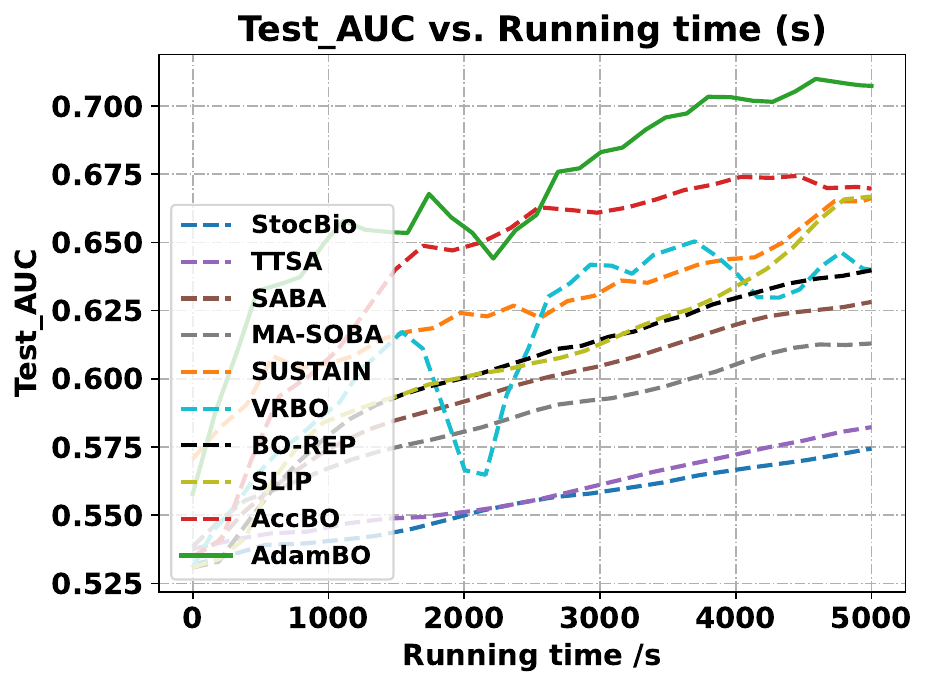}}
\end{center}
\vspace*{-0.15in}
\caption{Transformer for AUC maximization on Sentiment140 dataset with imbalance ratio of 0.9. }
\label{fig:auc_trans}
\end{figure*}

We also conduct the meta-learning experiments on RNN for text classification on dataset Stanford Natural Language Inference (SNLI)~\citep{bowman2015large}, which consists of 570k pairs of sentences with 3 classes. We construct $K=500$ tasks, where each task $\mathcal{D}_i^{tr}$ and $\mathcal{D}_i^{val}$ randomly sample two disjoint categories from the original data, respectively. Empirically, we use mini-batches of meta-tasks for training, with a task batch size of 25. A 3-layer recurrent network is used as representation layers and a fully-connected layer as an adapter. The input dimension, hidden dimension and output dimension are set to be 300, 4096, and 3, respectively. The comparison results of training and testing accuracy are shown in \cref{fig:acc_HR}. AdamBO outperforms other baselines on the training set and exhibits faster convergence rate.

% We conduct the meta-learning experiments for the text classification on dataset Stanford Natural Language Inference (SNLI)~\citep{bowman2015large}, which consists of 570k pairs of sentences with 3 classes. We construct $K=500$ tasks, where each task $\mathcal{D}_i^{tr}$ and $\mathcal{D}_i^{val}$ randomly sample two disjoint categories from the original data, respectively. Empirically, we use mini-batches of meta-tasks for training, with a task batch size of 25. A 3-layer recurrent network is used as representation layers and a fully-connected layer as an adapter. The input dimension, hidden dimension and output dimension are set to be 300, 4096, and 3, respectively.

% We compare with typical meta-learning algorithms, MAML \citep{rajeswaran2019meta} and ANIL \citep{raghu2019rapid}, and recent bilevel optimization algorithms, StocBio \citep{ji2021bilevel}, TTSA \citep{hong2023two}, SABA \citep{dagreou2022framework}, MA-SOBA \citep{chen2023optimal}, BO-REP~\citep{hao2024bilevel}, SLIP \citep{gong2024a}. The comparison results of training and testing accuracy are shown in \cref{fig:acc_HR}. AdamBO outperforms other baselines on training set, and exhibits faster convergence rate. One can refer to \cref{app:hyper_param} for detailed hyper-parameter choices and experimental settings. All the experiments are run on an single NVIDIA A6000 (48GB memory) GPU and a AMD EPYC 7513 32-Core CPU.

\vspace*{-0.05in}
\subsection{Deep AUC Maximization with RNNs/Transformers}

% \begin{figure*}[!t]
% \begin{center}
% \subfigure[\scriptsize Train AUC]{\includegraphics[width=0.24\linewidth]{figures/Train_AUC_imratio_0.9.pdf}}   
% \subfigure[\scriptsize Test AUC]{\includegraphics[width=0.24\linewidth]{figures/Test_AUC_imratio_0.9.pdf}}   
% \subfigure[\scriptsize Train AUC vs. Time]{\includegraphics[width=0.24\linewidth]{figures/Train_AUC_imratio_0.9_running_time.pdf}}  
% \subfigure[\scriptsize Test AUC vs. Time]{\includegraphics[width=0.24\linewidth]{figures/Test_AUC_imratio_0.9_running_time.pdf}}
% \end{center}
% \vspace*{-0.15in}
% \caption{Transformer for AUC maximization on Sentiment140 dataset with imbalance ratio of 0.9. }
% \label{fig:auc_trans}
% \end{figure*}

The Area Under the ROC Curve (AUC)~\citep{hanley1983method} is a widely used metric for evaluating the effectiveness of binary classification models, especially in the imbalanced data scenarios. It is defined as the probability that the prediction score of a positive example is higher than that of a negative example~\citep{hanley1982meaning}. Deep AUC maximization~\citep{liu2019stochastic,ying2016stochastic} can be formulated as a min-max optimization problem~\citep{liu2019stochastic}: $
\min_{\vw\in\mathbb{R}^{d}, (a,b)\in \mathbb{R}^2} \max_{\alpha\in \mathbb{R}} f(\vw, a, b, \alpha)\ \coloneqq \mathbb{E}_{\vz}[F(\vw, a, b, \alpha; \vz)]
$, where $F(\vw, a, b, \alpha; \vz) = (1-p)(h(\vw; \vx)-a )^2 \mathbb{I}_{[c=1]} + p (h(\vw;\vx)-b)^2\mathbb{I}_{[c=-1]}+ 2(1+\alpha)(p h(\vw;\vx)\mathbb{I}_{[c=-1]}-(1-p)h(\vw; \vx)\mathbb{I}_{[c=1]}) - p(1-p)\alpha^2$, $\vw$ denotes the model parameter of a deep neural network, and $\vz = (\vx, c)$ represents a random training data sample ($\vx$ represents the feature vector and $c \in \{+1, -1\}$ represents the class label), the function $h(\vw, \vx)$ is a scoring function for the sample with feature $\vx$, and $p = \pr(c=1)$ indicates the proportion of positive samples in the population. This min-max problem can be reformulated as the form of a bilevel optimization problem with lower-level objective function $g = -f$:
%\begin{small}
\begin{equation*}
\label{formula:deepauc}
    \begin{aligned}
        &\min_{\vw\in\mathbb{R}^{d}, (a,b)\in \mathbb{R}^2} \mathbb{E}_{\vz}[F(\vw, a, b, \alpha^*(\vw,a,b); \vz)] \\
        &\text{s.t.,} \quad
        \alpha^*(\vw, a, b) \in \arg\min_{\alpha \in \mathbb{R}} \ -\mathbb{E}_{\vz}[F(\vw, a, b, \alpha; \vz)].
    \end{aligned}
\end{equation*}
%\end{small}%
In above, $(\vw, a, b)$ is the upper-level variable, and $\alpha$ is the lower-level variable. The lower-level problem is a strongly convex one-dimensional quadratic function with respect to $\alpha$, while the upper-level objective is non-convex and can exhibit unbounded smoothness when using a recurrent neural network or a transformer as the predictive model~\citep{crawshaw2022robustness,zhang2019gradient}.

In our experiment, we focus on tackling an imbalanced text classification task by maximizing the AUC metric. Specifically, we conduct experiments using deep AUC maximization on the imbalanced Sentiment140 dataset~\citep{go2009twitter}, a binary text classification benchmark. Following the approach in \citep{yuan2021large}, we introduce imbalance in the training set using a pre-specified imbalance ratio ($p$) while keeping the test set distribution unchanged. For a given $p$, we randomly remove positive samples (labeled as 1) from the training set until the desired proportion of positive examples is achieved. In our experiment, we set $p$ to 0.8 (0.9), meaning that 80$\%$ (90$\%$) of the training samples are positive examples. We run the experiment using two different models, a two-layer transformer, and a two-layer recurrent neural network (RNN) with the same input dimension of 300,  hidden dimension of 4096, and an output dimension of 2.

To evaluate the effectiveness of our proposed bilevel optimization algorithm, we compare with recent bilevel optimization baselines, including StocBio \citep{ji2021bilevel}, TTSA \citep{hong2023two}, SABA \citep{dagreou2022framework}, MA-SOBA \citep{chen2023optimal}, SUSTAIN~\citep{khanduri2021near}, VRBO \citep{yang2021provably}, BO-REP~\citep{hao2024bilevel}, SLIP \citep{gong2024a}, and AccBO \citep{gong2024accelerated}. The training and testing results of the transformer model over 50 epochs are presented in \cref{fig:auc_rnn} (a) and (b), while the corresponding running times are shown in \cref{fig:auc_rnn} (c) and (d). Our proposed Adam-type algorithms, AdamBO, shows the faster convergence rate and significantly outperform other baselines. In particular, the performance on the training AUC (testing AUC) is better by at least 14\% (7\%) over other baselines. The running time results indicate that AdamBO converges much faster to a high AUC value compared to the other baselines. We also perform the AUC maximization on a RNN model with imbalance rario of 0.8, and the results are presented in \cref{app:rnn_results}.
More detailed parameter tuning and selection can be found in \cref{app:hyper_param}.

% \begin{figure*}[!t]
% \begin{center}
% \subfigure[\scriptsize Train AUC]{\includegraphics[width=0.24\linewidth]{figures/Train_AUC_imratio_0.9.pdf}}   
% \subfigure[\scriptsize Test AUC]{\includegraphics[width=0.24\linewidth]{figures/Test_AUC_imratio_0.9.pdf}}   
% \subfigure[\scriptsize Train AUC vs. Time]{\includegraphics[width=0.24\linewidth]{figures/Train_AUC_imratio_0.9_running_time.pdf}}  
% \subfigure[\scriptsize Test AUC vs. Time]{\includegraphics[width=0.24\linewidth]{figures/Test_AUC_imratio_0.9_running_time.pdf}}
% \end{center}
% % \vspace*{-0.15in}
% \caption{Transformer for AUC maximization on Sentiment140 dataset with imbalance ratio of 0.9. }
% \label{fig:auc_trans}
% \end{figure*}

% \begin{figure*}[t]
% % \vspace*{-0.15in}
% \begin{center}
% \subfigure[\scriptsize Train ACC]{\includegraphics[width=0.24\linewidth]{figures/Train_Acc_HR.pdf}}   \   
% \subfigure[\scriptsize Test ACC ]{\includegraphics[width=0.24\linewidth]{figures/Test_Acc_HR.pdf}} \ 
% \subfigure[\scriptsize Train ACC vs. Time]{\includegraphics[width=0.24\linewidth]{figures/Train_ACC_HR_running_time.pdf}} \
% \subfigure[\scriptsize Test ACC vs. Time]{\includegraphics[width=0.24\linewidth]{figures/Test_ACC_HR_running_time.pdf}}
% \end{center}
% % \vspace*{-0.15in}
% \caption{Comparison with bilevel optimization baselines on Hyper-representation. }
% \label{fig:acc_HR}
% \end{figure*}

\vspace*{-0.02in}
\section{Conclusion}

% In this paper, we propose AdamBO for solving bilevel optimization problems. It achieves $\widetilde{O}(\epsilon^{-4})$ oracle complexity to find $\epsilon$-stationary points. The experimental results on meta-learning and AUC maximization demonstrate the superior performance of our proposed method.

In this paper, we propose an Adam-type algorithm termed AdamBO for solving bilevel optimization problems under the unbounded smoothness setting. AdamBO is a single-loop algorithm with $\widetilde{O}(\epsilon^{-4})$ oracle complexity to find $\epsilon$-stationary points. We conduct experiments on meta-learning and deep AUC maximization for text classification using transformers. The experimental results demonstrate the superior performance of our proposed method. One limitation of our analysis is that the complexity bound of AdamBO depends on $O(\lambda^{-2})$, which can be large when $\lambda$ is small. However, our empirical sensitivity analysis indicates that AdamBO's performance remains largely unaffected by the choice of $\lambda$ within a reasonable range. In the future, we plan to improve the dependency on $\lambda$ in the complexity bound.
% One limitation of our analysis is that the complexity bound of AdamBO depends on $O(\lambda^{-2})$, which can be large when $\lambda$ is small. In contrast, our experimental results on sensitivity analysis of $\lambda$ indicate that the performance of AdamBO seems to be unaffected by the choice of $\lambda$ in a reasonable range. In the future, we plan to improve the dependency on $\lambda$.

%In the future, we plan to remove the requirement of $\lambda>0$ in the Adam proof

% \section*{Acknowledgements}

% \textbf{Do not} include acknowledgements in the initial version of
% the paper submitted for blind review.

% If a paper is accepted, the final camera-ready version can (and
% usually should) include acknowledgements.  Such acknowledgements
% should be placed at the end of the section, in an unnumbered section
% that does not count towards the paper page limit. Typically, this will 
% include thanks to reviewers who gave useful comments, to colleagues 
% who contributed to the ideas, and to funding agencies and corporate 
% sponsors that provided financial support.

\section*{Impact Statement}

This paper presents work whose goal is to advance the field of Machine Learning. There are many potential societal consequences of our work, none which we feel must be specifically highlighted here.


\bibliography{ref}

\begin{thebibliography}{65}
\providecommand{\natexlab}[1]{#1}
\providecommand{\url}[1]{\texttt{#1}}
\expandafter\ifx\csname urlstyle\endcsname\relax
  \providecommand{\doi}[1]{doi: #1}\else
  \providecommand{\doi}{doi: \begingroup \urlstyle{rm}\Url}\fi

\bibitem[Ahn et~al.(2023)Ahn, Cheng, Song, Yun, Jadbabaie, and
  Sra]{ahn2023linear}
Ahn, K., Cheng, X., Song, M., Yun, C., Jadbabaie, A., and Sra, S.
\newblock Linear attention is (maybe) all you need (to understand transformer
  optimization).
\newblock \emph{arXiv preprint arXiv:2310.01082}, 2023.

\bibitem[Anandalingam \& White(1990)Anandalingam and
  White]{anandalingam1990solution}
Anandalingam, G. and White, D.
\newblock A solution method for the linear static stackelberg problem using
  penalty functions.
\newblock \emph{IEEE Transactions on automatic control}, 35\penalty0
  (10):\penalty0 1170--1173, 1990.

\bibitem[Borsos et~al.(2020)Borsos, Mutny, and Krause]{borsos2020coresets}
Borsos, Z., Mutny, M., and Krause, A.
\newblock Coresets via bilevel optimization for continual learning and
  streaming.
\newblock \emph{Advances in neural information processing systems},
  33:\penalty0 14879--14890, 2020.

\bibitem[Bowman et~al.(2015)Bowman, Angeli, Potts, and
  Manning]{bowman2015large}
Bowman, S.~R., Angeli, G., Potts, C., and Manning, C.~D.
\newblock A large annotated corpus for learning natural language inference.
\newblock \emph{arXiv preprint arXiv:1508.05326}, 2015.

\bibitem[Bracken \& McGill(1973)Bracken and McGill]{bracken1973mathematical}
Bracken, J. and McGill, J.~T.
\newblock Mathematical programs with optimization problems in the constraints.
\newblock \emph{Operations research}, 21\penalty0 (1):\penalty0 37--44, 1973.

\bibitem[Chen et~al.(2021)Chen, Sun, and Yin]{chen2021single}
Chen, T., Sun, Y., and Yin, W.
\newblock A single-timescale stochastic bilevel optimization method.
\newblock \emph{arXiv preprint arXiv:2102.04671}, 2021.

\bibitem[Chen et~al.(2023{\natexlab{a}})Chen, Xiao, and
  Balasubramanian]{chen2023optimal}
Chen, X., Xiao, T., and Balasubramanian, K.
\newblock Optimal algorithms for stochastic bilevel optimization under relaxed
  smoothness conditions.
\newblock \emph{arXiv preprint arXiv:2306.12067}, 2023{\natexlab{a}}.

\bibitem[Chen et~al.(2023{\natexlab{b}})Chen, Zhou, Liang, and
  Lu]{chen2023generalized}
Chen, Z., Zhou, Y., Liang, Y., and Lu, Z.
\newblock Generalized-smooth nonconvex optimization is as efficient as smooth
  nonconvex optimization.
\newblock \emph{arXiv preprint arXiv:2303.02854}, 2023{\natexlab{b}}.

\bibitem[Crawshaw et~al.(2022)Crawshaw, Liu, Orabona, Zhang, and
  Zhuang]{crawshaw2022robustness}
Crawshaw, M., Liu, M., Orabona, F., Zhang, W., and Zhuang, Z.
\newblock Robustness to unbounded smoothness of generalized signsgd.
\newblock \emph{Advances in neural information processing systems}, 2022.

\bibitem[Crawshaw et~al.(2023{\natexlab{a}})Crawshaw, Bao, and
  Liu]{crawshaw2023episode}
Crawshaw, M., Bao, Y., and Liu, M.
\newblock Episode: Episodic gradient clipping with periodic resampled
  corrections for federated learning with heterogeneous data.
\newblock In \emph{The Eleventh International Conference on Learning
  Representations}, 2023{\natexlab{a}}.

\bibitem[Crawshaw et~al.(2023{\natexlab{b}})Crawshaw, Bao, and
  Liu]{crawshaw2023federated}
Crawshaw, M., Bao, Y., and Liu, M.
\newblock Federated learning with client subsampling, data heterogeneity, and
  unbounded smoothness: A new algorithm and lower bounds.
\newblock In \emph{Thirty-seventh Conference on Neural Information Processing
  Systems}, 2023{\natexlab{b}}.

\bibitem[Cutkosky \& Orabona(2019)Cutkosky and Orabona]{cutkosky2019momentum}
Cutkosky, A. and Orabona, F.
\newblock Momentum-based variance reduction in non-convex sgd.
\newblock \emph{Advances in neural information processing systems}, 32, 2019.

\bibitem[Cutler et~al.(2023)Cutler, Drusvyatskiy, and
  Harchaoui]{cutler2023stochastic}
Cutler, J., Drusvyatskiy, D., and Harchaoui, Z.
\newblock Stochastic optimization under distributional drift.
\newblock \emph{Journal of Machine Learning Research}, 24\penalty0
  (147):\penalty0 1--56, 2023.

\bibitem[Dagr{\'e}ou et~al.(2022)Dagr{\'e}ou, Ablin, Vaiter, and
  Moreau]{dagreou2022framework}
Dagr{\'e}ou, M., Ablin, P., Vaiter, S., and Moreau, T.
\newblock A framework for bilevel optimization that enables stochastic and
  global variance reduction algorithms.
\newblock \emph{Advances in Neural Information Processing Systems},
  35:\penalty0 26698--26710, 2022.

\bibitem[De et~al.(2018)De, Mukherjee, and Ullah]{de2018convergence}
De, S., Mukherjee, A., and Ullah, E.
\newblock Convergence guarantees for rmsprop and adam in non-convex
  optimization and an empirical comparison to nesterov acceleration.
\newblock \emph{arXiv preprint arXiv:1807.06766}, 2018.

\bibitem[D{\'e}fossez et~al.(2020)D{\'e}fossez, Bottou, Bach, and
  Usunier]{defossez2020simple}
D{\'e}fossez, A., Bottou, L., Bach, F., and Usunier, N.
\newblock A simple convergence proof of adam and adagrad.
\newblock \emph{arXiv preprint arXiv:2003.02395}, 2020.

\bibitem[Dempe(2002)]{dempe2002foundations}
Dempe, S.
\newblock \emph{Foundations of bilevel programming}.
\newblock Springer Science \& Business Media, 2002.

\bibitem[Devlin et~al.(2018)Devlin, Chang, Lee, and Toutanova]{devlin2018bert}
Devlin, J., Chang, M.-W., Lee, K., and Toutanova, K.
\newblock Bert: Pre-training of deep bidirectional transformers for language
  understanding.
\newblock \emph{arXiv preprint arXiv:1810.04805}, 2018.

\bibitem[Dosovitskiy et~al.(2021)Dosovitskiy, Beyer, Kolesnikov, Weissenborn,
  Zhai, Unterthiner, Dehghani, Minderer, Heigold, Gelly, Uszkoreit, and
  Houlsby]{dosovitskiy2021an}
Dosovitskiy, A., Beyer, L., Kolesnikov, A., Weissenborn, D., Zhai, X.,
  Unterthiner, T., Dehghani, M., Minderer, M., Heigold, G., Gelly, S.,
  Uszkoreit, J., and Houlsby, N.
\newblock An image is worth 16x16 words: Transformers for image recognition at
  scale.
\newblock In \emph{International Conference on Learning Representations}, 2021.

\bibitem[Elman(1990)]{elman1990finding}
Elman, J.~L.
\newblock Finding structure in time.
\newblock \emph{Cognitive science}, 14\penalty0 (2):\penalty0 179--211, 1990.

\bibitem[Faw et~al.(2023)Faw, Rout, Caramanis, and Shakkottai]{faw2023beyond}
Faw, M., Rout, L., Caramanis, C., and Shakkottai, S.
\newblock Beyond uniform smoothness: A stopped analysis of adaptive sgd.
\newblock \emph{arXiv preprint arXiv:2302.06570}, 2023.

\bibitem[Feurer \& Hutter(2019)Feurer and Hutter]{feurer2019hyperparameter}
Feurer, M. and Hutter, F.
\newblock Hyperparameter optimization.
\newblock \emph{Automated machine learning: Methods, systems, challenges}, pp.\
   3--33, 2019.

\bibitem[Finn et~al.(2017)Finn, Abbeel, and Levine]{finn2017model}
Finn, C., Abbeel, P., and Levine, S.
\newblock Model-agnostic meta-learning for fast adaptation of deep networks.
\newblock In \emph{International conference on machine learning}, pp.\
  1126--1135. PMLR, 2017.

\bibitem[Franceschi et~al.(2018)Franceschi, Frasconi, Salzo, Grazzi, and
  Pontil]{franceschi2018bilevel}
Franceschi, L., Frasconi, P., Salzo, S., Grazzi, R., and Pontil, M.
\newblock Bilevel programming for hyperparameter optimization and
  meta-learning.
\newblock In \emph{International conference on machine learning}, pp.\
  1568--1577. PMLR, 2018.

\bibitem[Ghadimi \& Wang(2018)Ghadimi and Wang]{ghadimi2018approximation}
Ghadimi, S. and Wang, M.
\newblock Approximation methods for bilevel programming.
\newblock \emph{arXiv preprint arXiv:1802.02246}, 2018.

\bibitem[Go et~al.(2009)Go, Bhayani, and Huang]{go2009twitter}
Go, A., Bhayani, R., and Huang, L.
\newblock Twitter sentiment classification using distant supervision.
\newblock \emph{CS224N project report, Stanford}, 1\penalty0 (12):\penalty0
  2009, 2009.

\bibitem[Gong et~al.(2024{\natexlab{a}})Gong, Hao, and Liu]{gong2024a}
Gong, X., Hao, J., and Liu, M.
\newblock A nearly optimal single loop algorithm for stochastic bilevel
  optimization under unbounded smoothness.
\newblock In \emph{Forty-first International Conference on Machine Learning},
  2024{\natexlab{a}}.

\bibitem[Gong et~al.(2024{\natexlab{b}})Gong, Hao, and
  Liu]{gong2024accelerated}
Gong, X., Hao, J., and Liu, M.
\newblock An accelerated algorithm for stochastic bilevel optimization under
  unbounded smoothness.
\newblock \emph{arXiv preprint arXiv:2409.19212}, 2024{\natexlab{b}}.

\bibitem[Guo et~al.(2021{\natexlab{a}})Guo, Hu, Zhang, and
  Yang]{guo2021randomized}
Guo, Z., Hu, Q., Zhang, L., and Yang, T.
\newblock Randomized stochastic variance-reduced methods for multi-task
  stochastic bilevel optimization.
\newblock \emph{arXiv preprint arXiv:2105.02266}, 2021{\natexlab{a}}.

\bibitem[Guo et~al.(2021{\natexlab{b}})Guo, Xu, Yin, Jin, and
  Yang]{guo2021novel}
Guo, Z., Xu, Y., Yin, W., Jin, R., and Yang, T.
\newblock A novel convergence analysis for algorithms of the adam family.
\newblock \emph{arXiv preprint arXiv:2112.03459}, 2021{\natexlab{b}}.

\bibitem[Hanley \& McNeil(1982)Hanley and McNeil]{hanley1982meaning}
Hanley, J.~A. and McNeil, B.~J.
\newblock The meaning and use of the area under a receiver operating
  characteristic (roc) curve.
\newblock \emph{Radiology}, 143\penalty0 (1):\penalty0 29--36, 1982.

\bibitem[Hanley \& McNeil(1983)Hanley and McNeil]{hanley1983method}
Hanley, J.~A. and McNeil, B.~J.
\newblock A method of comparing the areas under receiver operating
  characteristic curves derived from the same cases.
\newblock \emph{Radiology}, 148\penalty0 (3):\penalty0 839--843, 1983.

\bibitem[Hao et~al.(2023)Hao, Ji, and Liu]{hao2023bilevelcoreset}
Hao, J., Ji, K., and Liu, M.
\newblock Bilevel coreset selection in continual learning: A new formulation
  and algorithm.
\newblock \emph{Advances in Neural Information Processing Systems}, 36, 2023.

\bibitem[Hao et~al.(2024)Hao, Gong, and Liu]{hao2024bilevel}
Hao, J., Gong, X., and Liu, M.
\newblock Bilevel optimization under unbounded smoothness: A new algorithm and
  convergence analysis.
\newblock In \emph{The Twelfth International Conference on Learning
  Representations}, 2024.

\bibitem[Hong et~al.(2023)Hong, Wai, Wang, and Yang]{hong2023two}
Hong, M., Wai, H.-T., Wang, Z., and Yang, Z.
\newblock A two-timescale stochastic algorithm framework for bilevel
  optimization: Complexity analysis and application to actor-critic.
\newblock \emph{SIAM Journal on Optimization}, 33\penalty0 (1):\penalty0
  147--180, 2023.

\bibitem[Ji et~al.(2021)Ji, Yang, and Liang]{ji2021bilevel}
Ji, K., Yang, J., and Liang, Y.
\newblock Bilevel optimization: Convergence analysis and enhanced design.
\newblock In \emph{International conference on machine learning}, pp.\
  4882--4892. PMLR, 2021.

\bibitem[Jin et~al.(2021)Jin, Zhang, Wang, and Wang]{jin2021non}
Jin, J., Zhang, B., Wang, H., and Wang, L.
\newblock Non-convex distributionally robust optimization: Non-asymptotic
  analysis.
\newblock \emph{Advances in Neural Information Processing Systems},
  34:\penalty0 2771--2782, 2021.

\bibitem[Khanduri et~al.(2021)Khanduri, Zeng, Hong, Wai, Wang, and
  Yang]{khanduri2021near}
Khanduri, P., Zeng, S., Hong, M., Wai, H.-T., Wang, Z., and Yang, Z.
\newblock A near-optimal algorithm for stochastic bilevel optimization via
  double-momentum.
\newblock \emph{Advances in neural information processing systems},
  34:\penalty0 30271--30283, 2021.

\bibitem[Kingma \& Ba(2014)Kingma and Ba]{kingma2014adam}
Kingma, D.~P. and Ba, J.
\newblock Adam: A method for stochastic optimization.
\newblock \emph{International Conference on Learning Representations (ICLR)},
  2014.

\bibitem[Konda \& Tsitsiklis(2000)Konda and Tsitsiklis]{konda2000actor}
Konda, V.~R. and Tsitsiklis, J.~N.
\newblock Actor-critic algorithms.
\newblock In \emph{Advances in neural information processing systems
  (NeurIPS)}, pp.\  1008--1014, 2000.

\bibitem[Kunstner et~al.(2023)Kunstner, Chen, Lavington, and
  Schmidt]{kunstner2023noise}
Kunstner, F., Chen, J., Lavington, J.~W., and Schmidt, M.
\newblock Noise is not the main factor behind the gap between sgd and adam on
  transformers, but sign descent might be.
\newblock \emph{arXiv preprint arXiv:2304.13960}, 2023.

\bibitem[Kwon et~al.(2023)Kwon, Kwon, Wright, and Nowak]{kwon2023fully}
Kwon, J., Kwon, D., Wright, S., and Nowak, R.~D.
\newblock A fully first-order method for stochastic bilevel optimization.
\newblock In \emph{International Conference on Machine Learning}, pp.\
  18083--18113. PMLR, 2023.

\bibitem[Li et~al.(2023{\natexlab{a}})Li, Jadbabaie, and
  Rakhlin]{li2023convergence}
Li, H., Jadbabaie, A., and Rakhlin, A.
\newblock Convergence of adam under relaxed assumptions.
\newblock \emph{arXiv preprint arXiv:2304.13972}, 2023{\natexlab{a}}.

\bibitem[Li et~al.(2023{\natexlab{b}})Li, Qian, Tian, Rakhlin, and
  Jadbabaie]{li2023convex}
Li, H., Qian, J., Tian, Y., Rakhlin, A., and Jadbabaie, A.
\newblock Convex and non-convex optimization under generalized smoothness.
\newblock \emph{arXiv preprint arXiv:2306.01264}, 2023{\natexlab{b}}.

\bibitem[Li \& Roth(2002)Li and Roth]{li-roth-2002-learning}
Li, X. and Roth, D.
\newblock Learning question classifiers.
\newblock In \emph{{COLING} 2002: The 19th International Conference on
  Computational Linguistics}, 2002.
\newblock URL \url{https://www.aclweb.org/anthology/C02-1150}.

\bibitem[Liu et~al.(2020)Liu, Yuan, Ying, and Yang]{liu2019stochastic}
Liu, M., Yuan, Z., Ying, Y., and Yang, T.
\newblock Stochastic auc maximization with deep neural networks.
\newblock \emph{International Conference on Learning Representations}, 2020.

\bibitem[Liu et~al.(2022)Liu, Zhuang, Lei, and Liao]{liu2022communication}
Liu, M., Zhuang, Z., Lei, Y., and Liao, C.
\newblock A communication-efficient distributed gradient clipping algorithm for
  training deep neural networks.
\newblock \emph{Advances in Neural Information Processing Systems},
  35:\penalty0 26204--26217, 2022.

\bibitem[Liu et~al.(2023)Liu, Jagabathula, and Zhou]{liu2023near}
Liu, Z., Jagabathula, S., and Zhou, Z.
\newblock Near-optimal non-convex stochastic optimization under generalized
  smoothness.
\newblock \emph{arXiv preprint arXiv:2302.06032}, 2023.

\bibitem[Raghu et~al.(2019)Raghu, Raghu, Bengio, and Vinyals]{raghu2019rapid}
Raghu, A., Raghu, M., Bengio, S., and Vinyals, O.
\newblock Rapid learning or feature reuse? towards understanding the
  effectiveness of maml.
\newblock \emph{arXiv preprint arXiv:1909.09157}, 2019.

\bibitem[Rajeswaran et~al.(2019)Rajeswaran, Finn, Kakade, and
  Levine]{rajeswaran2019meta}
Rajeswaran, A., Finn, C., Kakade, S.~M., and Levine, S.
\newblock Meta-learning with implicit gradients.
\newblock \emph{Advances in neural information processing systems}, 32, 2019.

\bibitem[Reddi et~al.(2019)Reddi, Kale, and Kumar]{reddi2019convergence}
Reddi, S.~J., Kale, S., and Kumar, S.
\newblock On the convergence of adam and beyond.
\newblock \emph{arXiv preprint arXiv:1904.09237}, 2019.

\bibitem[Reisizadeh et~al.(2023)Reisizadeh, Li, Das, and
  Jadbabaie]{reisizadeh2023variance}
Reisizadeh, A., Li, H., Das, S., and Jadbabaie, A.
\newblock Variance-reduced clipping for non-convex optimization.
\newblock \emph{arXiv preprint arXiv:2303.00883}, 2023.

\bibitem[Vaswani et~al.(2017)Vaswani, Shazeer, Parmar, Uszkoreit, Jones, Gomez,
  Kaiser, and Polosukhin]{vaswani2017attention}
Vaswani, A., Shazeer, N., Parmar, N., Uszkoreit, J., Jones, L., Gomez, A.~N.,
  Kaiser, {\L}., and Polosukhin, I.
\newblock Attention is all you need.
\newblock \emph{Advances in neural information processing systems}, 30, 2017.

\bibitem[Vicente et~al.(1994)Vicente, Savard, and
  J{\'u}dice]{vicente1994descent}
Vicente, L., Savard, G., and J{\'u}dice, J.
\newblock Descent approaches for quadratic bilevel programming.
\newblock \emph{Journal of optimization theory and applications}, 81\penalty0
  (2):\penalty0 379--399, 1994.

\bibitem[Wang et~al.(2022)Wang, Zhang, Zhang, Meng, Ma, Liu, and
  Chen]{wang2022provable}
Wang, B., Zhang, Y., Zhang, H., Meng, Q., Ma, Z.-M., Liu, T.-Y., and Chen, W.
\newblock Provable adaptivity in adam.
\newblock \emph{arXiv preprint arXiv:2208.09900}, 2022.

\bibitem[Wang et~al.(2023)Wang, Zhang, Ma, and Chen]{wang2023convergence}
Wang, B., Zhang, H., Ma, Z., and Chen, W.
\newblock Convergence of adagrad for non-convex objectives: Simple proofs and
  relaxed assumptions.
\newblock In \emph{The Thirty Sixth Annual Conference on Learning Theory}, pp.\
   161--190. PMLR, 2023.

\bibitem[White \& Anandalingam(1993)White and Anandalingam]{white1993penalty}
White, D.~J. and Anandalingam, G.
\newblock A penalty function approach for solving bi-level linear programs.
\newblock \emph{Journal of Global Optimization}, 3:\penalty0 397--419, 1993.

\bibitem[Yang et~al.(2021)Yang, Ji, and Liang]{yang2021provably}
Yang, J., Ji, K., and Liang, Y.
\newblock Provably faster algorithms for bilevel optimization.
\newblock \emph{Advances in Neural Information Processing Systems},
  34:\penalty0 13670--13682, 2021.

\bibitem[Ying et~al.(2016)Ying, Wen, and Lyu]{ying2016stochastic}
Ying, Y., Wen, L., and Lyu, S.
\newblock Stochastic online auc maximization.
\newblock In \emph{Advances in Neural Information Processing Systems}, pp.\
  451--459, 2016.

\bibitem[Yuan et~al.(2021)Yuan, Yan, Sonka, and Yang]{yuan2021large}
Yuan, Z., Yan, Y., Sonka, M., and Yang, T.
\newblock Large-scale robust deep auc maximization: A new surrogate loss and
  empirical studies on medical image classification.
\newblock In \emph{Proceedings of the IEEE/CVF International Conference on
  Computer Vision}, pp.\  3040--3049, 2021.

\bibitem[Zhang et~al.(2020{\natexlab{a}})Zhang, Jin, Fang, and
  Wang]{zhang2020improved}
Zhang, B., Jin, J., Fang, C., and Wang, L.
\newblock Improved analysis of clipping algorithms for non-convex optimization.
\newblock \emph{Advances in Neural Information Processing Systems},
  2020{\natexlab{a}}.

\bibitem[Zhang et~al.(2019)Zhang, Karimireddy, Veit, Kim, Reddi, Kumar, and
  Sra]{zhang2019adaptive}
Zhang, J., Karimireddy, S.~P., Veit, A., Kim, S., Reddi, S.~J., Kumar, S., and
  Sra, S.
\newblock Why are adaptive methods good for attention models?
\newblock \emph{arXiv preprint arXiv:1912.03194}, 2019.

\bibitem[Zhang et~al.(2020{\natexlab{b}})Zhang, He, Sra, and
  Jadbabaie]{zhang2019gradient}
Zhang, J., He, T., Sra, S., and Jadbabaie, A.
\newblock Why gradient clipping accelerates training: A theoretical
  justification for adaptivity.
\newblock \emph{International Conference on Learning Representations},
  2020{\natexlab{b}}.

\bibitem[Zhang et~al.(2022)Zhang, Chen, Shi, Sun, and Luo]{zhang2022adam}
Zhang, Y., Chen, C., Shi, N., Sun, R., and Luo, Z.-Q.
\newblock Adam can converge without any modification on update rules.
\newblock \emph{Advances in neural information processing systems},
  35:\penalty0 28386--28399, 2022.

\bibitem[Zhou et~al.(2018)Zhou, Zhang, Lu, Wang, Zhang, and
  Yu]{zhou2018adashift}
Zhou, Z., Zhang, Q., Lu, G., Wang, H., Zhang, W., and Yu, Y.
\newblock Adashift: Decorrelation and convergence of adaptive learning rate
  methods.
\newblock \emph{arXiv preprint arXiv:1810.00143}, 2018.

\end{thebibliography}
\bibliographystyle{icml2025}

%%%%%%%%%%%%%%%%%%%%%%%%%%%%%%%%%%%%%%%%%%%%%%%%%%%%%%%%%%%%%%%%%%%%%%%%%%%%%%%
%%%%%%%%%%%%%%%%%%%%%%%%%%%%%%%%%%%%%%%%%%%%%%%%%%%%%%%%%%%%%%%%%%%%%%%%%%%%%%%
% APPENDIX
%%%%%%%%%%%%%%%%%%%%%%%%%%%%%%%%%%%%%%%%%%%%%%%%%%%%%%%%%%%%%%%%%%%%%%%%%%%%%%%
%%%%%%%%%%%%%%%%%%%%%%%%%%%%%%%%%%%%%%%%%%%%%%%%%%%%%%%%%%%%%%%%%%%%%%%%%%%%%%%
\newpage
\appendix
\onecolumn

% \newpage

% \tableofcontents

\newpage

\section{Equivalent Update Rule of AdamBO (\Cref{alg:bi-adam})}
\label{app:equiv_biadam}
\begin{algorithm}[!t]
    \caption{\textsc{SGD}} \label{alg:sgd}
    \begin{algorithmic}[1]
        \STATE \textbf{Input:} $x, y_0, \gamma, T_0$ \hfill \# \texttt{SGD}$(x, y_0, \gamma, T_0)$
        \STATE \textbf{Initialize} $y_0^{\init} = y_0$
        \FOR{$t=0, 1, \dots, T_0-1$}
            \STATE Sample $\pi_t$ from distribution $\gD_g$
            \STATE $y_{t+1}^{\init} = y_t^{\init} - \gamma\gdy G(x, y_t^{\init};\pi_t)$
        \ENDFOR
    \end{algorithmic}
\end{algorithm}

% \begin{algorithm}[tbh]
%     \caption{\textsc{AdamBO}}\label{alg:bi-adam}
%     \begin{algorithmic}[1]
%         \STATE \textbf{Input:} $\beta,\betasq,\eta,\gamma,\lambda, T_0, T, x_1, y_0$
%         \STATE \textbf{Initialize} $y_1=\texttt{SGD}(x_1, y_0, \gamma, T_0)$, $\hm_1=\hatphi(x_1,y_1;\Bar{\xi}_1)$ and $\hv_1=(\hatphi(x_1,y_1;\Bar{\xi}_1))^2$
% 	\FOR{$t=1,\dots,T$}
%             % \STATE Draw a new sample $\xi_t$ and perform the following updates
%             % \STATE $\alpha_t = \frac{\beta}{1-(1-\beta)^t}$, $\alphasqt = \frac{\betasq}{1-(1-\betasq)^t}$
%             % \STATE Draw new samples and perform the following updates
%             \STATE $y_{t+1} = y_t - \gamma\gdy G(x_t,y_t;\zeta_t)$
%             % \STATE $h_t = (1-\gamma)h_{t-1} + \gamma \hatphi(x_t)$
%     	\STATE $m_t = (1-\beta)m_{t-1} + \beta \hatphi(x_t,y_t;\Bar{\xi}_t)$
%     	\STATE $v_t=(1-\betasq)v_{t-1} + \betasq (\hatphi(x_t,y_t;\Bar{\xi}_t))^2$
%             \STATE $\hm_t = \frac{m_t}{1-(1-\beta)^t}$
%             \STATE $\hv_t = \frac{v_t}{1-(1-\betasq)^t}$
%     	\STATE $x_{t+1} = x_t-\frac{\eta}{\sqrt{\hv_t}+\lambda}\odot \hm_t$
% 	\ENDFOR
%     \end{algorithmic}
% \end{algorithm}

\begin{algorithm}[!t]
    \caption{\textsc{AdamBO} (Equivalent update rule of \Cref{alg:bi-adam})} \label{alg:bi-adam-equiv}
    \begin{algorithmic}[1]
        \STATE \textbf{Input:} $\beta,\betasq,\eta,\gamma,\lambda, T_0, T, x_1, y_0$
        \STATE \textbf{Initialize} $y_1=\texttt{SGD}(x_1, y_0, \gamma, T_0)$, $\hm_1=\hatphi(x_1,y_1;\Bar{\xi}_1)$ and $\hv_1=(\hatphi(x_1,y_1;\Bar{\xi}_1))^2$
	\FOR{$t=1,\dots,T$}
            % \STATE Draw a new sample $\xi_t$ and perform the following updates
            \STATE $\alpha_t = \frac{\beta}{1-(1-\beta)^t}$, $\alphasqt = \frac{\betasq}{1-(1-\betasq)^t}$
            \STATE Draw new samples and perform the following updates
            \STATE $y_{t+1} = y_t - \gamma\gdy G(x_t,y_t;\zeta_t)$
            % \STATE $h_t = (1-\gamma)h_{t-1} + \gamma \hatphi(x_t)$
    	\STATE $\hm_t = (1-\alpha_t)\hm_{t-1} + \alpha_t \hatphi(x_t,y_t;\Bar{\xi}_t)$
    	\STATE $\hv_t=(1-\alphasqt)\hv_{t-1}+\alphasqt (\hatphi(x_t,y_t;\Bar{\xi}_t))^2$
    	\STATE $x_{t+1} = x_t-\frac{\eta}{\sqrt{\hv_t}+\lambda}\odot \hm_t$
	\ENDFOR
    \end{algorithmic}
\end{algorithm}

% \begin{algorithm}[!t]
%     \caption{\textsc{Stochastic Nesterov Accelerated Gradient Method (SNAG)}} \label{alg:snag}
%     \begin{algorithmic}[1]
%         \STATE \textbf{Input:} $\tx, \ty_{-1}, \gamma, T_0$ \hfill \# \texttt{SNAG}$(x, y_0, \gamma, T_0)$
%         \STATE \textbf{Initialize} $\ty_0=\ty_{-1}$
%         \FOR{$t=0, 1, \dots, T_0-1$}
%             \STATE Sample $\Tilde{\pi}_t$ from distribution $\gD_g$
%             \STATE $\tz_t = \ty_t + \alpha(\ty_t-\ty_{t-1})$
%             \STATE $\ty_{t+1} = \tz_t - \gamma\gdy G(\tx,\tz_t;\Tilde{\pi}_t)$
%         \ENDFOR
%     \end{algorithmic}
% \end{algorithm}

%%%%%%%%%%%%%%%%%%%%%%%%%%%%%%%%%%%%%%%%%%%%%%%%%%%%%%%%%%%%%%%%%%%%%%%%%%%%%%%%%%%%%%%%%%%
In this section, we aim to provide a simplified version of the bias correction steps (lines 7-8) of \cref{alg:bi-adam}. Inspired by \citep[Appendix C.1]{li2023convergence}, we present an equivalent yet simpler update rule of \cref{alg:bi-adam} in the following \cref{prop:bi-adam-equiv}. The detailed equivalent framework is also outlined in \Cref{alg:bi-adam-equiv}.

\begin{proposition} \label{prop:bi-adam-equiv}
Let $\alpha_t = \frac{\beta}{1-(1-\beta)^t}$ and $\alphasqt = \frac{\betasq}{1-(1-\betasq)^t}$. Then the update rule in Bi-Adam (\Cref{alg:bi-adam}) is equivalent to that in \Cref{alg:bi-adam-equiv}:
\begin{equation} \label{eq:adam-equiv}
    \begin{aligned}
        &y_{t+1} = y_t - \gamma\gdy G(x_t,y_t;\zeta_t), \\
        &\hm_t = (1-\alpha_t)\hm_{t-1} + \alpha_t \hatphi(x_t,y_t;\Bar{\xi}_t), \\
        &\hv_t = (1-\alphasqt)\hv_{t-1}+\alphasqt (\hatphi(x_t,y_t;\Bar{\xi}_t))^2, \\
        &x_{t+1} = x_t-\frac{\eta}{\sqrt{\hv_t}+\lambda}\odot \hm_t,
    \end{aligned}
\end{equation}
where initially we set $\hm_1=\hatphi(x_1,y_1;\Bar{\xi}_1)$ and $\hv_1=(\hatphi(x_1,y_1;\Bar{\xi}_1))^2$. There is no need to define $\hm_0$ and $\hv_0$ since $1-\alpha_1=1-\alphasq_1=0$.
\end{proposition}

\begin{proof}[Proof of \Cref{prop:bi-adam-equiv}]
We follow the same proof as in \citep[Proposition E.1]{li2023convergence}, but replace the stochastic gradient $\nabla f(x_t,\xi_t)$ in \citep{li2023convergence} with the stochastic hypergradient estimator $\hatphi(x_t,y_t;\Bar{\xi}_t)$ in our setting. We still provide the proof here for completeness.

Let $Z_t = 1-(1-\beta)^t$. Then we know that $\alpha_t=\beta/Z_t$ and $m_t=Z_t\hm_t$. By line 6 of \Cref{alg:bi-adam} (the momentum update rule for $m_t$), we have
\begin{equation*}
    Z_t\hm_t = (1-\beta)Z_{t-1}\hm_{t-1} + \beta\hatphi(x_t,y_t;\Bar{\xi}_t).
\end{equation*}
Note that $Z_t$ satisfies the following property
\begin{equation*}
    (1-\beta)Z_{t-1} = 1-\beta - (1-\beta)^t = Z_t - \beta.
\end{equation*}
Then we have
\begin{equation*}
    \begin{aligned}
        \hm_t 
        &= \frac{Z_t-\beta}{Z_t}\hm_{t-1} + \frac{\beta}{Z_t}\hatphi(x_t,y_t;\Bar{\xi}_t) \\
        &= (1-\alpha_t)\hm_{t-1} + \alpha_t\hatphi(x_t,y_t;\Bar{\xi}_t).
    \end{aligned}
\end{equation*}
Next, we verify the initial condition. By \Cref{alg:bi-adam}, since we set $m_0=0$, then we have $m_1=\beta\hatphi(x_1,y_1;\Bar{\xi}_1)$. Therefore, we have $\hm_1=m_1/Z_1=\hatphi(x_1,y_1;\Bar{\xi}_1)$ since $Z_1=\beta$. Then the proof is completed by applying the same analysis on $v_t$ and $\hv_t$.
\end{proof}

\section{Technical Lemmas}
\label{app:tech_lemmas}
In this section, we present several useful algebraic facts (\cref{sec:algebraic-facts}), probabilistic lemmas (\cref{sec:prob-lemma}), and auxiliary lemmas for bilevel optimization under the unbounded smoothness setting (\cref{sec:auxiliary-bilevel}).

\subsection{Useful Algebraic Facts}
\label{sec:algebraic-facts}

In this section, we will frequently use $\alpha_t$ and $\alphasqt$, so we restate their definitions here for the reader's convenience:
\begin{equation} \label{eq:alphat-def}
    \alpha_t = \frac{\beta}{1-(1-\beta)^t}
    \quad\quad\text{and}\quad\quad
    \alphasqt = \frac{\betasq}{1-(1-\betasq)^t}.
\end{equation}

The following two lemmas, i.e., \cref{lm:prod-diff,lm:sum-geometric-sequence}, are useful for bounding the norm of the difference between Neumann series approximation matrices in \cref{sec:auxiliary-bilevel}.

\begin{lemma} \label{lm:prod-diff}
For any matrix sequences $\{A_i\}_{i=1}^{k}$ and $\{B_i\}_{i=1}^{k}$ (where $k\geq 1$), it holds that
\begin{equation*}
    \left\|\prod_{i=1}^{k}A_i - \prod_{i=1}^{k}B_i\right\| = \sum_{i=1}^{k}\|B_1\|\cdots \|B_{i-1}\|\|A_i-B_i\|\|A_{i+1}\|\cdots \|A_k\|,
\end{equation*}
where we use the convention $A_{k+1}=B_0=I$.
\end{lemma}

\begin{proof}[Proof of \cref{lm:prod-diff}]
It is easy to check that
\begin{equation*}
    \begin{aligned}
        \prod_{i=1}^{k}A_i - \prod_{i=1}^{k}B_i
        &= A_1\cdots A_k - B_1\cdots B_k \\
        &= (A_1-B_1)A_2\cdots A_k + B_1(A_2-B_2)A_3\cdots A_k + \cdots + B_1\cdots B_{k-1}(A_k-B_k) \\
        &= \sum_{i=1}^{k}B_1\cdots B_{i-1}(A_i-B_i)A_{i+1}\cdots A_k,
    \end{aligned}
\end{equation*}
where we set $A_{k+1}=B_0=I$ in the last equality. The result follows by noting that the operator norm is submultiplicative.
\end{proof}

%%%%%%%%%%%%%%%%%%%%%%%%%%%%%%%%%%%%%%%%%%%%%%%%%%%%%%%%%%%%%%%%%%%%%%%%%%%%%%%%%%%%%%%%%%%%%%%

\begin{lemma} \label{lm:sum-geometric-sequence}
For any $Q\geq 1$ and $a\in(0,1)$, we have
\begin{equation*}
    \begin{aligned}
        \sum_{q=0}^{Q-1}q\cdot a^{q-1} \leq \frac{1}{(1-a)^2}.
    \end{aligned}
\end{equation*}
\end{lemma}

\begin{proof}[Proof of \Cref{lm:sum-geometric-sequence}]
We obtain the result by simple calculation:
\begin{equation*}
    \begin{aligned}
        \sum_{q=0}^{Q-1}q\cdot a^{q-1}
        &= \frac{1-Qa^{Q-1}+(Q-1)a^Q}{(1-a)^2}
        \leq \frac{1-Qa^{Q-1}+(Q-1)a^{Q-1}}{(1-a)^2} \\
        &= \frac{1-a^{Q-1}}{(1-a)^2} 
        \leq \frac{1}{(1-a)^2}.
    \end{aligned}
\end{equation*}
\end{proof}

%%%%%%%%%%%%%%%%%%%%%%%%%%%%%%%%%%%%%%%%%%%%%%%%%%%%%%%%%%%%%%%%%%%%%%%%%%%%%%%%%%%%%%%%%%%%%%%
The next four lemmas, \cref{lm:dj-def,lm:algebra-log,lm:algebra-talphat,lm:algebra-sum-beta-alphai}, are useful for controlling the lower-level estimation error and for proving the randomness decoupling lemma (i.e., \cref{lm:any-sequence-main}) in \cref{app:any_sequence}.

\begin{lemma} \label{lm:dj-def}
For any $t\geq 1$, define $\{d_{t,j}\}_{j=0}^{t}$ as the following:
\begin{equation} \label{eq:dj-def}
    d_{t,j} = 
    \begin{cases}
        \prod_{i=1}^{t}(1-\alpha_i), & j=0 \\
        \alpha_j\prod_{i=j+1}^{t}(1-\alpha_i), & 1\leq j \leq t-1 \\
        \alpha_t, & j=t.
    \end{cases}
\end{equation}
Then $\{d_{t,j}\}_{j=0}^{t}$ has the following properties:
\begin{itemize}
    \item For $j=0$, $d_{t,j}=0$.
    \item For $1\leq j \leq t$, $d_{t,j}=\alpha_t(1-\beta)^{t-j}$.
    \item $\sum_{j=0}^{t}d_{t,j} = \sum_{j=1}^{t}d_{t,j} = 1$.
\end{itemize}
\end{lemma}

\begin{proof}[Proof of \Cref{lm:dj-def}]
Recall the definition of $\alpha_t$ in \Cref{alg:bi-adam-equiv}, we have
\begin{equation*}
    \alpha_t = \frac{\beta}{1-(1-\beta)^t}
    \quad\quad\text{and}\quad\quad
    1-\alpha_t = \frac{1-(1-\beta)^{t-1}}{1-(1-\beta)^t}(1-\beta).
\end{equation*}
It is obvious to see $\alpha_1=1$, then for $j=0$ we have
\begin{equation*}
    d_{t,0} = \prod_{i=1}^{t}(1-\alpha_i) = (1-\alpha_t)\cdots(1-\alpha_1) = 0.
\end{equation*}
For $1\leq j \leq t-1$ we have
\begin{equation*}
    \begin{aligned}
        d_{t,j} 
        = \alpha_j\prod_{i=j+1}^{t}(1-\alpha_i) = \frac{\beta}{1-(1-\beta)^j}\prod_{i=j+1}^{t}\frac{1-(1-\beta)^{i-1}}{1-(1-\beta)^i}(1-\beta) 
        = \alpha_t(1-\beta)^{t-j}.
    \end{aligned}
\end{equation*}
For $j=t$ we have
\begin{equation*}
    d_{t,t} = \alpha_t = \frac{\beta}{1-(1-\beta)^t} = \alpha_t(1-\beta)^{t-t}.
\end{equation*}
For the last result of the lemma, we have
\begin{equation*}
    \sum_{j=0}^{t}d_{t,j} = \sum_{j=1}^{t}d_{t,j} = \sum_{j=1}^{t}\alpha_t(1-\beta)^{t-j} = \frac{1-(1-\beta)^t}{\beta}\alpha_t = 1,
\end{equation*}
where we use $d_{t,0}=0$ in the first equality.
\end{proof}

%%%%%%%%%%%%%%%%%%%%%%%%%%%%%%%%%%%%%%%%%%%%%%%%%%%%%%%%%%%%%%%%%%%%%%%%%%%%%%%%%%%%%%%%%%%%%%%

\begin{lemma} \label{lm:algebra-log}
For any $x\in(0,1]$, we have
\begin{equation*}
    1-\frac{1}{x} \leq \ln x \leq x-1.
\end{equation*}
Consequently, for any $\beta\in[0,1)$ we have
\begin{equation*}
    \begin{aligned}
        -\frac{\beta}{1-\beta} \leq \ln(1-\beta) \leq -\beta 
        \quad\quad\text{and}\quad\quad
        \beta \leq -\ln(1-\beta) \leq \frac{\beta}{1-\beta}.
    \end{aligned}
\end{equation*}
\end{lemma}

\begin{proof}[Proof of \cref{lm:algebra-log}]
This is a well-known logarithm inequality, so we omit the proof here.
\end{proof}

%%%%%%%%%%%%%%%%%%%%%%%%%%%%%%%%%%%%%%%%%%%%%%%%%%%%%%%%%%%%%%%%%%%%%%%%%%%%%%%%%%%%%%%%%%%%%%%

\begin{lemma} \label{lm:algebra-talphat}
For any $t\geq 1$, we have
\begin{equation*}
    \begin{aligned}
        t\alpha_t(1-\beta)^{t-1} \leq 1.
    \end{aligned}
\end{equation*}
\end{lemma}

\begin{proof}[Proof of \Cref{lm:algebra-talphat}]
By definition of $\alpha_t$, we have
\begin{equation*}
    t\alpha_t(1-\beta)^{t-1} = \frac{\beta t(1-\beta)^{t-1}}{1-(1-\beta)^t}.
\end{equation*}
Let $f : \R \rightarrow \R$ be
\begin{equation*}
    f(t) = \frac{\beta t(1-\beta)^{t-1}}{1-(1-\beta)^t}.
\end{equation*}
Then we have
\begin{equation*}
    f'(t) = \frac{\beta(1-\beta)^{t-1}}{(1-(1-\beta)^t)^2}(1 - (1-\beta)^t + t\ln(1-\beta)).
\end{equation*}
Let $g : \R \rightarrow \R$ be
\begin{equation*}
    g(t) = 1 - (1-\beta)^t + t\ln(1-\beta).
\end{equation*}
Then we have
\begin{equation*}
    g'(t) = (1 - (1-\beta)^t)\ln(1-\beta) \leq 0.
\end{equation*}
Note that \Cref{lm:algebra-log} gives $g(1) = \beta + \ln(1-\beta) \leq 0$, then for any $t\geq 1$ we have $g(t)\leq g(1)\leq 0$, and
\begin{equation*}
    f'(t) = \frac{\beta(1-\beta)^{t-1}}{(1-(1-\beta)^t)^2}g(t) \leq 0.
\end{equation*}
Therefore, for any $t\geq 1$ we conclude that
\begin{equation*}
    t\alpha_t(1-\beta)^{t-1} = f(t) \leq f(1) = 1.
\end{equation*}
\end{proof}

%%%%%%%%%%%%%%%%%%%%%%%%%%%%%%%%%%%%%%%%%%%%%%%%%%%%%%%%%%%%%%%%%%%%%%%%%%%%%%%%%%%%%%%%%%%%%%%

\begin{lemma} \label{lm:algebra-sum-beta-alphai}
For any $t\geq 1$ and $0<\beta\leq 1/2$, we have
\begin{equation*}
    \sum_{i=1}^{t}(1-\beta)^{t-i}\alpha_i \leq 32 + 16\ln\frac{1}{\beta}.
\end{equation*}
\end{lemma}

\begin{proof}[Proof of \Cref{lm:algebra-sum-beta-alphai}]
We split the summation as the following:
\begin{equation*}
    \begin{aligned}
        \sum_{i=1}^{t}(1-\beta)^{t-i}\alpha_i
        &= \beta\sum_{i=1}^{t}\frac{(1-\beta)^{t-i}}{1-(1-\beta)^{i}} 
        = \beta(1-\beta)^{t}\sum_{i=1}^{t}\frac{(1-\beta)^{-i}}{1-(1-\beta)^{i}} \\
        &= \beta(1-\beta)^{t}\left(\sum_{1\leq i< 1/\beta}\frac{(1-\beta)^{-i}}{1-(1-\beta)^{i}} + \sum_{1/\beta\leq i\leq t}\frac{(1-\beta)^{-i}}{1-(1-\beta)^{i}}\right).
    \end{aligned}
\end{equation*}
Note that when $i < 1/\beta$, we have 
\begin{equation*}
    \begin{aligned}
        (1-\beta)^i\leq 1-\frac{1}{2}\beta i
        \quad\Longrightarrow\quad
        1-(1-\beta)^i \geq \frac{1}{2}\beta i
        \quad\Longrightarrow\quad
        \frac{1}{1-(1-\beta)^i} \leq \frac{2}{\beta i},
    \end{aligned}
\end{equation*}
and by \Cref{lm:algebra-log} and $\beta\leq 1/2$ we know that
\begin{equation*}
    \begin{aligned}
        (1-\beta)^{-i} = \exp(-i\ln(1-\beta)) \leq \exp\left(\frac{i\beta}{1-\beta}\right) \leq \exp\left(\frac{1}{1-\beta}\right) \leq e^2.
    \end{aligned}
\end{equation*}
Then for the first part of the summation we have
\begin{equation} \label{eq:sum-case-1}
    \begin{aligned}
        \sum_{1\leq i< 1/\beta}\frac{(1-\beta)^{-i}}{1-(1-\beta)^{i}}
        &\leq \frac{2e^2}{\beta}\sum_{1\leq i< 1/\beta}\frac{1}{i} 
        \leq \frac{2e^2}{\beta}\left(1+\ln\frac{1}{\beta}\right). 
    \end{aligned}
\end{equation}
Also note that when $i\geq 1/\beta$, we have
\begin{equation*}
    \begin{aligned}
        (1-\beta)^i\leq \frac{1}{e}
        \quad\Longrightarrow\quad
        1-(1-\beta)^i \geq 1-\frac{1}{e}
        \quad\Longrightarrow\quad
        \frac{1}{1-(1-\beta)^i} \leq \frac{e}{e-1}.
    \end{aligned}
\end{equation*}
Then for the second part of the summation we have
\begin{equation} \label{eq:sum-case-2}
    \begin{aligned}
        \sum_{1/\beta\leq i\leq t}\frac{(1-\beta)^{-i}}{1-(1-\beta)^{i}}
        &\leq \frac{e}{e-1}\sum_{1/\beta\leq i\leq t}(1-\beta)^{-i} 
        \leq \frac{e}{e-1}\sum_{1\leq i\leq t}(1-\beta)^{-i} 
        \leq \frac{e(1-\beta)^{-t}}{(e-1)\beta}.
    \end{aligned}
\end{equation}
Combining \eqref{eq:sum-case-1} and \eqref{eq:sum-case-2} we obtain that
\begin{equation*}
    \begin{aligned}
        \sum_{i=1}^{t}(1-\beta)^{t-i}\alpha_i
        &\leq \beta(1-\beta)^{t}\left(\sum_{1\leq i< 1/\beta}\frac{(1-\beta)^{-i}}{1-(1-\beta)^{i}} + \sum_{1/\beta\leq i\leq t}\frac{(1-\beta)^{-i}}{1-(1-\beta)^{i}}\right) \\
        &\leq \beta(1-\beta)^{t}\left(\frac{2e^2}{\beta}\left(1+\ln\frac{1}{\beta}\right) + \frac{e(1-\beta)^{-t}}{(e-1)\beta}\right) \\
        &= 2e^2(1-\beta)^{t}\left(1+\ln\frac{1}{\beta}\right) + \frac{e}{e-1} \\
        &\leq 2e^2\left(1+\ln\frac{1}{\beta}\right) + \frac{e}{e-1} \\
        &\leq 32 + 16\ln\frac{1}{\beta}.
    \end{aligned}
\end{equation*}
\end{proof}

%%%%%%%%%%%%%%%%%%%%%%%%%%%%%%%%%%%%%%%%%%%%%%%%%%%%%%%%%%%%%%%%%%%%%%%%%%%%%%%%%%%%%%%%%%%%%%%
Finally, we provide a useful lemma regarding the time-dependent re-scaled momentum parameters in \eqref{eq:adam-equiv} and \cref{alg:bi-adam-equiv} for upper-level analysis.

\begin{lemma}[{\citep[Lemma C.3]{li2023convergence}}] \label{lm:algebra-sum-alphasqt}
Let $\alpha_t=\frac{\beta}{1-(1-\beta)^t}$, then for all $T\geq 2$, we have 
\begin{equation*}
    \sum_{t=2}^{T}\alpha_t^2 \leq 3(1+\beta^2T).
\end{equation*}
\end{lemma}

%%%%%%%%%%%%%%%%%%%%%%%%%%%%%%%%%%%%%%%%%%%%%%%%%%%%%%%%%%%%%%%%%%%%%%%%%%%%%%%%%%%%%%%%%%%%%%%

\subsection{Probabilistic Lemmas}
\label{sec:prob-lemma}

In this section, we provide a well-known probabilistic lemma without proof.

\begin{lemma}[Optional Stopping Theorem] \label{lm:optional-stopping-thm}
Let $\{Z_t\}_{t\geq1}$ be a martingale with respect to a filtration $\{\gF_t\}_{t\geq0}$. Let $\tau$ be a bounded stopping time with respect to the same filtration. Then we have $\E[Z_{\tau}]=\E[Z_0]$.
\end{lemma}

%%%%%%%%%%%%%%%%%%%%%%%%%%%%%%%%%%%%%%%%%%%%%%%%%%%%%%%%%%%%%%%%%%%%%%%%%%%%%%%%%%%%%%%%%%%%%%%

\subsection{Auxiliary Lemmas for Bilevel Optimization}
\label{sec:auxiliary-bilevel}

In this section, we provide several useful lemmas for bilevel optimization under the unbounded smoothness setting, including the properties of the objective function $\Phi$ (\cref{sec:property-obj}), the Neumann series approximation error (\cref{sec:neumann-approx}), and the hypergradient estimation error (\cref{sec:hyper-estimation-error}).

%%%%%%%%%%%%%%%%%%%%%%%%%%%%%%%%%%%%%%%%%%%%%%%%%%%%%%%%%%%%%%%%%%%%%%%%%%%%%%%%%%%%%%%%%%%%%%%

\subsubsection{Properties of the Objective Function}
\label{sec:property-obj}

% \cref{lm:lip-y,lm:Phi-relax-smooth,lm:descent-inequality} characterize the properties of the objective function $\Phi$. In particular, under \cref{ass:bilevel-assumption}, $\Phi$ is $(L_0,L_1)$-smooth.

\begin{lemma}[{\citep[Lemma 8]{hao2024bilevel}}] \label{lm:lip-y}
Under \Cref{ass:bilevel-assumption}, we have
\begin{enumerate}[(I)]
    \item $y^*(x)$ is $(l_{g,1}/\mu)$-Lipschitz continuous.
    \item $\|\gdx f(x,y^*(x))\| \leq \|\gdphi(x)\| + l_{g,1}l_{f,0}/\mu$.
\end{enumerate}
\end{lemma}

%%%%%%%%%%%%%%%%%%%%%%%%%%%%%%%%%%%%%%%%%%%%%%%%%%%%%%%%%%%%%%%%%%%%%%%%%%%%%%%%%%%%%%%%%%%%%%%

\begin{lemma}[$(L_0,L_1)$-smoothness~{\citep[Lemma 9]{hao2024bilevel}}] \label{lm:Phi-relax-smooth}
Under \Cref{ass:bilevel-assumption}, for any $x, x'\in\R^{d_x}$ we have
\begin{align}
    \|\gdphi(x)-\gdphi(x')\| &\leq (L_0 + L_1\|\gdphi(x')\|)\|x-x'\| \notag \\
    &\text{if} \quad
    \|x-x'\| \leq r \coloneqq \frac{1}{\sqrt{(1+l_{g,1}^2/\mu^2)(L_{x,1}^2+L_{y,1}^2)}}, \label{eq:r-def}
\end{align}
where the $(L_0,L_1)$-smoothness constants $L_0$ and $L_1$ are defined as
\begin{equation} \label{eq:L0-L1}
    \begin{aligned}
        L_0 &= \sqrt{1+\frac{l_{g,1}^2}{\mu^2}}\left(L_{x,0} + L_{x,1}\frac{l_{g,1}l_{f,0}}{\mu} + \frac{l_{g,1}}{\mu}(L_{y,0}+L_{y,1}l_{f,0}) + l_{f,0}\frac{l_{g,1}l_{g,2}+\mu l_{g,2}}{\mu^2}\right), \\
        L_1 &= \sqrt{1+\frac{l_{g,1}^2}{\mu^2}}L_{x,1}.
    \end{aligned}
\end{equation}
\end{lemma}

%%%%%%%%%%%%%%%%%%%%%%%%%%%%%%%%%%%%%%%%%%%%%%%%%%%%%%%%%%%%%%%%%%%%%%%%%%%%%%%%%%%%%%%%%%%%%%%

\begin{lemma}[Descent Inequality~{\citep[Lemma 10]{hao2024bilevel}}] \label{lm:descent-inequality}
Under \Cref{ass:bilevel-assumption}, for any $x, x'\in\R^{d_x}$ we have
\begin{equation*}
    \begin{aligned}
        \Phi(x) \leq \Phi(x') &+ \langle \gdphi(x'),x-x' \rangle + \frac{L_0+L_1\|\gdphi(x')\|}{2}\|x-x'\|^2 \\
        &\text{if} \quad
        \|x-x'\| \leq r = \frac{1}{\sqrt{(1+l_{g,1}^2/\mu^2)(L_{x,1}^2+L_{y,1}^2)}}.
    \end{aligned}
\end{equation*}
\end{lemma}

%%%%%%%%%%%%%%%%%%%%%%%%%%%%%%%%%%%%%%%%%%%%%%%%%%%%%%%%%%%%%%%%%%%%%%%%%%%%%%%%%%%%%%%%%%%%%%%

\subsubsection{Neumann Series Approximation}
\label{sec:neumann-approx}

Throughout the paper, for given $(x,y)\in\R^{d_x}\times\R^{d_y}$, we estimate the hypergradient $\gdphi(x)$ using Neumann series approach and the following formulation:
% \begin{small}
\begin{equation*}
    \begin{aligned}
        \hatphi(x,y;\Bar{\xi}) = \gdx F(x,y;\xi) - \gdxy G(x,y;\zeta^{(0)})\left[\frac{1}{l_{g,1}}\sum_{q=0}^{Q-1}\prod_{j=1}^{q}\left(I - \frac{\gdyy G(x,y;\zeta^{(q,j)})}{l_{g,1}}\right)\right]\gdy F(x,y;\xi),
    \end{aligned}
\end{equation*}
% \end{small}%
where the randomness $\Bar{\xi}$ is defined as 
\begin{equation*}
    \Bar{\xi}\coloneqq \{\xi, \zeta^{(0)}, \Bar{\zeta}^{(0)}, \dots, \Bar{\zeta}^{(Q-1)}\}, 
    \quad\quad\text{with}\quad
    \Bar{\zeta}^{(q)}\coloneqq \{\zeta^{(q,1)}, \dots, \zeta^{(q,q)}\}.
\end{equation*}
For simplicity, denote $P$ as the Neumann series approximation matrix for the Hessian inverse, then $P$ and $\E_{\Bar{\xi}}[P]$ can be written as:
\begin{equation} \label{eq:P-def}
    \begin{aligned}
        P = \frac{1}{l_{g,1}}\sum_{q=0}^{Q-1}\prod_{j=1}^{q}\left(I - \frac{\gdyy G(x,y;\zeta^{(q,j)})}{l_{g,1}}\right)
        \quad\text{and}\quad
        \E_{\Bar{\xi}}[P] = \frac{1}{l_{g,1}}\sum_{q=0}^{Q-1}\left(I - \frac{\gdyy G(x,y)}{l_{g,1}}\right)^q.
    \end{aligned}
\end{equation}
Hence the simplified version of the hypergradient estimator and its expectation are
\begin{equation} \label{eq:hyper-estimator-def}
    \begin{aligned}
        &\hatphi(x,y;\Bar{\xi}) = \gdx F(x,y;\xi) - \gdxy G(x,y;\zeta^{(0)})P\gdy F(x,y;\xi), \\
        &\E_{\Bar{\xi}}[\hatphi(x,y;\Bar{\xi})] = \gdx f(x,y) - \gdxy g(x,y)\E_{\Bar{\xi}}[P]\gdy f(x,y).
    \end{aligned}
\end{equation}
Also, we define $\gdf(x,y)$ as 
\begin{equation*}
    \begin{aligned}
        \gdf(x,y) = \gdx f(x,y) - \gdxy g(x,y)[\gdyy g(x,y)]^{-1}\gdy f(x,y),
    \end{aligned}
\end{equation*}
which is useful for the following analysis.
% Similarly, for \cref{alg:bi-adam} we denote $P_t$ as the Neumann series approximation for Hessian inverse, and $P_t$ and $\E_t[P_t]$ can be written as
% \begin{equation*}
%     \begin{aligned}
%         P_t = \frac{1}{l_{g,1}}\sum_{q=0}^{Q-1}\prod_{j=1}^{q}\left(I - \frac{\gdyy G(x_t,y_t;\zeta_t^{(q,j)})}{l_{g,1}}\right)
%         \quad\text{and}\quad
%         \E_t[P_t] = \frac{1}{l_{g,1}}\sum_{q=0}^{Q-1}\left(I - \frac{\gdyy G(x_t,y_t)}{l_{g,1}}\right)^q.
%     \end{aligned}
% \end{equation*}

The following lemma bounds the norm of the Neumann series approximation matrix $P$ and characterizes the approximation error for the Hessian inverse in expectation.

%%%%%%%%%%%%%%%%%%%%%%%%%%%%%%%%%%%%%%%%%%%%%%%%%%%%%%%%%%%%%%%%%%%%%%%%%%%%%%%%%%%%%%%%%%%%%%%

\begin{lemma} \label{lm:neumann-series-bound}
Under \Cref{ass:bilevel-assumption,ass:noise,ass:bilevel-additional}, we have 
\begin{equation*}
    \begin{aligned}
        \|\E_{\Bar{\xi}}[P]\| \leq \|P\| \leq \frac{1}{\mu}
        \quad\quad\text{and}\quad\quad
        \|\E_{\Bar{\xi}}[P] - [\gdyy g(x,y)]^{-1}\| \leq \frac{1}{\mu}\left(1-\frac{\mu}{l_{g,1}}\right)^Q.
    \end{aligned}
\end{equation*}
\end{lemma}

\begin{proof}[Proof of \Cref{lm:neumann-series-bound}]
We follow the similar proof as in \citep[Lemma 3.2]{ghadimi2018approximation}. By \Cref{ass:bilevel-additional} and definition of $P$ in \eqref{eq:P-def}, for any $Q\geq 1$ we have
\begin{equation*}
    \begin{aligned}
        \|\E_{\Bar{\xi}}[P]\| \leq \|P\| = \left\|\frac{1}{l_{g,1}}\sum_{q=0}^{Q-1}\prod_{j=1}^{q}\left(I - \frac{\gdyy G(x,y;\zeta^{(q,j)})}{l_{g,1}}\right)\right\|
        \leq \frac{1}{l_{g,1}}\sum_{q=0}^{Q-1}\left(1-\frac{\mu}{l_{g,1}}\right)^q
        \leq \frac{1}{\mu}.
    \end{aligned}
\end{equation*}
As for the second result, we have
\begin{equation*}
    \begin{aligned}
        \|\E_{\Bar{\xi}}[P] - [\gdyy g(x,y)]^{-1}\| 
        &\leq \frac{1}{l_{g,1}}\left\|\sum_{q=Q}^{\infty}\left(I - \frac{\gdyy G(x,y)}{l_{g,1}}\right)^q\right\| \\
        &\leq \frac{1}{l_{g,1}}\sum_{q=Q}^{\infty}\left\|\left(I - \frac{\gdyy G(x,y)}{l_{g,1}}\right)\right\|^q
        \leq \frac{1}{\mu}\left(1-\frac{\mu}{l_{g,1}}\right)^Q.
    \end{aligned}
\end{equation*}
\end{proof}

%%%%%%%%%%%%%%%%%%%%%%%%%%%%%%%%%%%%%%%%%%%%%%%%%%%%%%%%%%%%%%%%%%%%%%%%%%%%%%%%%%%%%%%%%%%%%%%

\subsubsection{Hypergradient Estimation Error}
\label{sec:hyper-estimation-error}

\begin{lemma} \label{lm:hyper-noise}
Under \Cref{ass:bilevel-assumption,ass:noise,ass:bilevel-additional}, if $\|y-y^*(x)\|\leq r$, we have 
\begin{equation*}
    \begin{aligned}
        \|\hatphi&(x,y;\Bar{\xi}) - \E_{\Bar{\xi}}[\hatphi(x,y;\Bar{\xi})]\| \\
        &\leq \frac{\mu+3l_{g,1}+\sigma_{g,2}}{\mu}\sigma_f + \frac{2l_{g,1}+\sigma_{g,2}}{\mu}l_{f,0} + \frac{2l_{g,1}+\sigma_{g,2}}{\mu}(L_{y,0}+L_{y,1}l_{f,0})\|y-y^*(x)\|.
    \end{aligned}
\end{equation*}
\end{lemma}

\begin{proof}[Proof of \Cref{lm:hyper-noise}]
We will use a short hand $y^*=y^*(x)$. By triangle inequality, we have
\begin{equation*}
    \begin{aligned}
        &\|\hatphi(x,y;\Bar{\xi}) - \E_{\Bar{\xi}}[\hatphi(x,y;\Bar{\xi})]\| \\
        &= \|(\gdx F(x,y;\xi) - \gdxy G(x,y;\zeta^{(0)})P\gdy F(x,y;\xi)) - (\gdx f(x,y) - \gdxy g(x,y)\E_{\Bar{\xi}}[P]\gdy f(x,y))\| \\
        &\leq \underbrace{\|\gdx F(x,y;\xi) - \gdx f(x,y)\|}_{(A_1)} + \underbrace{\|(\gdxy G(x,y;\zeta^{(0)})-\gdxy g(x,y))P\gdy F(x,y;\xi)\|}_{(A_2)} \\
        &\quad+ \underbrace{\|\gdxy g(x,y)(P-\E_{\Bar{\xi}}[P])\gdy F(x,y;\xi)\|}_{(A_3)} + \underbrace{\|\gdxy g(x,y)\E_{\Bar{\xi}}[P](\gdy F(x,y;\xi)-\gdy f(x,y))\|}_{(A_4)}
    \end{aligned}
\end{equation*}
\paragraph{Bounding $(A_1)$.} 
By \Cref{ass:noise}, we have
\begin{equation*}
    \begin{aligned}
        (A_1)
        = \|\gdx F(x,y;\xi) - \gdx f(x,y)\| \leq \sigma_f.
    \end{aligned}
\end{equation*}
\paragraph{Bounding $(A_2)$.}
By \Cref{ass:bilevel-assumption,ass:noise,lm:neumann-series-bound}, we have
\begin{equation*}
    \begin{aligned}
        (A_2)
        &= \|(\gdxy G(x,y;\zeta^{(0)})-\gdxy g(x,y))P\gdy F(x,y;\xi)\| \\
        &\leq \|\gdxy G(x,y;\zeta^{(0)})-\gdxy g(x,y)\|\|P\|\|\gdy F(x,y;\xi)\| \\
        &\leq \frac{\sigma_{g,2}}{\mu}(\|\gdy F(x,y;\xi) - \gdy f(x,y)\| + \|\gdy f(x,y) - \gdy f(x,y^*)\| + \|\gdy f(x,y^*)\|) \\
        &\leq \frac{\sigma_{g,2}}{\mu}(\sigma_f + (L_{y,0}+L_{y,1}l_{f,0})\|y-y^*\| + l_{f,0}) \\
        &= \frac{\sigma_{g,2}}{\mu}(\sigma_f+l_{f,0}) + \frac{\sigma_{g,2}}{\mu}(L_{y,0}+L_{y,1}l_{f,0})\|y-y^*\|.
    \end{aligned}
\end{equation*}
\paragraph{Bounding $(A_3)$.}
By \Cref{ass:bilevel-assumption,ass:noise,lm:neumann-series-bound}, we have
\begin{equation*}
    \begin{aligned}
        (A_3)
        &= \|\gdxy g(x,y)(P-\E_{\Bar{\xi}}[P])\gdy F(x,y;\xi)\| \\
        &\leq \|\gdxy g(x,y)\|\|(P-\E_{\Bar{\xi}}[P])\|\|\gdy F(x,y;\xi)\| \\
        &\leq \frac{2l_{g,1}}{\mu}(\sigma_f + (L_{y,0}+L_{y,1}l_{f,0})\|y-y^*\| + l_{f,0}) \\
        &= \frac{2l_{g,1}}{\mu}(\sigma_f+l_{f,0}) + \frac{2l_{g,1}}{\mu}(L_{y,0}+L_{y,1}l_{f,0})\|y-y^*\|,
    \end{aligned}
\end{equation*}
where the second inequality uses the same step (the third inequality above) as in bounding $(A_2)$.
\paragraph{Bounding $(A_4)$.}
By \Cref{ass:bilevel-assumption,ass:noise,lm:neumann-series-bound}, we have
\begin{equation*}
    \begin{aligned}
        (A_4)
        =\|\gdxy g(x,y)\E_{\Bar{\xi}}[P](\gdy F(x,y;\xi)-\gdy f(x,y))\|
        \leq \frac{l_{g,1}}{\mu}\sigma_f.
    \end{aligned}
\end{equation*}
Then we obtain the final bound
\begin{equation*}
    \begin{aligned}
        \|\hatphi&(x,y;\Bar{\xi}) - \E_{\Bar{\xi}}[\hatphi(x,y;\Bar{\xi})]\| 
        \leq (A_1) + (A_2) + (A_3) + (A_4) \\
        &\leq \frac{\mu+3l_{g,1}+\sigma_{g,2}}{\mu}\sigma_f + \frac{2l_{g,1}+\sigma_{g,2}}{\mu}l_{f,0} + \frac{2l_{g,1}+\sigma_{g,2}}{\mu}(L_{y,0}+L_{y,1}l_{f,0})\|y-y^*\|.
    \end{aligned}
\end{equation*}
\end{proof}

%%%%%%%%%%%%%%%%%%%%%%%%%%%%%%%%%%%%%%%%%%%%%%%%%%%%%%%%%%%%%%%%%%%%%%%%%%%%%%%%%%%%%%%%%%%%%%%

\begin{lemma} \label{lm:stoc-hyper-error}
Under \Cref{ass:bilevel-assumption,ass:noise,ass:bilevel-additional}, if $\|y-y^*(x)\|\leq r$, we have 
\begin{equation*}
    \begin{aligned}
        \|\hatphi(x,y;\Bar{\xi}) - \gdphi(x)\| \leq C_{\phi,0} + (C_{\phi,1} + L_1\|\gdphi(x)\|)\|y-y^*(x)\|,
    \end{aligned}
\end{equation*}
where $L_1$ is defined in \eqref{eq:L0-L1} and constants $C_{\phi,0}$ and $C_{\phi,1}$ are defined as
\begin{equation} \label{eq:C-phi-01}
    \begin{aligned}
        C_{\phi,0} &= \frac{\mu+3l_{g,1}+\sigma_{g,2}}{\mu}\sigma_f + \frac{2l_{g,1}+\sigma_{g,2}}{\mu}l_{f,0} + \frac{l_{g,1}l_{f,0}}{\mu}, \\
        C_{\phi,1} &= \frac{2l_{g,1}+\sigma_{g,2}}{\mu}(L_{y,0}+L_{y,1}l_{f,0}) + \frac{l_{g,1}}{\mu}(L_{y,0}+L_{y,1}l_{f,0}) + L_0.
    \end{aligned}
\end{equation}
\end{lemma}

\begin{proof}[Proof of \Cref{lm:stoc-hyper-error}]
We have the following decomposition:
\begin{equation*}
    \begin{aligned}
        \|\hatphi(x,y;\Bar{\xi}) - \gdphi(x)\| 
        &\leq \|\hatphi(x,y;\Bar{\xi}) - \E_{\Bar{\xi}}[\hatphi(x,y;\Bar{\xi})]\| \\
        &\quad+ \|\E_{\Bar{\xi}}[\hatphi(x,y;\Bar{\xi})] - \gdf(x,y)\| + \|\gdf(x,y) - \gdphi(x)\|,
    \end{aligned}
\end{equation*}
For the first term, by \Cref{lm:hyper-noise} we have
\begin{equation} \label{eq:first-bound}
    \begin{aligned}
        \|\hatphi&(x,y;\Bar{\xi}) - \E_{\Bar{\xi}}[\hatphi(x,y;\Bar{\xi})]\| \\
        &\leq \frac{\mu+3l_{g,1}+\sigma_{g,2}}{\mu}\sigma_f + \frac{2l_{g,1}+\sigma_{g,2}}{\mu}l_{f,0} + \frac{2l_{g,1}+\sigma_{g,2}}{\mu}(L_{y,0}+L_{y,1}l_{f,0})\|y-y^*\|.
    \end{aligned}
\end{equation}
For the second term, by \Cref{ass:bilevel-assumption,lm:neumann-series-bound} we have
\begin{equation} \label{eq:second-bound}
    \begin{aligned}
        \|\E_{\Bar{\xi}}[\hatphi&(x,y;\Bar{\xi})] - \gdf(x,y)\| \\
        &= \|(\gdx f(x,y) - \gdxy g(x,y)\E_{\Bar{\xi}}[P]\gdy f(x,y)) \\
        &\quad\quad\quad\quad\quad\quad\quad- (\gdx f(x,y) - \gdxy g(x,y)[\gdyy g(x,y)]^{-1}\gdy f(x,y))\| \\
        &= \|\gdxy g(x,y)(\E_{\Bar{\xi}}[P]-[\gdyy g(x,y)]^{-1}])\gdy f(x,y)\| \\
        &\leq \frac{l_{g,1}}{\mu}\left(1-\frac{\mu}{l_{g,1}}\right)^Q(\|\gdy f(x,y) - \gdy f(x,y^*)\| + \|\gdy f(x,y)\|) \\
        &\leq \frac{l_{g,1}}{\mu}\left(1-\frac{\mu}{l_{g,1}}\right)^Q((L_{y,0}+L_{y,1}l_{f,0})\|y-y^*\| + l_{f,0}) \\
        &= \frac{l_{g,1}l_{f,0}}{\mu}\left(1-\frac{\mu}{l_{g,1}}\right)^Q + \frac{l_{g,1}}{\mu}\left(1-\frac{\mu}{l_{g,1}}\right)^Q(L_{y,0}+L_{y,1}l_{f,0})\|y-y^*\|.
    \end{aligned}
\end{equation}
For the third term, by \Cref{ass:bilevel-assumption,lm:lip-y} we have
\begin{equation} \label{eq:third-bound}
    \begin{aligned}
        &\|\gdf(x,y) - \gdphi(x)\| \\
        &\leq \|\gdx f(x,y) - \gdx f(x,y^*)\| \\
        &\quad+ \|\gdxy g(x,y)[\gdyy g(x,y)]^{-1}\gdy f(x,y) - \gdxy g(x,y^*)[\gdyy g(x,y^*)]^{-1}\gdy f(x,y^*)\| \\
        &\leq (L_{x,0}+L_{x,1}\|\gdx f(x,y^*)\|)\|y-y^*\| \\
        &\quad+ \|\gdxy g(x,y)[\gdyy g(x,y)]^{-1}\gdy f(x,y) - \gdxy g(x,y^*)[\gdyy g(x,y)]^{-1}\gdy f(x,y)\| \\
        &\quad+ \|\gdxy g(x,y^*)[\gdyy g(x,y)]^{-1}\gdy f(x,y) - \gdxy g(x,y^*)[\gdyy g(x,y^*)]^{-1}\gdy f(x,y)\| \\
        &\quad+ \|\gdxy g(x,y^*)[\gdyy g(x,y^*)]^{-1}\gdy f(x,y) - \gdxy g(x,y^*)[\gdyy g(x,y^*)]^{-1}\gdy f(x,y^*)\| \\
        &\leq \left(L_{x,0}+L_{x,1}\left(\frac{l_{g,1}l_{f,0}}{\mu} + \|\gdphi(x)\|\right)\right)\|y-y^*\| \\
        &\quad+ \frac{l_{f,0}}{\mu}l_{g,2}\|y-y^*\| + \frac{l_{f,0}l_{g,1}}{\mu^2}l_{g,2}\|y-y^*\| + \frac{l_{g,1}}{\mu}(L_{y,0}+L_{y,1}\|\gdy f(x,y^*)\|)\|y-y^*\| \\
        &= \left(L_{x,0} + L_{x,1}\frac{l_{g,1}l_{f,0}}{\mu} + \frac{l_{g,1}}{\mu}(L_{y,0}+L_{y,1}l_{f,0}) + l_{f,0}\frac{\mu l_{g,2}+l_{g,1}l_{g,2}}{\mu^2} + L_{x,1}\|\gdphi(x)\|\right)\|y-y^*\| \\
        &\leq (L_0 + L_1\|\gdphi(x)\|)\|y-y^*\|,
    \end{aligned}
\end{equation}
where the last inequality uses the definition of $L_0$ and $L_1$ as in \eqref{eq:L0-L1}. Summing up $\eqref{eq:first-bound} + \eqref{eq:second-bound} + \eqref{eq:third-bound}$ gives the final bound 
\begin{equation*}
    \begin{aligned}
        &\|\hatphi(x,y;\Bar{\xi}) - \gdphi(x)\| 
        \leq \|\hatphi(x,y;\Bar{\xi}) - \E_{\Bar{\xi}}[\hatphi(x,y;\Bar{\xi})]\| \\
        &\quad\quad\quad\quad\quad\quad\quad\quad\quad\quad\quad+ \|\E_{\Bar{\xi}}[\hatphi(x,y;\Bar{\xi})] - \gdf(x,y)\| + \|\gdf(x,y) - \gdphi(x)\| \\
        &\leq \frac{\mu+3l_{g,1}+\sigma_{g,2}}{\mu}\sigma_f + \frac{2l_{g,1}+\sigma_{g,2}}{\mu}l_{f,0} + \frac{l_{g,1}l_{f,0}}{\mu}\left(1-\frac{\mu}{l_{g,1}}\right)^Q \\
        &\quad+ \left(\frac{2l_{g,1}+\sigma_{g,2}}{\mu}(L_{y,0}+L_{y,1}l_{f,0}) + \frac{l_{g,1}}{\mu}\left(1-\frac{\mu}{l_{g,1}}\right)^Q(L_{y,0}+L_{y,1}l_{f,0}) + L_0 + L_1\|\gdphi(x)\|\right)\|y-y^*\| \\
        &\leq \frac{\mu+3l_{g,1}+\sigma_{g,2}}{\mu}\sigma_f + \frac{2l_{g,1}+\sigma_{g,2}}{\mu}l_{f,0} + \frac{l_{g,1}l_{f,0}}{\mu} \\
        &\quad+ \left(\frac{2l_{g,1}+\sigma_{g,2}}{\mu}(L_{y,0}+L_{y,1}l_{f,0}) + \frac{l_{g,1}}{\mu}(L_{y,0}+L_{y,1}l_{f,0}) + L_0 + L_1\|\gdphi(x)\|\right)\|y-y^*\| \\
        &= C_{\phi,0} + (C_{\phi,1} + L_1\|\gdphi(x)\|)\|y-y^*\|,
    \end{aligned}
\end{equation*}
where the second and the third inequalities use $Q\geq1$, and the last inequality is due to the definitions of $C_{\phi,0}$ and $C_{\phi,1}$ in \eqref{eq:C-phi-01}. 
\end{proof}

%%%%%%%%%%%%%%%%%%%%%%%%%%%%%%%%%%%%%%%%%%%%%%%%%%%%%%%%%%%%%%%%%%%%%%%%%%%%%%%%%%%%%%%%%%%%%%%

\subsubsection{Other Useful Lemmas}

\begin{lemma} \label{lm:hyper-stoc-bias}
Under \Cref{ass:bilevel-assumption,ass:noise,ass:bilevel-additional}, if $\|y-y^*(x)\|\leq r$, we have~\footnote{Please note that $x_1$ and $x_2$ here are unrelated to \cref{alg:bi-adam} and are deterministic.}
\begin{equation*}
    \begin{aligned}
        \|\hatphi(x,y;\Bar{\xi}) - \hatphi(x,y^*(x);\Bar{\xi})\| \leq (L_0 + L_1\|\gdphi(x)\|)\|y-y^*(x)\|;
        % \|\hatphi(x,y;\Bar{\xi}) - \hatphi(x,y^*(x);\Bar{\xi})\| &\leq (L_0 + L_1\|\gdphi(x)\|)\|y-y^*(x)\|, \\
        % \|\E_{\Bar{\xi}_1}[\hatphi(x_1,y_1^*;\Bar{\xi}_1)] - \E_{\Bar{\xi}_2}[\hatphi(x_2,y_2^*;\Bar{\xi}_2)]\| &\leq (L_0 + L_1\|\gdphi(x_1)\|)\|x_1-x_2\|,
    \end{aligned}
\end{equation*}
if $\|x_1-x_2\|\leq \mu r/(\mu+l_{g,1})$, we have
\begin{equation*}
    \begin{aligned}
        \|\E_{\Bar{\xi}_1}[\hatphi(x_1,y_1^*;\Bar{\xi}_1)] - \E_{\Bar{\xi}_2}[\hatphi(x_2,y_2^*;\Bar{\xi}_2)]\| &\leq (L_0 + L_1\|\gdphi(x_1)\|)\|x_1-x_2\|,
    \end{aligned}
\end{equation*}
where $y_i^*=y^*(x_i)$ for $i=1,2$, and constants $L_0$ and $L_1$ are defined in \eqref{eq:L0-L1}.
\end{lemma}

\begin{proof}[Proof of \Cref{lm:hyper-stoc-bias}]
We will use a short hand $y^*=y^*(x)$. Recall the definition of $\hatphi(x,y;\Bar{\xi})$ and $\hatphi(x,y^*;\Bar{\xi})$ in \eqref{eq:hyper-estimator-def}, we have
\begin{equation*}
    \begin{aligned}
        \hatphi(x,y;\Bar{\xi}) &= \gdx F(x,y;\xi) - \gdxy G(x,y;\zeta^{(0)})P\gdy F(x,y;\xi), \\
        \hatphi(x,y^*;\Bar{\xi}) &= \gdx F(x,y^*;\xi) - \gdxy G(x,y^*;\zeta^{(0)})P^*\gdy F(x,y^*;\xi).
    \end{aligned}
\end{equation*}
where similar to \eqref{eq:P-def}, we define the Neumann series approximation matrix $P^*$ as 
\begin{equation} \label{eq:P-star-def}
    \begin{aligned}
        P^* = \frac{1}{l_{g,1}}\sum_{q=0}^{Q-1}\prod_{j=1}^{q}\left(I - \frac{\gdyy G(x,y^*;\zeta^{(q,j)})}{l_{g,1}}\right).
    \end{aligned}
\end{equation}
% Now we bound the difference between the above two stochastic hypergradient estimators:
Then by triangle inequality we have
\begin{equation*}
    \begin{aligned}
        \|&\hatphi(x,y;\Bar{\xi}) - \hatphi(x,y^*;\Bar{\xi})\| \\
        &\leq \|\gdx F(x,y;\xi) - \gdx F(x,y^*;\xi)\| \\
        &\quad+ \|\gdxy G(x,y;\zeta^{(0)})P\gdy F(x,y;\xi) - \gdxy G(x,y^*;\zeta^{(0)})P^*\gdy F(x,y^*;\xi)\| \\
        &\leq \underbrace{\|\gdx F(x,y;\xi) - \gdx F(x,y^*;\xi)\|}_{(A_1)} + \underbrace{\|\gdxy G(x,y;\zeta^{(0)})P(\gdy F(x,y;\xi)-\gdy F(x,y^*;\xi))\|}_{(A_2)} \\
        &\quad+ \underbrace{\|\gdxy G(x,y;\zeta^{(0)})(P-P^*)\gdy F(x,y^*;\xi)\|}_{(A_3)} \\
        &\quad+ \underbrace{\|(\gdxy G(x,y;\zeta^{(0)}) - \gdxy G(x,y^*;\zeta^{(0)}))P^*\gdy F(x,y^*;\xi)\|}_{(A_4)}.
    \end{aligned}
\end{equation*}

\paragraph{Bounding $(A_1)$.}
By \Cref{ass:bilevel-additional,lm:lip-y}, we have
\begin{equation*}
    \begin{aligned}
        (A_1)
        &= \|\gdx F(x,y;\xi) - \gdx F(x,y^*;\xi)\|
        \leq (L_{x,0}+L_{x,1}\|\gdx f(x,y^*)\|)\|y-y^*\| \\
        &\leq \left(L_{x,0} + L_{x,1}\left(\frac{l_{g,1}l_{f,0}}{\mu}+\|\gdphi(x)\|\right)\right)\|y-y^*\| \\
        &= \left(L_{x,0} + \frac{L_{x,1}l_{g,1}l_{f,0}}{\mu} + L_{x,1}\|\gdphi(x)\|\right)\|y-y^*\|.
    \end{aligned}
\end{equation*}

\paragraph{Bounding $(A_2)$.}
By \Cref{ass:bilevel-additional,lm:neumann-series-bound}, we have 
\begin{equation*}
    \begin{aligned}
        (A_2)
        &= \|\gdxy G(x,y;\zeta^{(0)})P(\gdy F(x,y;\xi)-\gdy F(x,y^*;\xi))\| \\
        &= \|\gdxy G(x,y;\zeta^{(0)})\|\|P\|\|\gdy F(x,y;\xi)-\gdy F(x,y^*;\xi)\| \\
        &\leq \frac{l_{g,1}}{\mu}(L_{y,0}+L_{y,1}\|\gdy f(x,y^*)\|)\|y-y^*\| 
        \leq \frac{l_{g,1}}{\mu}(L_{y,0}+L_{y,1}l_{f,0})\|y-y^*\|.
    \end{aligned}
\end{equation*}

\paragraph{Bounding $(A_3)$.}
We first apply \Cref{lm:prod-diff} to obtain 
\begin{equation*}
    \begin{aligned}
        &\left\|\prod_{j=1}^{q}\left(I - \frac{\gdyy G(x,y;\zeta^{(q,j)})}{l_{g,1}}\right) - \prod_{j=1}^{q}\left(I - \frac{\gdyy G(x,y^*;\zeta^{(q,j)})}{l_{g,1}}\right)\right\| \\
        &\quad\quad\quad\quad\quad\leq \sum_{j=1}^{q}\left(1-\frac{\mu}{l_{g,1}}\right)^{q-1}\frac{l_{g,2}}{l_{g,1}}\|y-y^*\| 
        = q\left(1-\frac{\mu}{l_{g,1}}\right)^{q-1}\frac{l_{g,2}}{l_{g,1}}\|y-y^*\|.
    \end{aligned}
\end{equation*}
Hence we can write 
\begin{equation*}
    \begin{aligned}
        \|P-P^*\|
        \leq \frac{1}{l_{g,1}}\sum_{q=0}^{Q-1}q\left(1-\frac{\mu}{l_{g,1}}\right)^{q-1}\frac{l_{g,2}}{l_{g,1}}\|y-y^*\|
        \leq \frac{\mu^2l_{g,2}}{l_{g,1}^4}\|y-y^*\|
        \leq \frac{l_{g,2}}{\mu^2}\|y-y^*\|,
    \end{aligned}
\end{equation*}
where the second inequality uses \Cref{lm:sum-geometric-sequence} with $a=\mu/l_{g,1}$, and the last inequality is due to $\mu<l_{g,1}$. Then by \Cref{ass:bilevel-additional} we have
\begin{equation*}
    \begin{aligned}
        (A_3)
        &= \|\gdxy G(x,y;\zeta^{(0)})(P-P^*)\gdy F(x,y^*;\xi)\| \\
        &\leq \|\gdxy G(x,y;\zeta^{(0)})\|\|(P-P^*)\|\|\gdy F(x,y^*;\xi)\| 
        \leq \frac{l_{g,1}l_{g,2}l_{f,0}}{\mu^2}\|y-y^*\|.
    \end{aligned}
\end{equation*}
\paragraph{Bounding $(A_4)$.}
By \Cref{ass:bilevel-additional,lm:neumann-series-bound}, we have
\begin{equation*}
    \begin{aligned}
        (A_4)
        &= \|(\gdxy G(x,y;\zeta^{(0)}) - \gdxy G(x,y^*;\zeta^{(0)}))P^*\gdy F(x,y^*;\xi)\| \\
        &\leq \|\gdxy G(x,y;\zeta^{(0)}) - \gdxy G(x,y^*;\zeta^{(0)})\|\|P^*\|\|\gdy F(x,y^*;\xi)\| 
        \leq \frac{l_{g,2}l_{f,0}}{\mu}\|y-y^*\|.
    \end{aligned}
\end{equation*}
\paragraph{Final Bound.}
Summing up $(A_1)+(A_2)+(A_3)+(A_4)$ yields the final bound
\begin{equation*}
    \begin{aligned}
        \|&\hatphi(x,y;\Bar{\xi}) - \hatphi(x,y^*;\Bar{\xi})\| 
        \leq (A_1)+(A_2)+(A_3)+(A_4) \\
        &\leq \left(L_{x,0} + L_{x,1}\frac{l_{g,1}l_{f,0}}{\mu} + \frac{l_{g,1}}{\mu}(L_{y,0}+L_{y,1}l_{f,0}) + l_{f,0}\frac{l_{g,1}l_{g,2}+\mu l_{g,2}}{\mu^2} + L_{x,1}\|\gdphi(x)\|\right)\|y-y^*\| \\
        &\leq (L_0 + L_1\|\gdphi(x)\|)\|y-y^*\|,
    \end{aligned}
\end{equation*}
where the last inequality uses the definitions of $L_0$ and $L_1$ as in \eqref{eq:L0-L1}.

For the second result, we follow a similar procedure as above and obtain:
\begin{equation*}
    \begin{aligned}
        &\|\E_{\Bar{\xi}_1}[\hatphi(x_1,y_1^*;\Bar{\xi}_1)] - \E_{\Bar{\xi}_2}[\hatphi(x_2,y_2^*;\Bar{\xi}_2)]\|
        \leq (A_1)+(A_2)+(A_3)+(A_4) \\
        &\leq \sqrt{1+\frac{l_{g,1}^2}{\mu^2}}\left(L_{x,0} + L_{x,1}\frac{l_{g,1}l_{f,0}}{\mu} + \frac{l_{g,1}}{\mu}(L_{y,0}+L_{y,1}l_{f,0}) + l_{f,0}\frac{l_{g,1}l_{g,2}+\mu l_{g,2}}{\mu^2} + L_{x,1}\|\gdphi(x_1)\|\right)\|x_1-x_2\| \\
        &= (L_0 + L_1\|\gdphi(x_1)\|)\|x_1-x_2\|,
    \end{aligned}
\end{equation*}
where the last inequality uses the definitions of $L_0$ and $L_1$ as in \eqref{eq:L0-L1}.
\end{proof}

%%%%%%%%%%%%%%%%%%%%%%%%%%%%%%%%%%%%%%%%%%%%%%%%%%%%%%%%%%%%%%%%%%%%%%%%%%%%%%%%%%%%%%%%%%%%%%%

\begin{lemma} \label{lm:neumann-error}
Under \Cref{ass:bilevel-assumption,ass:noise,ass:bilevel-additional}, we have 
\begin{equation*}
    \begin{aligned}
        \|\E_{\Bar{\xi}}[\hatphi(x,y^*(x);\Bar{\xi})] - \gdphi(x)\| \leq \frac{l_{g,1}l_{f,0}}{\mu}\left(1-\frac{\mu}{l_{g,1}}\right)^Q.
    \end{aligned}
\end{equation*}
\end{lemma}

\begin{proof}[Proof of \Cref{lm:neumann-error}]
We will use a short hand $y^*=y^*(x)$. By definition of $\hatphi(x,y;\Bar{\xi})$ in \eqref{eq:hyper-estimator-def} and the hypergradient formulation, we have
\begin{equation*}
    \begin{aligned}
        \E_{\Bar{\xi}}[\hatphi(x,y^*;\Bar{\xi})] &= \gdx f(x,y^*) - \gdxy g(x,y^*)\E_{\Bar{\xi}}[P]\gdy f(x,y^*), \\
        \gdphi(x) &= \gdx f(x,y^*) - \gdxy g(x,y^*)[\gdyy g(x,y^*)]^{-1}\gdy f(x,y^*).
    \end{aligned}
\end{equation*}
Then we obtain the conclusion by applying \Cref{ass:bilevel-assumption,lm:neumann-series-bound}:
\begin{equation*}
    \begin{aligned}
        \|\E_{\Bar{\xi}}&[\hatphi(x,y;\Bar{\xi})] - \gdphi(x)\|
        = \|\gdxy g(x,y^*)(\E_{\Bar{\xi}}[P]-[\gdyy g(x,y^*)]^{-1})\gdy f(x,y^*)\| \\
        &\leq \|\gdxy g(x,y^*)\|\|\E_{\Bar{\xi}}[P]-[\gdyy g(x,y^*)]^{-1}\|\|\gdy f(x,y^*)\| 
        \leq \frac{l_{g,1}l_{f,0}}{\mu}\left(1-\frac{\mu}{l_{g,1}}\right)^Q.
    \end{aligned}
\end{equation*}
\end{proof}

\section{Proof of the Random Decoupling Lemma (\cref{lm:any-sequence-main})}
\label{app:any_sequence}
\subsection{Recursive Control on Moment Generating Function}

The following technical lemma on recursive control is crucial for establishing high probability guarantee for controlling the lower-level estimation error at anytime. We follow a similar argument as in~\citep[Proposition 29]{cutler2023stochastic} with a slight generalization.

\begin{proposition}[Recursive control on MGF] \label{prop:MGF-recursive control}
Consider scalar stochastic processes $(V_t)$, $(D_t)$, $(X_t)$ and $(Y_t)$ on a probability space with filtration $(\gH_t)$, which are linked by the inequality
\begin{equation}
    V_{t+1} \leq \rho_tV_t + D_t\sqrt{V_t} + X_t + Y_t + \kappa_t
\end{equation}
for some deterministic constants $\rho_t\in(-\infty,1]$ and $\kappa_t\in\R$. Suppose the following properties hold.
\begin{itemize}
    \item $V_t$ and $Y_t$ are non-negative and $\gH_t$-measurable.
    \item $D_t$ is mean-zero sub-Gaussian conditioned on $\gH_t$ with deterministic parameter $\sigma_t$:
    \begin{equation*}
        \E[\exp(\theta D_t) \mid \gH_t] \leq \exp(\theta^2\sigma_t^2/2)
        \quad\text{for all}\quad
        \theta \in \R.
    \end{equation*}
    \item $X_t$ is non-negative and sub-exponential conditioned on $\gH_t$ with deterministic parameter $\nu_t$:
    \begin{equation*}
        \E[\exp(\theta X_t) \mid \gH_t] \leq \exp(\theta\nu_t)
        \quad\text{for all}\quad
        0\leq \theta\leq 1/\nu_t.
    \end{equation*}
\end{itemize}
Then the estimate 
\begin{equation*}
    \E[\exp(\theta V_{t+1})] \leq \exp(\theta(\nu_t+\kappa_t))\E[\exp(\theta((1+\rho_t)V_t/2 + Y_t))]
\end{equation*}
holds for any $\theta$ satisfying $0\leq \theta \leq \min\left\{\frac{1-\rho_t}{2\sigma_t^2}, \frac{1}{2\nu_t}\right\}$.
\end{proposition}

\begin{proof}[Proof of \cref{prop:MGF-recursive control}]
For any index $t\geq0$ and any scalar $\theta\geq0$, the law of total expectation implies
\begin{equation*}
    \begin{aligned}
        \E[\exp(\theta V_{t+1})]
        &\leq \E\left[\exp\left(\theta\left(\rho_tV_t + D_t\sqrt{V_t} + X_t + Y_t + \kappa_t\right)\right)\right] \\
        &= \exp(\theta\kappa_t)\E\left[\exp(\theta(\rho_tV_t+Y_t))\E\left[\exp(\theta D_t\sqrt{V_t})\exp(\theta X_t) \mid \gH_t\right]\right].
    \end{aligned}
\end{equation*}
H\"{o}lder's inequality in turn yields
\begin{equation*}
    \begin{aligned}
        \E\left[\exp(\theta D_t\sqrt{V_t})\exp(\theta X_t) \mid \gH_t\right] 
        &\leq \sqrt{\E\left[\exp(2\theta D_t\sqrt{V_t}) \mid \gH_t\right] \cdot \E\left[\exp(2\theta X_t) \mid \gH_t\right]} \\
        &\leq \sqrt{\exp(2\theta^2\sigma_t^2V_t)\exp(2\theta\nu_t)} \\
        &= \exp(\theta^2\sigma_t^2V_t)\exp(\theta\nu_t)
    \end{aligned}
\end{equation*}
provided $0\leq\theta\leq \frac{1}{2\nu_t}$. Therefore, if $\theta$ satisfies
\begin{equation*}
    0\leq \theta\leq \min\left\{\frac{1-\rho_t}{2\sigma_t^2}, \frac{1}{2\nu_t}\right\},
\end{equation*}
then the following estimate holds for all $t\geq0$:
\begin{equation*}
    \begin{aligned}
        \E[\exp(\theta V_{t+1})]
        &\leq \exp(\theta\kappa_t)\E\left[\exp(\theta(\rho_tV_t+Y_t))\exp(\theta^2\sigma_t^2V_t)\exp(\theta\nu_t)\right] \\
        &= \exp(\theta(\nu_t+\kappa_t))\E\left[\exp(\theta((\rho_t+\theta\sigma_t^2)V_t+Y_t))\right] \\
        &\leq \exp(\theta(\nu_t+\kappa_t))\E\left[\exp(\theta((1+\rho_t)V_t/2 + Y_t))\right],
    \end{aligned}
\end{equation*}
where the last inequality uses the given range of $\theta$. Thus the proof is completed. 
\end{proof}

%%%%%%%%%%%%%%%%%%%%%%%%%%%%%%%%%%%%%%%%%%%%%%%%%%%%%%%%%%%%%%%%%%%%%%%%%%%%%%%%%%%%%%%%%%%%%%%

\subsection{Proof of \cref{lm:any-sequence-main}}
\label{sec:proof-randomness-decouple}

In this section, we aim to provide a high-probability guarantee for the approximation error of the lower-level variable, namely $\|y_t-y_t^*\|$. Our main technical contribution is the any-sequence argument, which separates the randomness in the updates of the upper-level variable $x_t$ and the lower-level variable $y_t$. Specifically, for any given sequence $\{\tx_t\}$, we consider the following update rule for $\{\ty_t\}$ (which is the same as line 5 of \Cref{alg:bi-adam}):
\begin{equation} \label{eq:yt-any-update}
    \ty_{t+1} = \ty_t - \gamma\gdy G(\tx_t,\ty_t;\tilde{\zeta}_t).
\end{equation}

Before proceeding, we will first define (or restate) a few key concepts and useful notations. 

\paragraph{Filtration.}
For any $t \geq 2$, define $\tilde{\gF}_t^y$ as the filtration of the randomness used in updating $\tilde{y}_t$ before the $t$-th iteration:
\begin{equation} \label{eq:tilde-gF-def}
    \tilde{\gF}_t^y = \sigma(\tilde{\zeta}_1,\dots,\tilde{\zeta}_{t-1}),
\end{equation}
where $\sigma(\cdot)$ denotes the $\sigma$-algebra generated by the random variables within the argument.

\paragraph{Auxiliary Sequence.}
We also introduce the following auxiliary sequence $\{\tu_t\}$ for our analysis:
\begin{equation} \label{eq:tut-def}
    \begin{aligned}
        \tu_t = (1-\alpha_t)\tu_{t-1} + \alpha_t\hatphi(\tx_t,\ty_t^*;\hat{\xi}_t)
        = \sum_{j=1}^{t}d_{t,j}\hatphi(\tx_t,\ty_t^*;\hat{\xi}_t),
    \end{aligned}
\end{equation}
where the sequence $\{d_{t,j}\}_{j=1}^{t}$ is defined in \eqref{eq:dj-def} of \cref{lm:dj-def}. Similar to \eqref{eq:P-def}, \eqref{eq:hyper-estimator-def} and \eqref{eq:P-star-def} in \cref{app:tech_lemmas}, the hypergradient estimators $\hatphi(\tx_t,\ty_t;\hat{\xi}_t)$ and $\hatphi(\tx_t,\ty_t^*;\hat{\xi}_t)$ can be written as 
\begin{equation*}
    \begin{aligned}
        \hatphi(\tx_t,\ty_t;\hat{\xi}_t) &= \gdx F(\tx_t,\ty_t;\tilde{\xi}_t) - \gdxy G(\tx_t,\ty_t;\tilde{\zeta}_t^{(0)})\tilde{P}_t\gdy F(\tx_t,\ty_t;\tilde{\xi}_t), \\
        \hatphi(\tx_t,\ty_t^*;\hat{\xi}_t) &= \gdx F(\tx_t,\ty_t^*;\tilde{\xi}_t) - \gdxy G(\tx_t,\ty_t^*;\tilde{\zeta}_t^{(0)})\tilde{P}_t^*\gdy F(\tx_t,\ty_t^*;\tilde{\xi}_t),
    \end{aligned}
\end{equation*}
where the randomness $\hat{\xi}_t$ is defined as 
\begin{equation} \label{eq:hat-xi-def}
    \begin{aligned}
        \hat{\xi}_t\coloneqq \{\tilde{\xi}_t, \tilde{\zeta}_t^{(0)}, \tilde{\Bar{\zeta}}^{(0)}, \dots, \tilde{\Bar{\zeta}}^{(Q-1)}\},
        \quad\quad\text{where}\quad
        \tilde{\Bar{\zeta}}^{(q)}\coloneqq \{\tilde{\zeta}^{(q,1)}, \dots, \tilde{\zeta}^{(q,q)}\};
    \end{aligned}
\end{equation}
and the Neumann series approximation matrices $\tilde{P}_t$ and $\tilde{P}_t^*$ are defined as 
\begin{equation*}
    \begin{aligned}
        \tilde{P}_t = \frac{1}{l_{g,1}}\sum_{q=0}^{Q-1}\prod_{j=1}^{q}\left(I - \frac{\gdyy G(\tx_t,\ty_t;\tilde{\zeta}_t^{(q,j)})}{l_{g,1}}\right)
        \quad\text{and}\quad
        \tilde{P}_t^* = \frac{1}{l_{g,1}}\sum_{q=0}^{Q-1}\prod_{j=1}^{q}\left(I - \frac{\gdyy G(\tx_t,\ty_t^*;\tilde{\zeta}_t^{(q,j)})}{l_{g,1}}\right).
    \end{aligned}
\end{equation*}

\paragraph{Constants.}
We define the following constants, which will be useful for analysis. Given any sequence $\{\tx_t\}$, denote $\tg_t$ and $\tl_t$ as
\begin{equation} \label{eq:tilde-Gt-Lt-def}
    \tg_t \coloneqq \max_{1\leq k\leq t}\|\gdphi(\tx_k)\|,
    \quad
    \tl_t \coloneqq L_0 + L_1\tg_t,
\end{equation}
where constants $L_0$ and $L_1$ are defined in \eqref{eq:L0-L1}.

%%%%%%%%%%%%%%%%%%%%%%%%%%%%%%%%%%%%%%%%%%%%%%%%%%%%%%%%%%%%%%%%%%%%%%%%%%%%%%%%%%%%%%%%%%%%%%%

\begin{lemma}[Distance recursion, {\citep[Lemma 25]{cutler2023stochastic}}] \label{lm:distance recursion} 
Suppose that \Cref{ass:bilevel-assumption,ass:noise} hold. For any given sequence $\{\tx_t\}$, let $\{\ty_t\}$ be the iterates generated by the update rule \eqref{eq:yt-any-update} with constant learning rate $\gamma\leq 1/2l_{g,1}$.
Then for any $t\geq1$, we have the following recursion:
\begin{equation}
    \begin{aligned}
        \|\ty_{t+1}-\ty_{t+1}^*\|^2 
        \leq (1-\mu\gamma)\|\ty_t-\ty_t^*\|^2 + 2\gamma\langle \tilde{\varepsilon}_t,\tilde{v}_t \rangle\|\ty_t-\ty_t^*\| + 2\gamma^2\|\tilde{\varepsilon}_t\|^2 + \frac{2}{\mu\gamma}D_t^2,
    \end{aligned}
\end{equation}
where $\tilde{v}_t\coloneqq \frac{\ty_t-\ty_t^*}{\|\ty_t-\ty_t^*\|}$ if $\ty_t$ is distinct from $\ty_t^*$ and zero otherwise, $\tilde{\varepsilon}_t=\gdy g(\tx_t,\ty_t)-\gdy G(\tx_t,\ty_t;\tilde{\zeta}_t)$ denotes the noise, and $D_t \coloneqq \|\ty_t^*-\ty_{t+1}^*\|$ is the minimizer drift at time $t$.
\end{lemma}

% \begin{proof}[Proof of \cref{lm:distance recursion}]
% Please see \citep[Lemma 25]{cutler2023stochastic} for proof details.
% \end{proof}

%%%%%%%%%%%%%%%%%%%%%%%%%%%%%%%%%%%%%%%%%%%%%%%%%%%%%%%%%%%%%%%%%%%%%%%%%%%%%%%%%%%%%%%%%%%%%%%

\begin{lemma}[Restatement of \cref{lm:any-sequence-main}] \label{lm:any-sequence}
Suppose that \Cref{ass:bilevel-assumption,ass:noise} hold. Given any sequence $\{\tx_t\}$ and any randomness $\{\hat{\xi}_t\}$ (see \eqref{eq:hat-xi-def} for definition) such that
\begin{equation} \label{eq:condition-any-sequence}
    \begin{aligned}
        \|\tx_{t+1}-\tx_t\|^2
        \leq \frac{2\eta^2}{\lambda^2}\left(\|\tu_t\|^2 + \tl_t^2\sum_{j=1}^{t}d_{t,j}\|\ty_j-\ty_j^*\|^2\right),
    \end{aligned}
\end{equation}
where $\tu_t$, $\{d_{t,j}\}_{j=1}^{t}$ and $\tl_t$ are defined in \eqref{eq:tut-def}, \eqref{eq:dj-def} and \eqref{eq:tilde-Gt-Lt-def}, respectively. Let $\{\ty_t\}$ be the iterates generated by the update rule \eqref{eq:yt-any-update} with constant learning rate $\gamma\leq 1/2l_{g,1}$, and choose $\gamma=2\beta/\mu$. Then for any given $\delta\in(0,1)$ and all $t\geq1$, the following estimate holds with probability at least $1-\delta$ over the randomness in $\tilde{\gF}_{T+1}^y$:
% \begin{small}
\begin{equation}
    \begin{aligned}
        &\|\ty_t-\ty_t^*\|^2 
        \leq \left(\left(1-\frac{\mu\gamma}{2}\right)^{t-1} + \frac{4\eta^2l_{g,1}^2}{\lambda^2\mu^3\gamma}\sum_{i=1}^{t-1}\left(1-\frac{\mu\gamma}{2}\right)^{t-1-i}\tl_i^2\right) \|\ty_1-\ty_1^*\|^2 \\
        &\quad\quad\quad+ \left(\frac{8\gamma}{\mu}\ln\frac{eT}{\delta} + \frac{16\eta^2l_{g,1}^2}{\lambda^2\mu^4}\sum_{i=1}^{t-1}\left(1-\frac{\mu\gamma}{2}\right)^{t-1-i}\tl_i^2\right)\sigma_{g,1}^2 \\
        &\quad\quad\quad+ \frac{4\eta^2l_{g,1}^2}{\lambda^2\mu^3\gamma}\sum_{i=1}^{t-1}\left(1-\frac{\mu\gamma}{2}\right)^{t-1-i}\|\tu_i\|^2 + \frac{64\eta^4l_{g,1}^4}{\lambda^4\mu^8\gamma^4}\sum_{i=1}^{t-1}\left(1-\frac{\mu\gamma}{2}\right)^{t-1-i}\alpha_i\tl_i^2\|\tu_i\|^2.
    \end{aligned}
\end{equation}
% \end{small}%
\end{lemma}

\begin{proof}[Proof of \Cref{lm:any-sequence}]
By \Cref{lm:distance recursion} and \Cref{lm:lip-y}, we have
% \begin{small}
\begin{equation} \label{eq:in-light}
    \begin{aligned}
        &\|\ty_{t+1}-\ty_{t+1}^*\|^2 
        \leq (1-\mu\gamma)\|\ty_t-\ty_t^*\|^2 + 2\gamma\langle \tilde{\varepsilon}_t,\tilde{v}_t \rangle\|\ty_t-\ty_t^*\| + 2\gamma^2\|\tilde{\varepsilon}_t\|^2 + \frac{2}{\mu\gamma}D_t^2 \\
        &\quad\quad\quad\quad\leq (1-\mu\gamma)\|\ty_t-\ty_t^*\|^2 + 2\gamma\langle \tilde{\varepsilon}_t,\tilde{v}_t \rangle\|\ty_t-\ty_t^*\| + 2\gamma^2\|\tilde{\varepsilon}_t\|^2 + \frac{2l_{g,1}^2}{\mu^3\gamma}\|\tx_{t+1}-\tx_t\|^2 \\
        &\quad\quad\quad\quad\leq (1-\mu\gamma)\|\ty_t-\ty_t^*\|^2 + 2\gamma\langle \tilde{\varepsilon}_t,\tilde{v}_t \rangle\|\ty_t-\ty_t^*\| + 2\gamma^2\|\tilde{\varepsilon}_t\|^2 \\
        &\quad\quad\quad\quad\quad+ \frac{4\eta^2l_{g,1}^2}{\lambda^2\mu^3\gamma}\left(\|\tu_t\|^2 + \tl_t^2\sum_{j=1}^{t}d_{t,j}\|\ty_j-\ty_j^*\|^2\right),
    \end{aligned}
\end{equation}
% \end{small}%
where the last inequality uses \eqref{eq:condition-any-sequence}. 
Note that under \Cref{ass:noise}, there exists an absolute constant $c\geq1$ such that for all $t\geq1$, $\|\tilde{\varepsilon}_t\|^2$ is sub-exponential conditioned on $\tilde{\gF}_t^y$ with parameter $c\sigma_{g,1}^2$, and $\tilde{\varepsilon}_t$ is mean-zero sub-Gaussian conditioned on $\tilde{\gF}_t^y$ with parameter $c\sigma_{g,1}$ \citep[Theorem 30]{cutler2023stochastic}. For simplicity we set $c=1$ here. Thus $\langle \tilde{\varepsilon}_t, u_t \rangle$ is mean-zero sub-Gaussian conditioned on $\tilde{\gF}_t^y$ with parameter $\sigma_{g,1}$. 
% and $\Delta_t^2$ is sub-exponential conditioned on $\tilde{\gF}_t^y$ with parameter $\Delta^2$ by assumption. 
Hence, in light of \eqref{eq:in-light}, we apply \Cref{prop:MGF-recursive control} with
\begin{equation*}
    \begin{aligned}
        \gH_t = \tilde{\gF}_t^y,
        \quad
        V_t = \|\ty_t-\ty_t^*\|^2,
        \quad
        D_t = 2\eta\langle \tilde{\varepsilon}_t,\tilde{v}_t \rangle, 
        \quad
        X_t = 2\gamma^2\|\tilde{\varepsilon}_t\|^2,
    \end{aligned}
\end{equation*}
\begin{equation*}
    \begin{aligned}
        Y_t = \frac{4\eta^2l_{g,1}^2}{\lambda^2\mu^3\gamma}\tl_t^2\sum_{j=1}^{t}d_{t,j}\|\ty_j-\ty_j^*\|^2,
    \end{aligned}
\end{equation*}
\begin{equation*}
    \begin{aligned}
        \rho_t = 1-\mu\gamma,
        \quad
        \kappa_t = \frac{4\eta^2l_{g,1}^2}{\lambda^2\mu^3\gamma}\|\tu_t\|^2,
        \quad
        \sigma_t = 2\gamma\sigma_{g,1}, 
        \quad
        \nu_t = 2\gamma^2\sigma_{g,1}^2,
    \end{aligned}
\end{equation*}
yielding the following recursion
\begin{equation} \label{eq:one-step-recur}
    \begin{aligned}
        \E\left[\exp(\theta\tv_{t+1})\right]
        &\leq \E\left[\exp\left\{\theta\left[\left(1-\frac{\mu\gamma}{2}\right)\tv_t + 2\gamma^2\sigma_{g,1}^2 + \frac{4\eta^2l_{g,1}^2}{\lambda^2\mu^3\gamma}\|\tu_t\|^2 + \frac{4\eta^2l_{g,1}^2}{\lambda^2\mu^3\gamma}\tl_t^2\sum_{j=1}^{t}d_{t,j}\tv_j\right]\right\}\right]
    \end{aligned}
\end{equation}
for all $\theta$ satisfying
\begin{equation} \label{eq:lambda-range}
    \begin{aligned}
        0 \leq \theta \leq \min\left\{\frac{\mu}{8\gamma\sigma_{g,1}^2}, \frac{1}{4\gamma^2\sigma_{g,1}^2}\right\} \leq \frac{\mu}{8\gamma\sigma_{g,1}^2},
    \end{aligned}
\end{equation}
where in \eqref{eq:one-step-recur} we denote $\tv_t \coloneqq \|\ty_t-\ty_t^*\|^2$, and the last inequality of \eqref{eq:lambda-range} uses $\gamma\leq 1/2l_{g,1}\leq 1/2\mu$.
By \cref{lm:induction-proof} we use induction to show that for any $t\geq 1$ and $\lambda$ satisfying \eqref{eq:lambda-range}, it holds that 
{\allowdisplaybreaks
% \begin{equation*}
    \begin{align*}
        \E\left[\exp(\theta\tv_t)\right]
        &\leq \E\left[\exp\left\{\theta\left[\left(1-\frac{\mu\gamma}{2}\right)^{t-1}\tv_1 + \frac{4\gamma\sigma_{g,1}^2}{\mu} + \frac{4\eta^2l_{g,1}^2}{\lambda^2\mu^3\gamma}\sum_{i=1}^{t-1}\left(1-\frac{\mu\gamma}{2}\right)^{t-1-i}\|\tu_i\|^2 \right.\right.\right. \\
        &\left.\left.\left.\quad+ \frac{4\eta^2l_{g,1}^2}{\lambda^2\mu^3\gamma}\tv_1\sum_{i=1}^{t-1}\left(1-\frac{\mu\gamma}{2}\right)^{t-1-i}\tl_i^2 + \frac{16\eta^2l_{g,1}^2}{\lambda^2\mu^4}\sigma_{g,1}^2\sum_{i=1}^{t-1}\left(1-\frac{\mu\gamma}{2}\right)^{t-1-i}\tl_i^2 \right.\right.\right. \\
        &\left.\left.\left.\quad+ \frac{64\eta^4l_{g,1}^4}{\lambda^4\mu^8\gamma^4}\sum_{i=1}^{t-1}\left(1-\frac{\mu\gamma}{2}\right)^{t-1-i}\alpha_i\tl_i^2\|\tu_i\|^2 \right]\right\}\right],
    \end{align*}
% \end{equation*}
}%
where the first and the last lines use the sum of geometric series, and the second line is due to \cref{lm:algebra-talphat}:
\begin{equation*}
    \sum_{i=1}^{t-1}\left(1-\frac{\mu\gamma}{2}\right)^{i-1} \leq \frac{2}{\mu\gamma},
    \quad\quad
    i\alpha_i(1-\beta)^{i-1} \leq 1,
\end{equation*}
\begin{equation*}
    \sum_{i=1}^{t-1}\left(1-\frac{\mu\gamma}{2}\right)^{t-1-i}\alpha_i\tl_i^2\sum_{j=1}^{i}\left(1-\frac{\mu\gamma}{2}\right)^{i-j}\|\tu_j\|^2 \leq \frac{2}{\mu\gamma}\sum_{i=1}^{t-1}\left(1-\frac{\mu\gamma}{2}\right)^{t-1-i}\alpha_i\tl_i^2\|\tu_i\|^2.
\end{equation*}
Moreover, by setting $\vartheta$ as follows, we have
\begin{equation*}
    \vartheta \coloneqq \frac{8\gamma\sigma_{g,1}^2}{\mu}
    \quad\Longrightarrow\quad
    \frac{4\gamma\sigma_{g,1}^2}{\mu} \leq \vartheta
    \quad\text{and}\quad
    \frac{1}{\vartheta} = \frac{\mu}{8\gamma\sigma_{g,1}^2}.
\end{equation*}
Hence for any $t\geq 1$ we obtain
\begin{equation*}
    \begin{aligned}
        &\E\left[\exp\left\{\theta\left[\tv_t - \left(1-\frac{\mu\gamma}{2}\right)^{t-1}\tv_1 - \frac{4\eta^2l_{g,1}^2}{\lambda^2\mu^3\gamma}\sum_{i=1}^{t-1}\left(1-\frac{\mu\gamma}{2}\right)^{t-1-i}\|\tu_i\|^2 \right.\right.\right. \\
        &\left.\left.\left.\quad- \frac{4\eta^2l_{g,1}^2}{\lambda^2\mu^3\gamma}\tv_1\sum_{i=1}^{t-1}\left(1-\frac{\mu\gamma}{2}\right)^{t-1-i}\tl_i^2 - \frac{16\eta^2l_{g,1}^2}{\lambda^2\mu^4}\sigma_{g,1}^2\sum_{i=1}^{t-1}\left(1-\frac{\mu\gamma}{2}\right)^{t-1-i}\tl_i^2 \right.\right.\right. \\
        &\left.\left.\left.\quad- \frac{64\eta^4l_{g,1}^4}{\lambda^4\mu^8\gamma^4}\sum_{i=1}^{t-1}\left(1-\frac{\mu\gamma}{2}\right)^{t-1-i}\alpha_i\tl_i^2\|\tu_i\|^2 \right]\right\}\right]
        \leq \exp(\theta\vartheta)
        \quad\text{for all}\quad 
        0\leq \theta\leq 1/\vartheta.
    \end{aligned}
\end{equation*}
Taking $\theta=1/\vartheta$ and applying Markov's inequality and union bound completes the proof.
\end{proof}

%%%%%%%%%%%%%%%%%%%%%%%%%%%%%%%%%%%%%%%%%%%%%%%%%%%%%%%%%%%%%%%%%%%%%%%%%%%%%%%%%%%%%%%%%%%%%%%

\begin{lemma} \label{lm:induction-proof}
Suppose \eqref{eq:one-step-recur} holds, where $\tu_t$, $\{d_{t,j}\}_{j=1}^{t}$ and $\tl_t$ are defined in \eqref{eq:tut-def}, \eqref{eq:dj-def} and \eqref{eq:tilde-Gt-Lt-def}, respectively. Choosing $\gamma=2\beta/\mu$, then for any $t\geq 1$ we have
% \begin{small}
\begin{equation} \label{eq:induction-proof}
    \begin{aligned}
        \E\left[\exp(\theta\tv_t)\right]
        &\leq \E\left[\exp\left\{\theta\left[\left(1-\frac{\mu\gamma}{2}\right)^{t-1}\tv_1 + 2\gamma^2\sigma_{g,1}^2\sum_{i=1}^{t-1}\left(1-\frac{\mu\gamma}{2}\right)^{i-1} + \frac{4\eta^2l_{g,1}^2}{\lambda^2\mu^3\gamma}\sum_{i=1}^{t-1}\left(1-\frac{\mu\gamma}{2}\right)^{t-1-i}\|\tu_i\|^2 \right.\right.\right. \\
        &\left.\left.\left.\quad+ \frac{4\eta^2l_{g,1}^2}{\lambda^2\mu^3\gamma}\tv_1\sum_{i=1}^{t-1}\left(1-\frac{\mu\gamma}{2}\right)^{t-1-i}i\alpha_i(1-\beta)^{i-1}\tl_i^2 + \frac{16\eta^2l_{g,1}^2}{\lambda^2\mu^4}\sigma_{g,1}^2\sum_{i=1}^{t-1}\left(1-\frac{\mu\gamma}{2}\right)^{t-1-i}\tl_i^2 \right.\right.\right. \\
        &\left.\left.\left.\quad+ \frac{32\eta^4l_{g,1}^4}{\lambda^4\mu^7\gamma^3}\sum_{i=1}^{t-1}\left(1-\frac{\mu\gamma}{2}\right)^{t-1-i}\alpha_i\tl_i^2\sum_{j=1}^{i}\left(1-\frac{\mu\gamma}{2}\right)^{i-j}\|\tu_j\|^2 \right]\right\}\right].
    \end{aligned}
\end{equation}
% \end{small}%
\end{lemma}

\begin{proof}[Proof of \cref{lm:induction-proof}]
We use induction to show that \eqref{eq:induction-proof} holds for any $t\geq 1$ and $\lambda$ satisfying \eqref{eq:lambda-range}.
\paragraph{Base Case.}
For the base case $t=1$, it is easy to check that
\begin{equation*}
    \E[\exp(\theta\tv_1)] \leq \E[\exp(\theta\tv_1)].
\end{equation*}
\paragraph{Induction Step.}
Now we assume that the induction hypothesis \eqref{eq:induction-proof} holds for $1\leq k\leq t$, then for $k=t+1$ we have 
\begin{equation*}
    \begin{aligned}
        \E[\exp(\theta\tv_{t+1})] \leq \E[\exp(\theta[(A_1) + (A_2) + (A_3) + (A_4) + (A_5) + (A_6)])],
    \end{aligned}
\end{equation*}
where $(A_1), (A_2), (A_3), (A_4), (A_5)$ and $(A_6)$ are defined as
{\allowdisplaybreaks
% \begin{equation*}
    \begin{align*}
        (A_1) &= \left(1-\frac{\mu\gamma}{2}\right)\left(1-\frac{\mu\gamma}{2}\right)^{t-1}\tv_1, \\
        (A_2) &= 2\gamma^2\sigma_{g,1}^2 + 2\gamma^2\sigma_{g,1}^2\left(1-\frac{\mu\gamma}{2}\right)\sum_{i=1}^{t-1}\left(1-\frac{\mu\gamma}{2}\right)^{i-1}, \\
        (A_3) &= \frac{4\eta^2l_{g,1}^2}{\lambda^2\mu^3\gamma}\|\tu_t\|^2 + \frac{4\eta^2l_{g,1}^2}{\lambda^2\mu^3\gamma}\left(1-\frac{\mu\gamma}{2}\right)\sum_{i=1}^{t-1}\left(1-\frac{\mu\gamma}{2}\right)^{t-1-i}\|\tu_i\|^2, \\
        (A_4) &= \frac{4\eta^2l_{g,1}^2}{\lambda^2\mu^3\gamma}\tl_t^2\sum_{j=1}^{t}d_{t,j}\left(1-\frac{\mu\gamma}{2}\right)^{j-1}\tv_1 + \frac{4\eta^2l_{g,1}^2}{\lambda^2\mu^3\gamma}\tv_1\left(1-\frac{\mu\gamma}{2}\right)\sum_{i=1}^{t-1}\left(1-\frac{\mu\gamma}{2}\right)^{t-1-i}i\alpha_i(1-\beta)^{i-1}\tl_i^2, \\
        (A_5) &= \frac{4\eta^2l_{g,1}^2}{\lambda^2\mu^3\gamma}\tl_t^2\sum_{j=1}^{t}d_{t,j}\cdot 2\gamma^2\sigma_{g,1}^2\sum_{i=1}^{j-1}\left(1-\frac{\mu\gamma}{2}\right)^{i-1} + \frac{16\eta^2l_{g,1}^2}{\lambda^2\mu^4}\sigma_{g,1}^2\left(1-\frac{\mu\gamma}{2}\right)\sum_{i=1}^{t-1}\left(1-\frac{\mu\gamma}{2}\right)^{t-1-i}\tl_i^2, \\
        (A_6) &= \frac{4\eta^2l_{g,1}^2}{\lambda^2\mu^3\gamma}\tl_t^2\sum_{j=1}^{t}d_{t,j}\frac{4\eta^2l_{g,1}^2}{\lambda^2\mu^3\gamma}\sum_{i=1}^{j-1}\left(1-\frac{\mu\gamma}{2}\right)^{j-1-i}\|\tu_i\|^2 \\
        &\quad+ \frac{32\eta^4l_{g,1}^4}{\lambda^4\mu^7\gamma^3}\left(1-\frac{\mu\gamma}{2}\right)\sum_{i=1}^{t-1}\left(1-\frac{\mu\gamma}{2}\right)^{t-1-i}\alpha_i\tl_i^2\sum_{j=1}^{i}\left(1-\frac{\mu\gamma}{2}\right)^{i-j}\|\tu_j\|^2.
    \end{align*}
% \end{equation*}
}%
We continue to bound each term individually.
\paragraph{Bounding $(A_1)$.}
\begin{equation*}
    \begin{aligned}
        (A_1)
        = \left(1-\frac{\mu\gamma}{2}\right)\left(1-\frac{\mu\gamma}{2}\right)^{t-1}\tv_1
        = \left(1-\frac{\mu\gamma}{2}\right)^{t}\tv_1.
    \end{aligned}
\end{equation*}
\paragraph{Bounding $(A_2)$.}
\begin{equation*}
    \begin{aligned}
        (A_2)
        = 2\gamma^2\sigma_{g,1}^2 + 2\gamma^2\sigma_{g,1}^2\left(1-\frac{\mu\gamma}{2}\right)\sum_{i=1}^{t-1}\left(1-\frac{\mu\gamma}{2}\right)^{i-1} 
        = 2\gamma^2\sigma_{g,1}^2\left(1-\frac{\mu\gamma}{2}\right)\sum_{i=1}^{t}\left(1-\frac{\mu\gamma}{2}\right)^{i-1}.
    \end{aligned}
\end{equation*}
\paragraph{Bounding $(A_3)$.}
\begin{equation*}
    \begin{aligned}
        (A_3)
        = \frac{4\eta^2l_{g,1}^2}{\lambda^2\mu^3\gamma}\|\tu_t\|^2 + \frac{4\eta^2l_{g,1}^2}{\lambda^2\mu^3\gamma}\left(1-\frac{\mu\gamma}{2}\right)\sum_{i=1}^{t-1}\left(1-\frac{\mu\gamma}{2}\right)^{t-1-i}\|\tu_i\|^2
        = \frac{4\eta^2l_{g,1}^2}{\lambda^2\mu^3\gamma}\sum_{i=1}^{t}\left(1-\frac{\mu\gamma}{2}\right)^{t-i}\|\tu_i\|^2.
    \end{aligned}
\end{equation*}
\paragraph{Bounding $(A_4)$.}
By \Cref{lm:dj-def} and the choice of $\gamma=2\beta/\mu$, we have
\begin{equation*}
    \begin{aligned}
        \frac{4\eta^2l_{g,1}^2}{\lambda^2\mu^3\gamma}\tl_t^2\sum_{j=1}^{t}d_{t,j}\left(1-\frac{\mu\gamma}{2}\right)^{j-1}\tv_1
        &= \frac{4\eta^2l_{g,1}^2}{\lambda^2\mu^3\gamma}\tl_t^2\sum_{j=1}^{t}\alpha_t(1-\beta)^{t-j}(1-\beta)^{j-1}\tv_1 \\
        &= \frac{4\eta^2l_{g,1}^2}{\lambda^2\mu^3\gamma}\tl_t^2\sum_{j=1}^{t}\alpha_t(1-\beta)^{t-1}\tv_1 \\
        &= \frac{4\eta^2l_{g,1}^2}{\lambda^2\mu^3\gamma}t\alpha_t(1-\beta)^{t-1}\tl_t^2\tv_1.
    \end{aligned}
\end{equation*}
Then we obtain
\begin{equation*}
    \begin{aligned}
        (A_4)
        &= \frac{4\eta^2l_{g,1}^2}{\lambda^2\mu^3\gamma}\tl_t^2\sum_{j=1}^{t}d_{t,j}\left(1-\frac{\mu\gamma}{2}\right)^{j-1}\tv_1 + \frac{4\eta^2l_{g,1}^2}{\lambda^2\mu^3\gamma}\tv_1\left(1-\frac{\mu\gamma}{2}\right)\sum_{i=1}^{t-1}\left(1-\frac{\mu\gamma}{2}\right)^{t-1-i}i\alpha_i(1-\beta)^{i-1}\tl_i^2 \\
        &= \frac{4\eta^2l_{g,1}^2}{\lambda^2\mu^3\gamma}t\alpha_t(1-\beta)^{t-1}\tl_t^2\tv_1 + \frac{4\eta^2l_{g,1}^2}{\lambda^2\mu^3\gamma}\tv_1\left(1-\frac{\mu\gamma}{2}\right)\sum_{i=1}^{t-1}\left(1-\frac{\mu\gamma}{2}\right)^{t-1-i}i\alpha_i(1-\beta)^{i-1}\tl_i^2 \\
        &= \frac{4\eta^2l_{g,1}^2}{\lambda^2\mu^3\gamma}\tv_1\sum_{i=1}^{t}\left(1-\frac{\mu\gamma}{2}\right)^{t-i}i\alpha_i(1-\beta)^{i-1}\tl_i^2.
    \end{aligned}
\end{equation*}
\paragraph{Bounding $(A_5)$.}
By \Cref{lm:dj-def} and the choice of $\gamma=2\beta/\mu$, we have
\begin{equation*}
    \begin{aligned}
        \frac{4\eta^2l_{g,1}^2}{\lambda^2\mu^3\gamma}\tl_t^2\sum_{j=1}^{t}d_{t,j}\cdot 2\gamma^2\sigma_{g,1}^2\sum_{i=1}^{j-1}\left(1-\frac{\mu\gamma}{2}\right)^{i-1}
        &\leq \frac{8\eta^2l_{g,1}^2}{\lambda^2\mu^4\gamma^2}2\gamma^2\sigma_{g,1}^2\tl_t^2\sum_{j=1}^{t}d_{t,j} \\
        % &= \frac{8\eta^2l_{g,1}^2}{\lambda^2\mu^4\gamma^2}2\gamma^2\sigma_{g,1}^2\tl_t^2 \\
        &= \frac{16\eta^2l_{g,1}^2}{\lambda^2\mu^4}\sigma_{g,1}^2\tl_t^2.
    \end{aligned}
\end{equation*}
Then we obtain
\begin{equation*}
    \begin{aligned}
        (A_5) 
        &= \frac{4\eta^2l_{g,1}^2}{\lambda^2\mu^3\gamma}\tl_t^2\sum_{j=1}^{t}d_{t,j}\cdot 2\gamma^2\sigma_{g,1}^2\sum_{i=1}^{j-1}\left(1-\frac{\mu\gamma}{2}\right)^{i-1} + \frac{16\eta^2l_{g,1}^2}{\lambda^2\mu^4}\sigma_{g,1}^2\left(1-\frac{\mu\gamma}{2}\right)\sum_{i=1}^{t-1}\left(1-\frac{\mu\gamma}{2}\right)^{t-1-i}\tl_i^2 \\
        &\leq \frac{16\eta^2l_{g,1}^2}{\lambda^2\mu^4}\sigma_{g,1}^2\tl_t^2 + \frac{16\eta^2l_{g,1}^2}{\lambda^2\mu^4}\sigma_{g,1}^2\sum_{i=1}^{t-1}\left(1-\frac{\mu\gamma}{2}\right)^{t-i}\tl_i^2 \\
        &= \frac{16\eta^2l_{g,1}^2}{\lambda^2\mu^4}\sigma_{g,1}^2\sum_{i=1}^{t}\left(1-\frac{\mu\gamma}{2}\right)^{t-i}\tl_i^2.
    \end{aligned}
\end{equation*}
\paragraph{Bounding $(A_6)$.}
By \Cref{lm:dj-def} and the choice of $\gamma=2\beta/\mu$, we have
\begin{equation*}
    \begin{aligned}
        \frac{4\eta^2l_{g,1}^2}{\lambda^2\mu^3\gamma}\tl_t^2\sum_{j=1}^{t}d_{t,j}\frac{4\eta^2l_{g,1}^2}{\lambda^2\mu^3\gamma}\sum_{i=1}^{j-1}\left(1-\frac{\mu\gamma}{2}\right)^{j-1-i}\|\tu_i\|^2
        &\leq \frac{4\eta^2l_{g,1}^2}{\lambda^2\mu^3\gamma}\frac{8\eta^2l_{g,1}^2}{\lambda^2\mu^4\gamma^2}\tl_t^2\sum_{j=1}^{t}d_{t,j}\|\tu_j\|^2 \\
        &\leq \frac{4\eta^2l_{g,1}^2}{\lambda^2\mu^3\gamma}\frac{8\eta^2l_{g,1}^2}{\lambda^2\mu^4\gamma^2}\alpha_t\tl_t^2\sum_{j=1}^{t}(1-\beta)^{t-j}\|\tu_j\|^2 \\
        &= \frac{32\eta^4l_{g,1}^4}{\lambda^4\mu^7\gamma^3}\alpha_t\tl_t^2\sum_{j=1}^{t}\left(1-\frac{\mu\gamma}{2}\right)^{t-j}\|\tu_j\|^2.
    \end{aligned}
\end{equation*}
Then we obtain
\begin{equation*}
    \begin{aligned}
        (A_6) 
        &= \frac{4\eta^2l_{g,1}^2}{\lambda^2\mu^3\gamma}\tl_t^2\sum_{j=1}^{t}d_{t,j}\frac{4\eta^2l_{g,1}^2}{\lambda^2\mu^3\gamma}\sum_{i=1}^{j-1}\left(1-\frac{\mu\gamma}{2}\right)^{j-1-i}\|\tu_i\|^2 \\
        &\quad+ \frac{32\eta^4l_{g,1}^4}{\lambda^4\mu^7\gamma^3}\left(1-\frac{\mu\gamma}{2}\right)\sum_{i=1}^{t-1}\left(1-\frac{\mu\gamma}{2}\right)^{t-1-i}\alpha_i\tl_i^2\sum_{j=1}^{i}\left(1-\frac{\mu\gamma}{2}\right)^{i-j}\|\tu_j\|^2 \\
        &\leq \frac{32\eta^4l_{g,1}^4}{\lambda^4\mu^7\gamma^3}\alpha_t\tl_t^2\sum_{j=1}^{t}\left(1-\frac{\mu\gamma}{2}\right)^{t-j}\|\tu_j\|^2 + \frac{32\eta^4l_{g,1}^4}{\lambda^4\mu^7\gamma^3}\sum_{i=1}^{t-1}\left(1-\frac{\mu\gamma}{2}\right)^{t-i}\alpha_i\tl_i^2\sum_{j=1}^{i}\left(1-\frac{\mu\gamma}{2}\right)^{i-j}\|\tu_j\|^2 \\
        &= \frac{32\eta^4l_{g,1}^4}{\lambda^4\mu^7\gamma^3}\sum_{i=1}^{t}\left(1-\frac{\mu\gamma}{2}\right)^{t-i}\alpha_i\tl_i^2\sum_{j=1}^{i}\left(1-\frac{\mu\gamma}{2}\right)^{i-j}\|\tu_j\|^2.
    \end{aligned}
\end{equation*}
\paragraph{Final Bound for the Induction Step.}
Putting these terms together and rearranging yields
\begin{equation*}
    \begin{aligned}
        \E\left[\exp(\theta\tv_{t+1})\right]
        &\leq \E\left[\exp\left\{\theta\left[\left(1-\frac{\mu\gamma}{2}\right)^{t}\tv_1 + 2\gamma^2\sigma_{g,1}^2\sum_{i=1}^{t}\left(1-\frac{\mu\gamma}{2}\right)^{i-1} + \frac{4\eta^2l_{g,1}^2}{\lambda^2\mu^3\gamma}\sum_{i=1}^{t}\left(1-\frac{\mu\gamma}{2}\right)^{t-i}\|\tu_i\|^2 \right.\right.\right. \\
        &\left.\left.\left.\quad+ \frac{4\eta^2l_{g,1}^2}{\lambda^2\mu^3\gamma}\tv_1\sum_{i=1}^{t}\left(1-\frac{\mu\gamma}{2}\right)^{t-i}i\alpha_i(1-\beta)^{i-1}\tl_i^2 + \frac{16\eta^2l_{g,1}^2}{\lambda^2\mu^4}\sigma_{g,1}^2\sum_{i=1}^{t}\left(1-\frac{\mu\gamma}{2}\right)^{t-i}\tl_i^2 \right.\right.\right. \\
        &\left.\left.\left.\quad+ \frac{32\eta^4l_{g,1}^4}{\lambda^4\mu^7\gamma^3}\sum_{i=1}^{t}\left(1-\frac{\mu\gamma}{2}\right)^{t-i}\alpha_i\tl_i^2\sum_{j=1}^{i}\left(1-\frac{\mu\gamma}{2}\right)^{i-j}\|\tu_j\|^2 \right]\right\}\right],
    \end{aligned}
\end{equation*}
which aligns with \eqref{eq:induction-proof} for $k=t+1$. Thus, the induction step is complete, and \eqref{eq:induction-proof} holds for any $t \geq 1$.
\end{proof}

\section{Convergence Analysis of AdamBO (\cref{alg:bi-adam})}
\label{app:proof_biadam}
In this section, we provide detailed convergence analysis of \cref{alg:bi-adam} (or equivalently, \cref{alg:bi-adam-equiv}). Before presenting the lemmas and the main theorem, we will first define (or restate) a few key concepts and useful notations.
% and repeat some important technical definitions here.

\subsection{Technical Definitions and Useful Notations}
\label{sec:adambo-notations}

\paragraph{Filtration.}
Define $\gF_{\init}$ as the filtration for updating $y_1$ (i.e., the filtration of warm-start phase):
\begin{equation*}
    \gF_t^{\init} = \sigma(\pi_0,\dots,\pi_{T_0-1}).
\end{equation*}
For any $t \geq 2$, define $\gF_t^x$ and $\gF_t^y$ as the filtrations of the randomness used in updating $x_t$ and $y_t$, respectively, before the $t$-th iteration:
\begin{equation*}
    % \gF_t^{\init} = \sigma(\pi_1,\dots,\pi_{t-1}),
    % \quad
    \gF_t^x = \sigma(\Bar{\xi}_1,\dots,\Bar{\xi}_{t-1}),
    \quad
    \gF_t^y = \sigma(\zeta_1,\dots,\zeta_{t-1}),
\end{equation*}
where $\sigma(\cdot)$ denotes the $\sigma$-algebra generated by the random variables within the argument. Additionally, let $\gF_t$ denote the filtration of all randomness before the $t$-th iteration:
\begin{equation*}
    \gF_t = \sigma(\gF_{\init} \cup \gF_t^x \cup \gF_t^y).
\end{equation*}

\paragraph{Expectation.}
We use $\E_t[\cdot]$ to denote the conditional expectation $\E[\cdot \mid \gF_t]$.

\paragraph{Auxiliary Sequence.}
Note that $\hm_t$ (line 7 of \Cref{alg:bi-adam-equiv}) can be written as
\begin{equation} \label{eq:hm-def}
    \begin{aligned}
        \hm_t 
        = (1-\alpha_t)\hm_{t-1} + \alpha_t\hatphi(x_t,y_t;\Bar{\xi}_t) 
        = \sum_{j=1}^{t}d_{t,j}\hatphi(x_t,y_t;\Bar{\xi}_t).
    \end{aligned}
\end{equation}
Similar to \cref{sec:proof-randomness-decouple}, we introduce the following auxiliary sequence $\{\hu_t\}$ for our analysis:
\begin{equation} \label{eq:hu-def}
    \begin{aligned}
        \hu_t 
        = (1-\alpha_t)\hu_{t-1} + \alpha_t\hatphi(x_t,y_t^*;\Bar{\xi}_t) 
        = \sum_{j=1}^{t}d_{t,j}\hatphi(x_t,y_t^*;\Bar{\xi}_t).
    \end{aligned}
\end{equation}

\paragraph{Other Definitions.}
We define the deviation of the rescaled auxiliary momentum from the conditional expectation of the hypergradient estimator as
\begin{equation} \label{eq:epst-def}
    \epsilon_t \coloneqq \hu_t - \E_t[\hatphi(x_t,y_t^*;\Bar{\xi}_t)].
\end{equation}
Also, let $h_t$ be the learning rate vector and $H_t$ be the learning rate matrix:
\begin{equation} \label{eq:ht-odot}
    h_t \coloneqq \frac{\eta}{\sqrt{\hv_t}+\lambda}
    \quad\quad\text{and}\quad\quad
    H_t \coloneqq \diag(h_t).
\end{equation}
Then the update rule for upper-level variable $x_t$ (line 10 of \cref{alg:bi-adam}) can be written as
\begin{equation}
    x_{t+1} = x_t - h_t\odot\hm_t = x_t - H_t\hm_t.
\end{equation}

\paragraph{Stopping Time.}
Given a large enough constant $G$ as defined in \cref{thm:main-appendix}, denote $L$ and $\psi$ as
\begin{equation} \label{eq:psi-def}
    L = L_0 + L_1G
    \quad\quad\text{and}\quad\quad
    \psi = \frac{C_LG^2}{2L},
\end{equation}
where constants $L_0,L_1$ and $C_L$ are defined in \eqref{eq:L0-L1} and \eqref{eq:Cu-def}. Now we formally define the stopping time $\tau$ as
\begin{equation} \label{eq:tau-def}
    \tau \coloneqq \min\{t \mid \Phi(x_t)-\Phi^* > \psi\} \wedge (T+1).
\end{equation}
In other words, $\tau$ is the first time when the sub-optimality gap is strictly larger than $\psi$, truncated at $T+1$ to make sure it is bounded. Based on \cref{lm:reverse-PL}, we know that if $t<\tau$, we have both $\Phi(x_t)-\Phi^*\leq \psi$ and $\|\gdphi(x_t)\|\leq G$.

\paragraph{Constants.}
We define the following constants, which will be useful for analysis.
\begin{equation} \label{eq:L-def}
    G_t = \max_{1\leq k\leq t}\|\gdphi(x_k)\|,
    \quad
    \hl_t = L_0 + L_1G_t,
    \quad
    L = L_0 + L_1G,
    \quad
    \Delta_1 = \Phi(x_1) - \Phi^*,
\end{equation}
\begin{equation} \label{eq:Cu-def}
    C_L = \frac{L_{x,1}}{\sqrt{L_{x,1}^2+L_{y,1}^2}},
    \quad
    C_{u,0} = C_{\phi,0} + G,
    \quad
    C_{u,1} = C_{\phi,1} + L_1G,
\end{equation}
\begin{equation} \label{eq:sigma-phi-def}
    \sigma_{\phi} = \frac{\mu+3l_{g,1}+\sigma_{g,2}}{\mu} + \frac{2l_{g,1}+\sigma_{g,2}}{\mu}l_{f,0} + \frac{2l_{g,1}+\sigma_{g,2}}{\mu}(L_{y,0}+L_{y,1}l_{f,0})r.
\end{equation}
\begin{equation} \label{eq:Cbeta-def}
    \begin{aligned}
        C_{\beta} \geq \max&\left\{\frac{8e\sigma_{\phi}^4G^2\max\{1,\iota\}}{c_1^2\delta\lambda^2\epsilon^4}, \frac{8C_2e\Delta_1L\sigma_{\phi}G^3}{c_1c_2\delta\lambda^2\epsilon^4}\left(1 + \frac{\sigma_{\phi}^2G}{c_1\lambda\epsilon^2}\right)\max\{1,\sqrt{\iota},\iota\}, \right. \\
        &\left. \quad\quad \left(\frac{32e\sigma_{\phi}^4G^2}{c_1^2\delta\lambda^2\epsilon^4}\right)^2, \left(\frac{48C_2e\Delta_1L\sigma_{\phi}G^3}{c_1c_2\delta\lambda^2\epsilon^4}\left(1 + \frac{\sigma_{\phi}^2G}{c_1\lambda\epsilon^2}\right)\max\{1,\sqrt{\iota},\iota\}\right)^2\right\}.
    \end{aligned}
\end{equation}
Besides, constants $L_0,L_1$ are defined in \eqref{eq:L0-L1}, $C_{\phi,0}, C_{\phi,1}$ are defined in \eqref{eq:C-phi-01}, and $r$ is defined in \eqref{eq:r-def}, respectively.

%%%%%%%%%%%%%%%%%%%%%%%%%%%%%%%%%%%%%%%%%%%%%%%%%%%%%%%%%%%%%%%%%%%%%%%%%%%%%%%%%%%%%%%%%%%%%%%

\subsection{Auxiliary Lemmas}

We first introduce the following useful lemma, which is crucial for the subsequent stopping time analysis and for establishing the contradiction argument.

\begin{lemma} \label{lm:reverse-PL}
Under \cref{ass:bilevel-assumption}, we have 
\begin{equation*}
    \begin{aligned}
        \|\gdphi(x)\|^2 \leq \frac{2}{C_L}(L_0+L_1\|\gdphi(x)\|)(\Phi(x)-\Phi^*), 
    \end{aligned}
\end{equation*}
where constants $L_0,L_1$ and $C_L$ are defined in \eqref{eq:L0-L1} and \eqref{eq:Cu-def}. Further, for any given constant $G>0$, if we denote $\psi$ as in \eqref{eq:psi-def} and $\Phi(x)-\Phi^*\leq \psi$, then we have $\|\gdphi(x)\|\leq G$. 
\end{lemma}

\begin{proof}[Proof of \cref{lm:reverse-PL}]
Let $x'$ be 
\begin{equation*}
    x' = x - \frac{C_L\|\gdphi(x)\|}{L_0+L_1\|\gdphi(x)\|},
\end{equation*}
then we have
\begin{equation*}
    \begin{aligned}
        \|x' - x\| = \frac{C_L\|\gdphi(x)\|}{L_0+L_1\|\gdphi(x)\|} \leq \frac{C_L}{L_1} = \frac{1}{\sqrt{(1+l_{g,1}^2/\mu^2)(L_{x,1}^2+L_{y,1}^2)}} = r,
    \end{aligned}
\end{equation*}
where the inequality can be verified by considering both cases of $\|\gdphi(x)\|\leq L_0/L_1$ and $\|\gdphi(x)\|\geq L_0/L_1$. By \cref{lm:Phi-relax-smooth}, we have
% Denote $G,L$ and $C_L$ as
% \begin{equation} \label{eq:CL-def}
%     G\coloneqq\|\gdphi(x)\|, 
%     \quad
%     L\coloneqq L_0+L_1G,
%     \quad
%     C_L \coloneqq \frac{L_{x,1}}{\sqrt{L_{x,1}^2+L_{y,1}^2}}.
% \end{equation}
% Let $x'=x-\frac{C_L}{L}\gdphi(x)$, then we have
% \begin{equation*}
%     \begin{aligned}
%         \|x' - x\| = \frac{C_LG}{L} = \frac{C_LG}{L_0+L_1G} \leq \frac{C_L}{L_1} = \frac{1}{\sqrt{(1+l_{g,1}^2/\mu^2)(L_{x,1}^2+L_{y,1}^2)}} = r,
%     \end{aligned}
% \end{equation*}
% where the inequality can be verified by considering both cases of $G\leq L_0/L_1$ and $G\geq L_0/L_1$. By \cref{lm:Phi-relax-smooth}, we have
\begin{equation*}
    \begin{aligned}
        \Phi^* - \Phi(x) 
        &\leq \Phi(x') - \Phi(x) \leq \langle \gdphi(x), x'-x \rangle + \frac{L_0+L_1\|\gdphi(x)\|}{2}\|x'-x\|^2 \\
        &= -\frac{C_L(2-C_L)}{2(L_0+L_1\|\gdphi(x)\|)}\|\gdphi(x)\|^2.
    \end{aligned}
\end{equation*}
Rearranging the above inequality yields
\begin{equation} \label{eq:reverse-PL}
    \begin{aligned}
        \|\gdphi(x)\|^2 
        &\leq \frac{2(L_0+L_1\|\gdphi(x)\|)}{C_L(2-C_L)}(\Phi(x)-\Phi^*) \leq \frac{2(L_0+L_1\|\gdphi(x)\|)}{C_L}(\Phi(x)-\Phi^*).
    \end{aligned}
\end{equation}
where the last inequality uses the definition of $C_L$ in \eqref{eq:Cu-def} and $C_L\leq 1$.

Now define the function $\varphi:\R_0^+\rightarrow\R$ as 
\begin{equation*}
    \varphi(u)\coloneqq\frac{C_Lu^2}{2(L_0+L_1u)}.
\end{equation*}
It is easy to verify $\varphi$ is increasing and $\varphi(u)\in[0,\infty)$. Thus, $\varphi$ is invertible and $\varphi^{-1}$ is also increasing. Then for any constant $G\geq0$, denote $L$ and $\psi$ as in \eqref{eq:psi-def},
\begin{equation*}
    L = L_0 + L_1G,
    \quad\quad
    \psi = \frac{C_LG^2}{2L} = \varphi(G).
\end{equation*}
The property of function $\varphi^{-1}$ and \eqref{eq:reverse-PL} imply that if $\Phi(x)-\Phi^*\leq \psi$, we have
\begin{equation*}
    \|\gdphi(x)\| \leq \varphi^{-1}(\Phi(x)-\Phi^*) \leq \varphi^{-1}(\psi) = G.
\end{equation*}
\end{proof}

%%%%%%%%%%%%%%%%%%%%%%%%%%%%%%%%%%%%%%%%%%%%%%%%%%%%%%%%%%%%%%%%%%%%%%%%%%%%%%%%%%%%%%%%%%%%%%%
Note that when $t<\tau$, some of the quantities in \cref{alg:bi-adam,sec:adambo-notations} are bounded almost surely. In particular, we have the following lemma.

\begin{lemma} \label{lm:gradient-bound}
If $t<\tau$, we have
\begin{equation*}
    \begin{aligned}
        \|\gdphi(x_t)\| \leq G,
        \quad
        \hl_t \leq L,
        \quad
        \|\hu_t\| \leq C_{u,0},
        \quad
        h_t \preceq \frac{\eta}{\lambda}, 
        \quad
        \|H_t\| \preceq \frac{\eta}{\lambda}.
    \end{aligned}
\end{equation*}
where $h_t$ is defined in \eqref{eq:ht-odot}, constants $\hl_t, L$ and $C_{u,0}$ are defined in \eqref{eq:L-def} and \eqref{eq:Cu-def}, respectively.
\end{lemma}

\begin{proof}[Proof of \Cref{lm:gradient-bound}]
By \cref{lm:reverse-PL} and definition of $\tau$, we have $\|\gdphi(x_t)\|\leq G$ if $t<\tau$. 
Also, recall the definition of $G_t$, $\hl_t$ and $L$ as in \eqref{eq:L-def}, we have $G_t=\max_{k\leq t}\|\gdphi(x_k)\|\leq G$ if $t<\tau$, and hence gives $\hl_t=L_0+L_1G_t\leq L_0+L_1G=L$. Before bounding $\|\hu_t\|$, we first show $\|\hatphi(x_t,y_t^*;\Bar{\xi}_t)\|\leq C_{u,0}$. \Cref{lm:stoc-hyper-error} directly implies that if $t<\tau$, then
\begin{equation*}
    \begin{aligned}
        \|\hatphi(x_t,y_t^*;\Bar{\xi}_t)\| 
        \leq C_{\phi,0} + (C_{\phi,1} + L_1\|\gdphi(x_t)\|)\|y_t^*-y_t^*\| + \|\gdphi(x_t)\|
        \leq C_{\phi,0} + G
        = C_{u,0},
    \end{aligned}
\end{equation*}
where the last equality is due to the definition of $C_{u,0}$ in \eqref{eq:Cu-def}.
Now $\|\hu_t\|$ can be bounded by a standard induction argument as follows. First, for the base case $k=1$, note that $\|\hatphi(x_1,y_1^*;\Bar{\xi}_1)\| \leq C_{u,0}$.
Suppose $\|\hu_{k-1}\|\leq C_{u,0}$ for some $k<\tau$, then by update rule of $\hu_k$ in \eqref{eq:hu-def} we have
\begin{equation*}
    \begin{aligned}
        \|\hu_k\| 
        \leq (1-\alpha_k)\|\hu_{k-1}\| + \alpha_k\|\hatphi(x_k,y_k;\Bar{\xi}_k)\|
        \leq C_{u,0}.
    \end{aligned}
\end{equation*}
Therefore, the induction is complete. The last two results directly follow from the definitions of $h_t$ and $H_t$ in \eqref{eq:ht-odot}.
\end{proof}

%%%%%%%%%%%%%%%%%%%%%%%%%%%%%%%%%%%%%%%%%%%%%%%%%%%%%%%%%%%%%%%%%%%%%%%%%%%%%%%%%%%%%%%%%%%%%%%
\subsection{Proof of \cref{lm:warm-start-main}}
\label{sec:proof-warm-start}

In the next lemma, we provide high probability bound for the warm-start phase.

\begin{lemma}[Warm-Start, Restatement of \cref{lm:warm-start-main}] \label{lm:warm-start}
Suppose that \cref{ass:bilevel-assumption,ass:noise} hold. Let $\{y_t^{\init}\}$ be the iterates generated by \cref{alg:sgd} with constant learning rate $\gamma\leq 1/2l_{g,1}$. Then for any given $\delta\in(0,1)$, the following estimate holds with probability at least $1-\delta/4$ over the randomness in $\gF_{\init}$ (we denote this event as $\gE_0$):
\begin{equation} \label{eq:varrho-1-def}
    \|y_1-y_1^*\|^2 \leq \left(1-\frac{\mu\gamma}{2}\right)^{T_0}\|y_0-y_0^*\|^2 + \frac{8\gamma\sigma_{g,1}^2}{\mu}\ln\frac{4e}{\delta}.
\end{equation}
\end{lemma}

\begin{proof}[Proof of \cref{lm:warm-start}]
For any given $\delta\in(0,1)$ and any fixed $t\geq 0$, we invoke \citep[Theorem 30]{cutler2023stochastic} to obtain that 
\begin{equation}
    \|y_t^{\init}-y_0^*\|^2 \leq \left(1-\frac{\mu\gamma}{2}\right)^t\|y_0-y_0^*\|^2 + \frac{8\gamma\sigma_{g,1}^2}{\mu}\ln\frac{4e}{\delta}
\end{equation}
holds with probability at least $1-\delta$ over the randomness in $\gF_{\init}$. Set $t=T_0$ and then we have
\begin{equation*}
    \|y_1-y_1^*\|^2 = \|y_{T_0}^{\init}-y_0^*\|^2 \leq \left(1-\frac{\mu\gamma}{2}\right)^{T_0}\|y_0-y_0^*\|^2 + \frac{8\gamma\sigma_{g,1}^2}{\mu}\ln\frac{4e}{\delta},
\end{equation*}
where the first equality is due to $y_1=y_{T_0}^{\init}$ and $y_1^*=y_0^*$ (since $x_1=x_0$) by line 2 of \cref{alg:bi-adam}.
\end{proof}

%%%%%%%%%%%%%%%%%%%%%%%%%%%%%%%%%%%%%%%%%%%%%%%%%%%%%%%%%%%%%%%%%%%%%%%%%%%%%%%%%%%%%%%%%%%%%%%
\subsection{Proof of \cref{lm:event-y}}
\label{sec:proof-event-y}
The following \cref{lm:general-recursion} (i.e., the full statement of \cref{lm:event-y}) is a direct application of the randomness decoupling lemma (i.e., \cref{lm:any-sequence-main}) to the actual sequences $\{x_t\}, \{y_t\}$ in \cref{alg:bi-adam}.

\begin{lemma}[Restatement of \cref{lm:event-y}] \label{lm:general-recursion}
Suppose that \Cref{ass:bilevel-assumption,ass:noise,ass:bilevel-additional} hold. Let $\{y_t\}$ be the iterates generated by \Cref{alg:bi-adam}. Under the parameter choices in \cref{thm:main-appendix}, let $\eta$ further satisfy
\begin{equation} \label{eq:eta-additional}
    \begin{aligned}
        \eta \leq c_2\min\left\{\frac{r\lambda}{G_{T}}, \frac{\lambda}{6L}, \frac{\sigma_{\phi}\lambda\beta}{\hl_TG_T\max\{1, \sqrt{\iota}, \ln(1/\beta), \ln(C_{\beta})\}}, \frac{\lambda^{3/2}\beta}{\hl_T\sqrt{G_{T}}}\right\}, 
    \end{aligned}
\end{equation}
then for any given $\delta\in(0,1)$ and all $t\geq 1$, the following estimate holds with probability at least $1-\delta/4$ over the randomness in $\gF_{T+1}^y$ (we denote this event as $\gE_{y}$):
\begin{equation} \label{eq:general-recursion}
    \begin{aligned}
        \|y_t-y_t^*\|^2 
        &\leq \left(\left(1-\frac{\mu\gamma}{2}\right)^{t-1} + \frac{4\eta^2l_{g,1}^2}{\lambda^2\mu^3\gamma}\sum_{i=1}^{t-1}\left(1-\frac{\mu\gamma}{2}\right)^{t-1-i}\hl_i^2\right)\|y_1-y_1^*\|^2 \\
        &\quad+ \left(\frac{8\gamma}{\mu}\ln\frac{4eT}{\delta} + \frac{16\eta^2l_{g,1}^2}{\lambda^2\mu^4}\sum_{i=1}^{t-1}\left(1-\frac{\mu\gamma}{2}\right)^{t-1-i}\hl_i^2\right)\sigma_{g,1}^2 \\
        &\quad+ \frac{4\eta^2l_{g,1}^2}{\lambda^2\mu^3\gamma}\sum_{i=1}^{t-1}\left(1-\frac{\mu\gamma}{2}\right)^{t-1-i}\|\hu_i\|^2 + \frac{64\eta^4l_{g,1}^4}{\lambda^4\mu^8\gamma^4}\sum_{i=1}^{t-1}\left(1-\frac{\mu\gamma}{2}\right)^{t-1-i}\alpha_i\hl_i^2\|\hu_i\|^2,
    \end{aligned}
\end{equation}
where constant $\hl_i$ and sequence $\{\hu_i\}$ are defined in \eqref{eq:L-def} and \eqref{eq:hu-def}, respectively.
\end{lemma}

\begin{proof}[Proof of \Cref{lm:general-recursion}]
First, with the parameter choices in \cref{thm:main-appendix} and the additional choice for $\eta$ as in \eqref{eq:eta-additional}, we can follow the same procedure as \cref{lm:para-verify} (see ``Verification for $\varrho \leq \min\{r,1/4L_1\}$'') to show that $\|y_t-y_t^*\|\leq r$ for all $t\in[T]$. Thus, the condition for applying \cref{lm:hyper-stoc-bias} is satisfied.
Recall the definitions of $\hm_t$ and $\hu_t$ in \eqref{eq:hm-def} and \eqref{eq:hu-def}, we have
\begin{equation} \label{eq:hm-hu-diff}
    \begin{aligned}
        \|\hm_t-\hu_t\|^2
        &\leq \left\|\sum_{j=1}^{t}d_{t,j}(\hatphi(x_j,y_j;\Bar{\xi}_j) - \hatphi(x_j,y_j^*;\Bar{\xi}_j))\right\|^2 \\
        &\leq \sum_{j=1}^{t}d_{t,j}\|\hatphi(x_j,y_j;\Bar{\xi}_j) - \hatphi(x_j,y_j^*;\Bar{\xi}_j)\|^2 \\
        &\leq \sum_{j=1}^{t}d_{t,j}(L_0+L_1\|\gdphi(x_j)\|)^2\|y_j-y_j^*\|^2 
        \leq \hl_t^2\sum_{j=1}^{t}d_{t,j}\|y_j-y_j^*\|^2,
    \end{aligned}
\end{equation}
where the second inequality uses Jensen's inequality, the third inequality is due to \Cref{lm:hyper-stoc-bias}, and the last inequality uses the definition of $\hl_t$ in \eqref{eq:L-def}. By the update rule in \Cref{alg:bi-adam-equiv}, we have
\begin{equation*}
    \begin{aligned}
        \|x_{t+1}-x_t\|^2
        &\leq \|H_t\|^2\|\hm_t\|^2 
        \leq \frac{\eta^2}{\lambda^2}\|\hm_t\|^2
        \leq \frac{2\eta^2}{\lambda^2}\left(\|\hu_t\|^2 + \|\hm_t-\hu_t\|^2\right) \\
        &\leq \frac{2\eta^2}{\lambda^2}\left(\|\hu_t\|^2 + \hl_t^2\sum_{j=1}^{t}d_{t,j}\|y_j-y_j^*\|^2\right),
    \end{aligned}
\end{equation*}
where the first inequality uses \eqref{eq:ht-odot}; the second inequality is due to \cref{lm:gradient-bound}; the third inequality uses Young's inequality; and the last inequality is due to \eqref{eq:hm-hu-diff}. This implies that the sequence $\{x_t\}$ and the randomness $\{\Bar{\xi}_t\}$ generated by \cref{alg:bi-adam} satisfy the condition \eqref{eq:condition-any-sequence} in \Cref{lm:any-sequence}. Therefore, the result follows by applying \Cref{lm:any-sequence} with $\{\tx_t\}=\{x_t\}$ and $\{\hat{\xi}_t\}=\{\Bar{\xi}_t\}$.
\end{proof}

\textbf{Remark.} In the end, we will show $\tau=T+1$ in the proof of \cref{thm:main-appendix} (i.e., the Full statement of \cref{thm:main}), thus we can apply \cref{lm:gradient-bound} to obtain $G_T\leq G$ and $\hl_T\leq L$. This suggests that under event $\gE_0\cap\gE_y$, the additional requirement \eqref{eq:eta-additional} is actually included in the parameter choices of \cref{thm:main-appendix}. Therefore, there is no need to worry about this temporary iterate-dependent requirement for the choice of $\eta$.

%%%%%%%%%%%%%%%%%%%%%%%%%%%%%%%%%%%%%%%%%%%%%%%%%%%%%%%%%%%%%%%%%%%%%%%%%%%%%%%%%%%%%%%%%%%%%%%
\subsection{Proof of \cref{lem:upperlevel}}
\label{sec:proof-momentum}
Before proving \cref{lem:upperlevel}, first note that when $t<\tau$ and $\gE_0\cap\gE_y$ holds, some of the time-dependent quantities (such as $\hl_t$ and $\|\hu_t\|$) in \cref{lm:general-recursion} can be well bounded by \cref{lm:gradient-bound}. In particular, we have the following two high probability bounds for the lower-level approximation error $\|y_t-y_t^*\|$: the first one, \eqref{eq:recur-1}, is useful for the convergence analysis; and the second one, \eqref{eq:recur-2}, is crucial for proving \cref{lm:momentum-bound,lm:rs-condition}.

\begin{lemma} \label{lm:recursion-two-type}
Under event $\gE_0\cap\gE_{y}$ and the parameter choices in \cref{lm:general-recursion}, if $t\leq \tau$, we have 
\begin{equation} \label{eq:recur-1}
    \begin{aligned}
        \|y_t-y_t^*\|^2 
        &\leq \left(\left(1-\frac{\mu\gamma}{2}\right)^{t-1} + \frac{8\eta^2l_{g,1}^2L^2}{\lambda^2\mu^4\gamma^2}\right) \|y_1-y_1^*\|^2 + \left(\frac{8\gamma}{\mu}\ln\frac{4eT}{\delta} + \frac{32\eta^2l_{g,1}^2L^2}{\lambda^2\mu^5\gamma}\right)\sigma_{g,1}^2 \\
        &\quad+ \frac{4\eta^2l_{g,1}^2}{\lambda^2\mu^3\gamma}\sum_{i=1}^{t-1}\left(1-\frac{\mu\gamma}{2}\right)^{t-1-i}\|\hu_i\|^2 + \frac{64\eta^4l_{g,1}^4L^2}{\lambda^4\mu^8\gamma^4}\sum_{i=1}^{t-1}\left(1-\frac{\mu\gamma}{2}\right)^{t-1-i}\alpha_i\|\hu_i\|^2
    \end{aligned}
\end{equation}
and 
\begin{equation} \label{eq:recur-2}
    \begin{aligned}
        \|y_t-y_t^*\|^2 
        % &\leq \left(\left(1-\frac{\mu\gamma}{2}\right)^{t-1} + \frac{8\eta^2l_{g,1}^2L^2}{\lambda^2\mu^4\gamma^2}\right) \|y_1-y_1^*\|^2 + \left(\frac{8\gamma}{\mu}\ln\frac{4eT}{\delta} + \frac{32\eta^2l_{g,1}^2L^2}{\lambda^2\mu^5\gamma}\right)\sigma_{g,1}^2 \\
        &\leq \left(1 + \frac{8\eta^2l_{g,1}^2L^2}{\lambda^2\mu^4\gamma^2}\right) \|y_1-y_1^*\|^2 + \left(\frac{8\gamma}{\mu}\ln\frac{4eT}{\delta} + \frac{32\eta^2l_{g,1}^2L^2}{\lambda^2\mu^5\gamma}\right)\sigma_{g,1}^2 \\
        &\quad+ \frac{8\eta^2l_{g,1}^2C_{u,0}^2}{\lambda^2\mu^4\gamma^2} + \frac{1024\eta^4l_{g,1}^4L^2C_{u,0}^2}{\lambda^4\mu^8\gamma^4}\left(2+\ln\frac{1}{\beta}\right)
        =: \varrho^2,
    \end{aligned}
\end{equation}
where constants $L$ and sequence $\{\hu_i\}$ are defined in \eqref{eq:L-def} and \eqref{eq:hu-def}, respectively.
\end{lemma}

\begin{proof}[Proof of \Cref{lm:recursion-two-type}]
By \Cref{lm:gradient-bound}, we know that $\hl_t\leq L$ and $\|\hu_t\|\leq C_{u,0}$ if $t<\tau$. Then under event $\gE_0\cap\gE_{y}$, \eqref{eq:recur-1} is obtained by replacing $\hl_i$ with $L$, and \eqref{eq:recur-2} is obtained by substituting both $\hl_i$ and $\|\hu_i\|$ with $L$ and $C_{u,0}$, respectively.
\end{proof}

%%%%%%%%%%%%%%%%%%%%%%%%%%%%%%%%%%%%%%%%%%%%%%%%%%%%%%%%%%%%%%%%%%%%%%%%%%%%%%%%%%%%%%%%%%%%%%%
With \cref{lm:recursion-two-type} in place, we now formally present the statement of \cref{lem:upperlevel} below.

\begin{lemma}[Restatement of \cref{lem:upperlevel}] \label{lm:momentum-bound}
Under event $\gE_0\cap\gE_{y}$ and the parameter choices in \cref{lm:general-recursion}, if $t<\tau$, we have
\begin{equation*}
    \begin{aligned}
        \|\hm_t\| \leq C_{u,0} + C_{u,1}\varrho,
        \quad
        \hv_t \preceq (C_{u,0} + C_{u,1}\varrho)^2,
        \quad
        \frac{\eta}{C_{u,0} + C_{u,1}\varrho + \lambda} \preceq h_t \preceq \frac{\eta}{\lambda};
    \end{aligned}
\end{equation*}
if $t\leq \tau$, we have
\begin{equation*}
    \begin{aligned}
        \|\hatphi(x_t,y_t;\Bar{\xi}_t) - \E_t[\hatphi(x_t,y_t;\Bar{\xi}_t)]\| \leq \sigma_{\phi}, 
    \end{aligned}
\end{equation*}
\begin{equation*}
    \begin{aligned}
        \|\E_t[\hatphi(x_t,y_t^*;\Bar{\xi}_t)] - \E_{t-1}[\hatphi(x_{t-1},y_{t-1}^*;\Bar{\xi}_{t-1})]\| \leq L\|x_t-x_{t-1}\|;
    \end{aligned}
\end{equation*}
where constants $C_{u,0}, C_{u,1}, \sigma_{\phi}, L$ and $\varrho$ are defined in \eqref{eq:Cu-def}, \eqref{eq:L-def} and \eqref{eq:recur-2}, respectively.
\end{lemma}

\begin{proof}[Proof of \Cref{lm:momentum-bound}]
By \Cref{lm:stoc-hyper-error}, under event $\gE_0\cap\gE_{y}$, if $t<\tau$, we have
\begin{equation*}
    \begin{aligned}
        \|\hatphi(x_t,y_t;\Bar{\xi}_t)\| 
        &\leq C_{\phi,0} + (C_{\phi,1} + L_1\|\gdphi(x_t)\|)\|y_t-y_t^*\| + \|\gdphi(x_t)\| \\
        &\leq C_{\phi,0} + G + (C_{\phi,1} + L_1G)\varrho
        = C_{u,0} + C_{u,1}\varrho,
    \end{aligned}
\end{equation*}
where the second inequality is due to \Cref{lm:gradient-bound} and \eqref{eq:recur-2} in \Cref{lm:recursion-two-type}, and the last equality uses the definitions in \eqref{eq:Cu-def}.
We can bound $\|\hm_t\|$ by a standard induction argument as follows. First, for the base case $k=1$, note that
\begin{equation*}
    \|\hm_1\|=\|\hatphi(x_1,y_1;\Bar{\xi}_1)\|\leq C_{u,0} + C_{u,1}\varrho.
\end{equation*}
Suppose $\|\hm_{k-1}\|\leq C_{u,0} + C_{u,1}\varrho$ for some $k<\tau$, then we have
\begin{equation*}
    \begin{aligned}
        \|\hm_k\| 
        &\leq (1-\alpha_k)\|\hm_{k-1}\| + \alpha_k\|\hatphi(x_k,y_k;\Bar{\xi}_k)\|
        \leq C_{u,0} + C_{u,1}\varrho.
    \end{aligned}
\end{equation*}
Then we can show $\hv_t\preceq (C_{u,0} + C_{u,1}\varrho)^2$ in a similar way (by induction argument) by noting that
\begin{equation*}
    (\hatphi(x_t,y_t;\Bar{\xi}_t))^2 \preceq \|\hatphi(x_t,y_t;\Bar{\xi}_t)\|^2 \leq (C_{u,0} + C_{u,1}\varrho)^2.
\end{equation*}
Given the bound on $\hv_t$, it is straight forward to bound the learning rate $h_t$.
As for the second last bound, by \Cref{lm:hyper-noise} and \eqref{eq:recur-2} of \Cref{lm:recursion-two-type},  under event $\gE_0\cap\gE_{y}$, if $t\leq \tau$, we have
\begin{equation*}
    \begin{aligned}
        \|\hatphi(x_t,y_t;\Bar{\xi}_t) &- \E_t[\hatphi(x_t,y_t;\Bar{\xi}_t)]\| \\
        &\leq \frac{\mu+3l_{g,1}+\sigma_{g,2}}{\mu} + \frac{2l_{g,1}+\sigma_{g,2}}{\mu}l_{f,0} + \frac{2l_{g,1}+\sigma_{g,2}}{\mu}(L_{y,0}+L_{y,1}l_{f,0})\|y_t-y_t^*\| \\
        &\leq \frac{\mu+3l_{g,1}+\sigma_{g,2}}{\mu} + \frac{2l_{g,1}+\sigma_{g,2}}{\mu}l_{f,0} + \frac{2l_{g,1}+\sigma_{g,2}}{\mu}(L_{y,0}+L_{y,1}l_{f,0})\varrho \\
        &\leq \sigma_{\phi},
    \end{aligned}
\end{equation*}
where the last equality uses $\varrho\leq r$ by \cref{lm:para-verify} and the definition of $\sigma_{\phi}$ in \eqref{eq:sigma-phi-def}. The last bound can be obtained by applying \Cref{lm:hyper-stoc-bias,lm:gradient-bound}:
\begin{equation*}
    \begin{aligned}
        \|\E_t[\hatphi(x_t,y_t^*;\Bar{\xi}_t)] - \E_{t-1}[\hatphi(x_{t-1},y_{t-1}^*;\Bar{\xi}_{t-1})]\| 
        &\leq (L_0 + L_1\|\gdphi(x_{t-1})\|)\|x_t-x_{t-1}\| \\
        &\leq (L_0 + L_1G)\|x_t-x_{t-1}\| \\
        &= L\|x_t-x_{t-1}\|,
    \end{aligned}
\end{equation*}
where the last inequality uses the definition of $L$ in \eqref{eq:L-def}.
\end{proof}

%%%%%%%%%%%%%%%%%%%%%%%%%%%%%%%%%%%%%%%%%%%%%%%%%%%%%%%%%%%%%%%%%%%%%%%%%%%%%%%%%%%%%%%%%%%%%%%
\subsection{Proof of \cref{lem:biasvirtual}}
\label{sec:proof-biasvirtual}
The following lemma provides a bound for the difference between the actual momentum $\hat{m}_t$ versus the auxiliary momentum $\hat{u}_t$ under the good event $\gE_0 \cap \gE_y$, which is crucial for establishing the convergence guarantees for \cref{alg:bi-adam}.

\begin{lemma}[Restatement of \cref{lem:biasvirtual}] \label{lm:hm-hu-diff}
Under event $\gE_0\cap\gE_{y}$ and the parameter choices in \cref{lm:general-recursion}, we have
\begin{equation*}
    \begin{aligned}
        &\sum_{t=1}^{\tau-1}\|\hm_t-\hu_t\|^2
        \leq TL^2\left(\left(1 + \frac{8\eta^2l_{g,1}^2L^2}{\lambda^2\mu^4\gamma^2}\right) \|y_1-y_1^*\|^2 + \left(\frac{8\gamma}{\mu}\ln\frac{4eT}{\delta} + \frac{32\eta^2l_{g,1}^2L^2}{\lambda^2\mu^5\gamma}\right)\sigma_{g,1}^2\right) \\
        &\quad+ L^2\left(\frac{8\eta^2l_{g,1}^2}{\lambda^2\mu^4\gamma^2} + \frac{2048\eta^4l_{g,1}^4L^2}{\lambda^4\mu^8\gamma^4}\left(2+\ln\frac{1}{\beta}\right)\right)\sum_{t=1}^{\tau-1}\|\epsilon_t\|^2 + 2\|\E_t[\hatphi(x_t,y_t^*;\Bar{\xi}_t)] - \gdphi(x_t)\|^2 \\
        &\quad+ 2L^2\left(\frac{8\eta^2l_{g,1}^2}{\lambda^2\mu^4\gamma^2} + \frac{2048\eta^4l_{g,1}^4L^2}{\lambda^4\mu^8\gamma^4}\left(2+\ln\frac{1}{\beta}\right)\right)\sum_{t=1}^{\tau-1}\|\gdphi(x_t)\|^2.
    \end{aligned}
\end{equation*}
\end{lemma}

\begin{proof}[Proof of \Cref{lm:hm-hu-diff}]
Under event $\gE_0\cap\gE_{y}$, if $t<\tau$, by \Cref{lm:gradient-bound} and \eqref{eq:hm-hu-diff} in \Cref{lm:general-recursion} we have
\begin{equation*}
    \begin{aligned}
        \|\hm_t-\hu_t\|^2
        \leq \hl_t^2\sum_{j=1}^{t}d_{t,j}\|y_j-y_j^*\|^2
        \leq L^2\sum_{j=1}^{t}d_{t,j}\|y_j-y_j^*\|^2.
    \end{aligned}
\end{equation*}
Now we apply \eqref{eq:recur-1} of \Cref{lm:recursion-two-type} and take summation to obtain
% \begin{equation*}
    \begin{align}
        &\sum_{t=1}^{\tau-1}\sum_{j=1}^{t}d_{t,j}\|y_j-y_j^*\|^2 \notag \\
        &\leq \sum_{t=1}^{\tau-1}\sum_{j=1}^{t}d_{t,j}\left(\left(\left(1-\frac{\mu\gamma}{2}\right)^{j-1} + \frac{8\eta^2l_{g,1}^2L^2}{\lambda^2\mu^4\gamma^2}\right) \|y_1-y_1^*\|^2 + \left(\frac{8\gamma}{\mu}\ln\frac{4eT}{\delta} + \frac{32\eta^2l_{g,1}^2L^2}{\lambda^2\mu^5\gamma}\right)\sigma_{g,1}^2\right) \label{eq:line-A1} \tag{$A_1$} \\
        &\quad+ \sum_{t=1}^{\tau-1}\sum_{j=1}^{t}d_{t,j}\left(\frac{4\eta^2l_{g,1}^2}{\lambda^2\mu^3\gamma}\sum_{i=1}^{j-1}\left(1-\frac{\mu\gamma}{2}\right)^{j-1-i}\|\hu_i\|^2 + \frac{64\eta^4l_{g,1}^4L^2}{\lambda^4\mu^8\gamma^4}\sum_{i=1}^{j-1}\left(1-\frac{\mu\gamma}{2}\right)^{j-1-i}\alpha_i\|\hu_i\|^2\right) \label{eq:line-A2} \tag{$A_2$}
    \end{align}
% \end{equation*}
We continue to bound each term individually.
\paragraph{Bounding \eqref{eq:line-A1}.}
By \Cref{lm:dj-def,lm:algebra-talphat} and choice of $\gamma=2\beta/\mu$, we have
\begin{equation} \label{eq:line-A1-bound}
    \begin{aligned}
        \eqref{eq:line-A1}
        &= \sum_{t=1}^{\tau-1}\sum_{j=1}^{t}d_{t,j}\left(\left(\left(1-\frac{\mu\gamma}{2}\right)^{j-1} + \frac{8\eta^2l_{g,1}^2L^2}{\lambda^2\mu^4\gamma^2}\right) \|y_1-y_1^*\|^2 + \left(\frac{8\gamma}{\mu}\ln\frac{4eT}{\delta} + \frac{32\eta^2l_{g,1}^2L^2}{\lambda^2\mu^5\gamma}\right)\sigma_{g,1}^2\right) \\
        &= \sum_{t=1}^{\tau-1}\sum_{j=1}^{t}d_{t,j}\left(1-\frac{\mu\gamma}{2}\right)^{j-1}\|y_1-y_1^*\|^2 \\
        &\quad+ \sum_{t=1}^{\tau-1}\sum_{j=1}^{t}d_{t,j}\left(\frac{8\eta^2l_{g,1}^2L^2}{\lambda^2\mu^4\gamma^2}\|y_1-y_1^*\|^2 + \left(\frac{8\gamma}{\mu}\ln\frac{4eT}{\delta} + \frac{32\eta^2l_{g,1}^2L^2}{\lambda^2\mu^5\gamma}\right)\sigma_{g,1}^2\right) \\
        &= \sum_{t=1}^{\tau-1}t\alpha_t(1-\beta)^{t-1}\|y_1-y_1^*\|^2 + \sum_{t=1}^{\tau-1}\left(\frac{8\eta^2l_{g,1}^2L^2}{\lambda^2\mu^4\gamma^2}\|y_1-y_1^*\|^2 + \left(\frac{8\gamma}{\mu}\ln\frac{4eT}{\delta} + \frac{32\eta^2l_{g,1}^2L^2}{\lambda^2\mu^5\gamma}\right)\sigma_{g,1}^2\right) \\
        &\leq T\left(\left(1 + \frac{8\eta^2l_{g,1}^2L^2}{\lambda^2\mu^4\gamma^2}\right) \|y_1-y_1^*\|^2 + \left(\frac{8\gamma}{\mu}\ln\frac{4eT}{\delta} + \frac{32\eta^2l_{g,1}^2L^2}{\lambda^2\mu^5\gamma}\right)\sigma_{g,1}^2\right),
    \end{aligned}
\end{equation}
where the last inequality uses $\tau\leq T+1$ by definition of $\tau$.
\paragraph{Bounding \eqref{eq:line-A2}.}
By \Cref{lm:dj-def,lm:algebra-sum-beta-alphai} and choice of $\gamma=2\beta/\mu$, we have
\begin{equation} \label{eq:line-A2-bound}
    \begin{aligned}
        \eqref{eq:line-A2}
        &= \sum_{t=1}^{\tau-1}\sum_{j=1}^{t}d_{t,j}\left(\frac{4\eta^2l_{g,1}^2}{\lambda^2\mu^3\gamma}\sum_{i=1}^{j-1}\left(1-\frac{\mu\gamma}{2}\right)^{j-1-i}\|\hu_i\|^2 + \frac{64\eta^4l_{g,1}^4L^2}{\lambda^4\mu^8\gamma^4}\sum_{i=1}^{j-1}\left(1-\frac{\mu\gamma}{2}\right)^{j-1-i}\alpha_i\|\hu_i\|^2\right) \\
        &\leq \frac{4\eta^2l_{g,1}^2}{\lambda^2\mu^4\gamma^2}\sum_{t=1}^{\tau-1}\sum_{j=1}^{t}d_{t,j}\|\hu_j\|^2 + \frac{64\eta^4l_{g,1}^4L^2}{\lambda^4\mu^8\gamma^4}\left(32+16\ln\frac{1}{\beta}\right)\sum_{t=1}^{\tau-1}\sum_{j=1}^{t}d_{t,j}\|\hu_j\|^2 \\
        &\leq \left(\frac{4\eta^2l_{g,1}^2}{\lambda^2\mu^4\gamma^2} + \frac{1024\eta^4l_{g,1}^4L^2}{\lambda^4\mu^8\gamma^4}\left(2+\ln\frac{1}{\beta}\right)\right)\sum_{t=1}^{\tau-1}\|\hu_t\|^2.
    \end{aligned}
\end{equation}
\paragraph{Final Bound.}
Combining \eqref{eq:line-A1-bound} and \eqref{eq:line-A2-bound} yields
\begin{equation*}
    \begin{aligned}
        \sum_{t=1}^{\tau-1}\sum_{j=1}^{t}d_{t,j}\|y_j-y_j^*\|^2
        &\leq T\left(\left(1 + \frac{8\eta^2l_{g,1}^2L^2}{\lambda^2\mu^4\gamma^2}\right) \|y_1-y_1^*\|^2 + \left(\frac{8\gamma}{\mu}\ln\frac{4eT}{\delta} + \frac{32\eta^2l_{g,1}^2L^2}{\lambda^2\mu^5\gamma}\right)\sigma_{g,1}^2\right) \\
        &\quad + \left(\frac{4\eta^2l_{g,1}^2}{\lambda^2\mu^4\gamma^2} + \frac{1024\eta^4l_{g,1}^4L^2}{\lambda^4\mu^8\gamma^4}\left(2+\ln\frac{1}{\beta}\right)\right)\sum_{t=1}^{\tau-1}\|\hu_t\|^2.
        % &\leq T\left(\left(1 + \frac{8\eta^2l_{g,1}^2L^2}{\lambda^2\mu^4\gamma^2}\right) \|y_1-y_1^*\|^2 + \left(\frac{8\gamma}{\mu}\ln\frac{4eT}{\delta} + \frac{32\eta^2l_{g,1}^2L^2}{\lambda^2\mu^5\gamma}\right)\sigma_{g,1}^2\right) \\
        % &\quad + \left(\frac{8\eta^2l_{g,1}^2}{\lambda^2\mu^4\gamma^2} + \frac{2048\eta^4l_{g,1}^4L^2}{\lambda^4\mu^8\gamma^4}\left(2+\ln\frac{1}{\beta}\right)\right)\sum_{t=1}^{\tau-1}\|\epsilon_t\|^2 + \|\E_t[\hatphi(x_t,y_t^*;\Bar{\xi}_t)] - \gdphi(x_t)\|^2
    \end{aligned}
\end{equation*}
In addition, recall the definition of $\hu_t$ and $\epsilon_t$ in \eqref{eq:hu-def} and \eqref{eq:epst-def}, by Young's inequality we have
\begin{equation*}
    \begin{aligned}
        \|\hu_t\|^2  
        &\leq 2\|\epsilon_t\|^2 + 4\|\E_t[\hatphi(x_t,y_t^*;\Bar{\xi}_t)] - \gdphi(x_t)\|^2 + 4\|\gdphi(x_t)\|^2.
    \end{aligned}
\end{equation*}
Therefore, we conclude that
\begin{equation*}
    \begin{aligned}
        &\sum_{t=1}^{\tau-1}\|\hm_t-\hu_t\|^2
        \leq L^2\sum_{t=1}^{\tau-1}\sum_{j=1}^{t}d_{t,j}\|y_j-y_j^*\|^2 \\
        &\leq TL^2\left(\left(1 + \frac{8\eta^2l_{g,1}^2L^2}{\lambda^2\mu^4\gamma^2}\right) \|y_1-y_1^*\|^2 + \left(\frac{8\gamma}{\mu}\ln\frac{4eT}{\delta} + \frac{32\eta^2l_{g,1}^2L^2}{\lambda^2\mu^5\gamma}\right)\sigma_{g,1}^2\right) \\
        &\quad+ L^2\left(\frac{8\eta^2l_{g,1}^2}{\lambda^2\mu^4\gamma^2} + \frac{2048\eta^4l_{g,1}^4L^2}{\lambda^4\mu^8\gamma^4}\left(2+\ln\frac{1}{\beta}\right)\right)\sum_{t=1}^{\tau-1}\|\epsilon_t\|^2 + 2\|\E_t[\hatphi(x_t,y_t^*;\Bar{\xi}_t)] - \gdphi(x_t)\|^2 \\
        &\quad+ 2L^2\left(\frac{8\eta^2l_{g,1}^2}{\lambda^2\mu^4\gamma^2} + \frac{2048\eta^4l_{g,1}^4L^2}{\lambda^4\mu^8\gamma^4}\left(2+\ln\frac{1}{\beta}\right)\right)\sum_{t=1}^{\tau-1}\|\gdphi(x_t)\|^2.
    \end{aligned}
\end{equation*}
\end{proof}

%%%%%%%%%%%%%%%%%%%%%%%%%%%%%%%%%%%%%%%%%%%%%%%%%%%%%%%%%%%%%%%%%%%%%%%%%%%%%%%%%%%%%%%%%%%%%%%
\subsection{Proof of \cref{thm:main}}
\label{sec:proof-thm}

The following lemma ensures that $x_{t+1}$ and $x_t$ remain close for sufficiently small $\eta$, allowing us to apply \cref{lm:descent-inequality} in \cref{lm:descent-lemma}.

\begin{lemma} \label{lm:rs-condition}
Under event $\gE_0\cap\gE_{y}$ and the parameter choices in \cref{lm:general-recursion}, if $t<\tau$, then we have $\|x_{t+1}-x_t\|\leq \eta D$ where $D\coloneqq 2G/\lambda$. 
\end{lemma}

\begin{proof}[Proof of \Cref{lm:rs-condition}]
Under event $\gE_0\cap\gE_{y}$, if $t<\tau$, then we have
\begin{equation*}
    \begin{aligned}
        \|x_{t+1}-x_t\| \leq \|H_t\|\|\hm_t\| \leq \frac{\eta}{\lambda}\|\hm_t\| \leq \frac{\eta(C_{u,0} + C_{u,1}\varrho)}{\lambda} \leq \frac{2\eta G}{\lambda} = \eta D,
    \end{aligned}
\end{equation*}
where the first inequality uses \eqref{eq:ht-odot}, the second inequality is due to \cref{lm:gradient-bound}, the third inequality uses \cref{lm:momentum-bound}, the fourth inequality is due to \cref{lm:para-verify}, and the last equality uses the definition of $D$.
\end{proof}

%%%%%%%%%%%%%%%%%%%%%%%%%%%%%%%%%%%%%%%%%%%%%%%%%%%%%%%%%%%%%%%%%%%%%%%%%%%%%%%%%%%%%%%%%%%%%%%
Next, we provide a descent lemma for AdamBO.

\begin{lemma} \label{lm:descent-lemma}
Under event $\gE_0\cap\gE_{y}$ and the parameter choices in \cref{lm:general-recursion}, if $t<\tau$, 
% choosing $G$ and $\eta$ as
% \begin{equation*}
%     G \geq \max\left\{4C_{\phi,0}, \frac{C_{\phi,1}}{L_1}, 4\lambda\right\}
%     \quad\quad\text{and}\quad\quad
%     \eta \leq \min\left\{\frac{r}{D}, \frac{\lambda}{6L}\right\},
% \end{equation*}
we have
\begin{equation} \label{eq:descent-bi-adam}
    \begin{aligned}
        \Phi(x_{t+1}) - \Phi(x_t) 
        &\leq -\frac{\eta}{4G}\|\gdphi(x_t)\|^2 + \frac{2\eta}{\lambda}\|\hm_t-\hu_t\|^2 \\
        &\quad+ \frac{4\eta}{\lambda}\|\epsilon_t\|^2 + \frac{4\eta}{\lambda}\|\E_t[\hatphi(x_t,y_t^*;\Bar{\xi}_t)]-\gdphi(x_t)\|^2.
    \end{aligned}
\end{equation}
\end{lemma}

\begin{proof}[Proof of \Cref{lm:descent-lemma}]
By \Cref{lm:momentum-bound,lm:para-verify} and choice of $G$, if $t<\tau$, we have
\begin{equation} \label{eq:Ht-range}
    \begin{aligned}
        \frac{\eta I}{2G}\preceq \frac{\eta}{C_{u,0} + C_{u,1}\varrho + \lambda} \preceq H_t \preceq \frac{\eta I}{\lambda}.
    \end{aligned}
\end{equation}
Since we choose $\eta\leq r/D$, then by \Cref{lm:rs-condition} we have $\|x_{t+1}-x_t\|\leq r$ if $t<\tau$. Define $\hat{\epsilon}_t$ and $\epsilon_t$ as
\begin{equation} \label{eq:epsilont-def}
    \hat{\epsilon}_t = \hm_t-\gdphi(x_t)
    \quad\quad\text{and}\quad\quad
    \epsilon_t = \hu_t-\E_t[\hatphi(x_t,y_t^*;\Bar{\xi}_t)].
\end{equation}
% $\hat{\epsilon}_t = \hm_t-\gdphi(x_t)$ and $\epsilon_t = m_t-\E_t[\hatphi(x_t,y_t^*;\Bar{\xi}_t)]$, 
For any $t< \tau$, we apply \Cref{lm:descent-inequality} to obtain that
\begin{equation*}
    \begin{aligned}
        \Phi(x_{t+1}) - \Phi(x_t)
        &\leq \langle \gdphi(x_t), x_{t+1}-x_t \rangle + \frac{L_0+L_1\|\gdphi(x_t)\|}{2}\|x_{t+1}-x_t\|^2 \\
        &\leq \langle \gdphi(x_t), x_{t+1}-x_t \rangle + \frac{L}{2}\|x_{t+1}-x_t\|^2 \\
        &= -\gdphi(x_t)^{\top}H_t\hm_t + \frac{L}{2}\hm_t^{\top}H_t^2\hm_t \\
        &\leq -\|\gdphi(x_t)\|_{H_t}^2 - \gdphi(x_t)^{\top}H_t\hat{\epsilon}_t + \frac{\eta L}{2\lambda}\|\hm_t\|_{H_t}^2 \\
        &\leq -\frac{2}{3}\|\gdphi(x_t)\|_{H_t}^2 + \frac{3}{4}\|\hat{\epsilon}_t\|_{H_t}^2 + \frac{\eta L}{\lambda}\left(\|\gdphi(x_t)\|_{H_t}^2 + \|\hat{\epsilon}_t\|_{H_t}^2\right) \\
        &\leq -\frac{1}{2}\|\gdphi(x_t)\|_{H_t}^2 + \|\hat{\epsilon}_t\|_{H_t}^2 \\
        &\leq -\frac{\eta}{4G}\|\gdphi(x_t)\|^2 + \frac{\eta}{\lambda}\|\hat{\epsilon}_t\|^2 \\
        &\leq -\frac{\eta}{4G}\|\gdphi(x_t)\|^2 + \frac{2\eta}{\lambda}\|\hm_t-\hu_t\|^2 + \frac{4\eta}{\lambda}\|\epsilon_t\|^2 + \frac{4\eta}{\lambda}\|\E_t[\hatphi(x_t,y_t^*;\Bar{\xi}_t)]-\gdphi(x_t)\|^2,
    \end{aligned}
\end{equation*}
where the second inequality is due to \cref{lm:gradient-bound} and definition of $L$ in \eqref{eq:L-def}; the third inequality uses \eqref{eq:epsilont-def} and \eqref{eq:Ht-range}; the fourth inequality is due to Young's inequality $a^{\top}Ab \leq \frac{1}{3}\|a\|_{A}^2 + \frac{3}{4}\|b\|_{A}^2$ and $\|a+b\|^2 \leq 2\|a\|_{A}^2 + 2\|b\|_{A}^2$ for any PSD matrix $A$; the fifth inequality uses the choice of $\eta\leq \lambda/6L$; the second last inequality is due to \eqref{eq:Ht-range}; and the last inequality uses \eqref{eq:epsilont-def} and Young's inequality.
\end{proof}

%%%%%%%%%%%%%%%%%%%%%%%%%%%%%%%%%%%%%%%%%%%%%%%%%%%%%%%%%%%%%%%%%%%%%%%%%%%%%%%%%%%%%%%%%%%%%%%
The following lemma is essential for bounding the sum of the error terms $\|\epsilon_t\|^2$ before time $\tau$. Since we introduce $\E_t[\hatphi(x_t,y_t^*;\Bar{\xi}_t)]$ as part of the definition of $\epsilon_t$ (see \eqref{eq:epsilont-def}), we can directly invoke \citep[Lemma C.10]{li2023convergence} to obtain the high probability bound.

\begin{lemma}[{\citep[Lemma C.10]{li2023convergence}}] \label{lm:martingale-sum}
Denote $w_t$ as 
\begin{equation*}
    w_{t-1} = (1-\alpha_t)(\epsilon_{t-1} + \E_{t-1}[\hatphi(x_{t-1},y_{t-1}^*;\Bar{\xi}_{t-1})] - \E_t[\hatphi(x_t,y_t^*;\Bar{\xi}_t)]).
\end{equation*}
% If $G$ and $\eta$ satisfy
% \begin{equation*}
%     G\geq 2\sigma_{\phi}
%     \quad\quad\text{and}\quad\quad
%     \eta\leq \min\left\{\frac{r}{D}, \frac{\beta\sigma_{\phi}}{DL}\right\},
% \end{equation*}
Under the parameter choices in \cref{thm:main-appendix},
for any given $\delta\in(0,1)$, the following holds with probability at least $1-\delta/4$ over the randomness in $\gF_{T+1}$ (we denote this event as $\gE_x$):
\begin{equation*}
    \begin{aligned}
        \sum_{t=2}^{\tau}\alpha_t\langle w_{t-1}, \hatphi(x_t,y_t^*;\Bar{\xi}_t) - \E_t[\hatphi(x_t,y_t^*;\Bar{\xi}_t)] \rangle \leq 5\sigma_{\phi}^2\sqrt{(1+\beta^2T)\ln(4/\delta)}.
    \end{aligned}
\end{equation*}
\end{lemma}

%%%%%%%%%%%%%%%%%%%%%%%%%%%%%%%%%%%%%%%%%%%%%%%%%%%%%%%%%%%%%%%%%%%%%%%%%%%%%%%%%%%%%%%%%%%%%%%
The next lemma bounds the sum of the error terms $\|\epsilon_t\|^2$ before time $\tau$.

\begin{lemma} \label{lm:momentum-error-sum}
Under event $\gE_0\cap\gE_{y}\cap\gE_x$ and the parameter choices in \cref{lm:general-recursion},
% if $G$ and $\eta$ satisfy
% \begin{equation*}
%     G\geq 2\sigma_{\phi}
%     \quad\quad\text{and}\quad\quad
%     \eta\leq \min\left\{\frac{r}{D}, \frac{\lambda^{3/2}\beta}{32L\sqrt{G}}, \frac{\beta\sigma_{\phi}}{DL}\right\},
% \end{equation*}
we have
\begin{equation} \label{eq:momentum-error-sum}
    \begin{aligned}
        \sum_{t=1}^{\tau-1}\|\epsilon_t\|^2 - \frac{\lambda}{128G}\|\gdphi(x_t)\|^2
        &\leq 8\sigma_{\phi}^2(1/\beta + \beta T) + 20\sigma_{\phi}^2\sqrt{(1/\beta^2+T)\ln(4/\delta)} \\
        &+ \frac{\lambda}{128G}\sum_{t=1}^{\tau-1}\|\hm_t-\hu_t\|^2 + \|\E_t[\hatphi(x_t,y_t^*;\Bar{\xi}_t)] - \gdphi(x_t)\|^2.
    \end{aligned}
\end{equation}
\end{lemma}

\begin{proof}[Proof of \Cref{lm:momentum-error-sum}]
We first denote $w_t$ as 
\begin{equation*}
    w_{t-1} = (1-\alpha_t)(\epsilon_{t-1} + \E_{t-1}[\hatphi(x_{t-1},y_{t-1}^*;\Bar{\xi}_{t-1})] - \E_t[\hatphi(x_t,y_t^*;\Bar{\xi}_t)]).
\end{equation*}
By definition of $\epsilon_t$ and the update rule \eqref{eq:adam-equiv}, we have
\begin{equation} \label{eq:eps-nu}
    \begin{aligned}
        \epsilon_t
        &= (1-\alpha_t)(\epsilon_{t-1} + \E_{t-1}[\hatphi(x_{t-1},y_{t-1}^*;\Bar{\xi}_{t-1})] - \E_t[\hatphi(x_t,y_t^*;\Bar{\xi}_t)]) \\
        &\quad+ \alpha_t(\hatphi(x_t,y_t^*;\Bar{\xi}_t) - \E_t[\hatphi(x_t,y_t^*;\Bar{\xi}_t)]) \\
        &= w_{t-1} + \alpha_t(\hatphi(x_t,y_t^*;\Bar{\xi}_t) - \E_t[\hatphi(x_t,y_t^*;\Bar{\xi}_t)]).
    \end{aligned}
\end{equation}
By choice of $\eta$ we have
\begin{equation*}
    \eta \leq \frac{c_2r\lambda}{G} \leq \frac{2c_2r}{D} \leq \frac{r}{D},
\end{equation*}
where in the last inequality we choose small enough $c_2$. By \Cref{lm:rs-condition} we have $\|x_t-x_{t-1}\|\leq r$ if $t\leq \tau$. Then for $2\leq t\leq \tau$, we apply \Cref{lm:momentum-bound} to obtain
\begin{equation} \label{eq:expectation-diff}
    \begin{aligned}
        &\|\E_{t-1}[\hatphi(x_{t-1},y_{t-1}^*;\Bar{\xi}_{t-1})] - \E_t[\hatphi(x_t,y_t^*;\Bar{\xi}_t)]\| \\
        &\quad\leq L\|x_{t-1}-x_t\|
        \leq \frac{\eta L}{\lambda}\|\hm_{t-1}\| 
        \leq \frac{\eta L}{\lambda}(\|\gdphi(x_{t-1})\| + \|\hat{\epsilon}_{t-1}\|) \\
        &\quad\leq \frac{\eta L}{\lambda}\left(\|\gdphi(x_{t-1})\| + \|\hm_{t-1}-\hu_{t-1}\| + \|\epsilon_{t-1}\| + \|\E_{t-1}[\hatphi(x_{t-1},y_{t-1}^*;\Bar{\xi}_{t-1})] - \gdphi(x_{t-1})\|\right),
    \end{aligned}
\end{equation}
where the third inequality uses \eqref{eq:epsilont-def}. Hence we have
\begin{equation} \label{eq:nu-t}
    \begin{aligned}
        \|w_{t-1}\|^2
        &= \|(1-\alpha_t)(\epsilon_{t-1} + \E_{t-1}[\hatphi(x_{t-1},y_{t-1}^*;\Bar{\xi}_{t-1})] - \E_t[\hatphi(x_t,y_t^*;\Bar{\xi}_t)])\|^2 \\
        &\leq (1-\alpha_t)^2(1+\alpha_t)\|\epsilon_{t-1}\|^2 \\
        &\quad+ (1-\alpha_t)^2\left(1+\frac{1}{\alpha_t}\right)\|\E_{t-1}[\hatphi(x_{t-1},y_{t-1}^*;\Bar{\xi}_{t-1})] - \E_t[\hatphi(x_t,y_t^*;\Bar{\xi}_t)])\|^2 \\
        &\leq (1-\alpha_t)\|\epsilon_{t-1}\|^2 + \frac{1}{\alpha_t}\|\E_{t-1}[\hatphi(x_{t-1},y_{t-1}^*;\Bar{\xi}_{t-1})] - \E_t[\hatphi(x_t,y_t^*;\Bar{\xi}_t)])\|^2 \\
        &\leq (1-\alpha_t)\|\epsilon_{t-1}\|^2 + \frac{4\eta^2L^2}{\lambda^2\beta}\left(\|\gdphi(x_{t-1})\|^2 + \|\epsilon_{t-1}\|^2\right) \\
        &\quad+ \frac{4\eta^2L^2}{\lambda^2\beta}\left(\|\hm_{t-1}-\hu_{t-1}\|^2 + \|\E_{t-1}[\hatphi(x_{t-1},y_{t-1}^*;\Bar{\xi}_{t-1})] - \gdphi(x_{t-1})\|^2\right) \\
        &\leq \left(1-\frac{\alpha_t}{2}\right)\|\epsilon_{t-1}\|^2 + \frac{\lambda\beta}{256G}\|\gdphi(x_{t-1})\|^2 \\
        &\quad+ \frac{\lambda\beta}{256G}\left(\|\hm_{t-1}-\hu_{t-1}\|^2 + \|\E_{t-1}[\hatphi(x_{t-1},y_{t-1}^*;\Bar{\xi}_{t-1})] - \gdphi(x_{t-1})\|^2\right),
    \end{aligned}
\end{equation}
where the first inequality uses Young's inequality $\|a+b\|^2 \leq (1+c)\|a\|^2 + (1+1/c)\|b\|^2$ for any $c>0$; the second inequality is due to 
\begin{equation*}
    \begin{aligned}
        &(1-\alpha_t)^2(1+\alpha_t) \leq (1-\alpha_t)(1-\alpha_t^2) \leq 1-\alpha_t, \\
        &(1-\alpha_t)^2\left(1+\frac{1}{\alpha_t}\right) = \frac{1}{\alpha_t}(1-\alpha_t)^2(1+\alpha_t) \leq \frac{1}{\alpha}(1-\alpha_t) \leq \frac{1}{\alpha_t};
    \end{aligned}
\end{equation*}
the third inequality uses \eqref{eq:expectation-diff} and Young's inequality; and the last inequality is due to the choice of $G$ and $\eta$ with small enough $c_2$:
\begin{equation*}
    \eta \leq \frac{c_2\lambda^{3/2}\beta}{L\sqrt{G}} \leq \frac{\lambda^{3/2}\beta}{32L\sqrt{G}}
    \quad\Longrightarrow\quad
    \frac{4\eta^2L^2}{\lambda^2\beta} \leq \frac{\lambda\beta}{256G} \leq \frac{\beta}{256} \leq \frac{\beta}{2} \leq \frac{\alpha_t}{2}.
\end{equation*}
Plugging \eqref{eq:nu-t} back into \eqref{eq:eps-nu} gives
\begin{equation*}
    \begin{aligned}
        \|\epsilon_t\|^2
        &= \|w_{t-1}\|^2 + 2\alpha_t\langle w_{t-1}, \hatphi(x_t,y_t^*;\Bar{\xi}_t) - \E_t[\hatphi(x_t,y_t^*;\Bar{\xi}_t)] \rangle \\
        &\quad+ \alpha_t^2\|\hatphi(x_t,y_t^*;\Bar{\xi}_t) - \E_t[\hatphi(x_t,y_t^*;\Bar{\xi}_t)]\|^2 \\
        &\leq \left(1-\frac{\alpha_t}{2}\right)\|\epsilon_{t-1}\|^2 + \frac{\lambda\beta}{256G}\|\gdphi(x_{t-1})\|^2 + \alpha_t^2\sigma_{\phi}^2 \\
        &\quad+ 2\alpha_t\langle \nu_{t-1}, \hatphi(x_t,y_t^*;\Bar{\xi}_t) - \E_t[\hatphi(x_t,y_t^*;\Bar{\xi}_t)] \rangle \\
        &\quad+ \frac{\lambda\beta}{256G}\left(\|\hm_{t-1}-\hu_{t-1}\|^2 + \|\E_{t-1}[\hatphi(x_{t-1},y_{t-1}^*;\Bar{\xi}_{t-1})] - \gdphi(x_{t-1})\|^2\right).
    \end{aligned}
\end{equation*}
Rearranging the above inequality, for any $2\leq t\leq \tau$, we have 
\begin{equation*}
    \begin{aligned}
        \frac{\beta}{2}\|\epsilon_{t-1}\|^2
        \leq \frac{\alpha_t}{2}\|\epsilon_{t-1}\|^2
        &\leq \|\epsilon_{t-1}\|^2 - \|\epsilon_t\|^2 + \frac{\lambda\beta}{256G}\|\gdphi(x_{t-1})\|^2 \\
        &\quad+ \alpha_t^2\sigma_{\phi}^2 + 2\alpha_t\langle \nu_{t-1}, \hatphi(x_t,y_t^*;\Bar{\xi}_t) - \E_t[\hatphi(x_t,y_t^*;\Bar{\xi}_t)] \rangle \\
        &\quad+ \frac{\lambda\beta}{256G}\left(\|\hm_{t-1}-\hu_{t-1}\|^2 + \|\E_{t-1}[\hatphi(x_{t-1},y_{t-1}^*;\Bar{\xi}_{t-1})] - \gdphi(x_{t-1})\|^2\right).
    \end{aligned}
\end{equation*}
Then taking summation over $t$ from $2$ to $\tau$ we obtain that
\begin{equation*}
    \begin{aligned}
        \sum_{t=2}^{\tau}&\frac{\beta}{2}\|\epsilon_{t-1}\|^2 - \frac{\lambda\beta}{256G}\|\gdphi(x_{t-1})\|^2 \\
        &\leq \|\epsilon_1\|^2 - \|\epsilon_{\tau}\|^2 + \sigma_{\phi}^2\sum_{t=2}^{\tau}\alpha_t^2 + 2\sum_{t=2}^{\tau}\alpha_t\langle w_{t-1}, \hatphi(x_t,y_t^*;\Bar{\xi}_t) - \E_t[\hatphi(x_t,y_t^*;\Bar{\xi}_t)] \rangle \\
        &\quad+ \frac{\lambda\beta}{256G}\sum_{t=2}^{\tau}\|\hm_{t-1}-\hu_{t-1}\|^2 + \|\E_{t-1}[\hatphi(x_{t-1},y_{t-1}^*;\Bar{\xi}_{t-1})] - \gdphi(x_{t-1})\|^2 \\
        &\leq 4\sigma_{\phi}^2(1+\beta^2T) + 10\sigma_{\phi}^2\sqrt{(1+\beta^2T)\ln(4/\delta)} \\
        &\quad+ \frac{\lambda\beta}{256G}\sum_{t=2}^{\tau}\|\hm_{t-1}-\hu_{t-1}\|^2 + \|\E_{t-1}[\hatphi(x_{t-1},y_{t-1}^*;\Bar{\xi}_{t-1})] - \gdphi(x_{t-1})\|^2,
    \end{aligned}
\end{equation*}
where the last inequality uses \Cref{lm:algebra-sum-alphasqt,lm:martingale-sum} and the fact that $\|\epsilon_1\|^2\leq \sigma_{\phi}^2$. Then we complete the proof by multiplying both sides by $2/\beta$.
\end{proof}

%%%%%%%%%%%%%%%%%%%%%%%%%%%%%%%%%%%%%%%%%%%%%%%%%%%%%%%%%%%%%%%%%%%%%%%%%%%%%%%%%%%%%%%%%%%%%%%

%%%%%%%%%%%%%%%%%%%%%%%%%%%%%%%%%%%%%%%%%%%%%%%%%%%%%%%%%%%%%%%%%%%%%%%%%%%%%%%%%%%%%%%%%%%%%%%
With \cref{lm:descent-lemma,lm:momentum-error-sum}, we are ready to prove \cref{thm:main}. Below is the full statement of \cref{thm:main} with detailed parameter choices, where we use $c_1,c_2,c_3$ to denote small enough constants and $C_1,C_2$ to denote large enough ones. The definitions of problem-dependent constants $\sigma_{\phi}, C_{\phi,0}, C_{\phi,1}, \Delta_1, L_0, L_1, L, C_{\beta}$ are provided in \cref{sec:adambo-notations}.

\begin{theorem}[Restatement of \cref{thm:main}] \label{thm:main-appendix}
Suppose that \cref{ass:bilevel-assumption,ass:noise,ass:bilevel-additional} hold. Let $G$ be a constant satisfying 
\begin{equation} \label{eq:G-def}
    \begin{aligned}
        G \geq \max\left\{4\lambda, 2\sigma_{\phi}, 4C_{\phi,0}, \frac{C_{\phi,1}}{L_1}, \sqrt{\frac{C_1\Delta_1L_0}{C_L}}, \frac{C_1\Delta_1L_1}{C_L}\right\},
    \end{aligned}
\end{equation}
Given any $\epsilon>0$ and $\delta\in(0,1)$, denote $\iota\coloneqq \ln(4/\delta)$, and choose
\begin{equation} \label{eq:beta-def}
    0\leq \betasq \leq 1,
    \quad
    \beta \leq \min\left\{1, \frac{c_1\lambda\epsilon^2}{\sigma_{\phi}^2G\max\{1, \sqrt{\iota}, \ln(C_{\beta})\}}\right\},
    \quad
    \gamma = \frac{2\beta}{\mu},
\end{equation}
\begin{equation} \label{eq:eta-def}
    \begin{aligned}
        \eta \leq c_2\min\left\{\frac{r\lambda}{G}, \frac{\lambda}{6L}, \frac{\sigma_{\phi}\lambda\beta}{LG\max\{1, \sqrt{\iota}, \ln(1/\beta), \ln(C_{\beta})\}}, \frac{\lambda^{3/2}\beta}{L\sqrt{G}}\right\}, 
    \end{aligned}
\end{equation}
\begin{equation} \label{eq:Q-def}
    Q \geq \frac{1}{2}\max\left\{\ln\beta \Big/ \ln\left(1-\frac{\mu}{l_{g,1}}\right), \ln\left(\frac{c_3\lambda\mu^2\epsilon^2}{Gl_{g,1}^2l_{f,0}^2}\right) \Big/ \ln\left(1-\frac{\mu}{l_{g,1}}\right)\right\},
\end{equation}
\begin{equation} \label{eq:T0-def}
    T_0 = \ln\left(\frac{\sigma_{g,1}^2\beta}{\mu^2\|y_0-y_0^*\|^2}\right) \Big/ \ln(1-\beta),
    \quad
    T = \max\left\{\frac{1}{\beta^2}, \frac{C_2\Delta_1G}{\eta\epsilon^2}\right\},
\end{equation}
where constant $C_{\beta}$ is defined as
\begin{equation*}
    \begin{aligned}
        C_{\beta} \geq \max&\left\{\frac{8e\sigma_{\phi}^4G^2\max\{1,\iota\}}{c_1^2\delta\lambda^2\epsilon^4}, \frac{8C_2e\Delta_1L\sigma_{\phi}G^3}{c_1c_2\delta\lambda^2\epsilon^4}\left(1 + \frac{\sigma_{\phi}^2G}{c_1\lambda\epsilon^2}\right)\max\{1,\sqrt{\iota},\iota\}, \right. \\
        &\left. \quad\quad \left(\frac{32e\sigma_{\phi}^4G^2}{c_1^2\delta\lambda^2\epsilon^4}\right)^2, \left(\frac{48C_2e\Delta_1L\sigma_{\phi}G^3}{c_1c_2\delta\lambda^2\epsilon^4}\left(1 + \frac{\sigma_{\phi}^2G}{c_1\lambda\epsilon^2}\right)\max\{1,\sqrt{\iota},\iota\}\right)^2\right\}.
    \end{aligned}
\end{equation*}
% where the definitions of problem-dependent constants are listed in \cref{sec:adambo-notations}. 
Run \cref{alg:bi-adam} for $T$ iterations. Then with probability at least $1-\delta$ over the randomness in $\gF_{T+1}$, we have $\|\gdphi(x_t)\|\leq G$ for all $t\in[T]$, and $\frac{1}{T}\sum_{t=1}^{T}\|\gdphi(x_t)\|\leq \epsilon^2$.
\end{theorem}

\begin{proof}[Proof of \cref{thm:main-appendix}]
By \cref{lm:warm-start,lm:general-recursion,lm:martingale-sum}, we have $\pr(\gE_0 \cap \gE_y \cap \gE_x) \geq 1-3\delta/4 \geq 1-\delta$.
% \begin{equation*}
%     \pr(\gE_0 \cap \gE_y \cap \gE_x) \geq 1-3\delta/4 \geq 1-\delta.
% \end{equation*}
The following analysis is conditioned on the event $\gE_0 \cap \gE_y \cap \gE_x$. 

Rearranging \eqref{eq:descent-bi-adam} of \Cref{lm:descent-lemma} and telescoping over $t$ from 1 to $\tau-1$, we have
\begin{equation} \label{eq:part-1}
    \begin{aligned}
        \sum_{t=1}^{\tau-1}4\|\gdphi(x_t)\|^2 - \frac{64G}{\lambda}\|&\epsilon_t\|^2
        \leq \frac{16G}{\eta}[(\Phi(x_1)-\Phi^*) - (\Phi(x_{\tau})-\Phi^*)] \\
        &+ \frac{32G}{\lambda}\sum_{t=1}^{\tau-1}\|\hm_t-\hu_t\|^2 + 2\|\E_t[\hatphi(x_t,y_t^*;\Bar{\xi}_t)]-\gdphi(x_t)\|^2.
    \end{aligned}
\end{equation}
Also, \eqref{eq:momentum-error-sum} of \Cref{lm:momentum-error-sum} can be written as
\begin{equation} \label{eq:part-2}
    \begin{aligned}
        \sum_{t=1}^{\tau-1}\frac{128G}{\lambda}\|\epsilon_t\|^2 - \|\gdphi(x_t)\|^2
        &\leq \frac{128G}{\lambda}\left(8\sigma_{\phi}^2(1/\beta + \beta T) + 20\sigma_{\phi}^2\sqrt{(1/\beta^2+T)\ln(4/\delta)}\right) \\
        &\quad+ \sum_{t=1}^{\tau-1}\|\hm_t-\hu_t\|^2 + \|\E_t[\hatphi(x_t,y_t^*;\Bar{\xi}_t)] - \gdphi(x_t)\|^2.
    \end{aligned}
\end{equation}
Summing \eqref{eq:part-1} and $\eqref{eq:part-2}$ and rearranging gives
\begin{equation*}
    \begin{aligned}
        &\frac{16G}{\eta}(\Phi(x_{\tau})-\Phi^*) + 3\sum_{t=1}^{\tau-1}\|\gdphi(x_t)\|^2 + \frac{64G}{\lambda}\sum_{t=1}^{\tau-1}\|\epsilon_t\|^2 \\
        &\quad\leq \frac{16G}{\eta\lambda}\left(\lambda\Delta_1 + 64\sigma_{\phi}^2\left(\frac{\eta}{\beta} + \eta\beta T\right) + 160\eta\sigma_{\phi}^2\sqrt{(1/\beta^2+T)\ln(4/\delta)}\right) \\
        &\quad\quad+ \left(1 + \frac{32G}{\lambda}\right)\sum_{t=1}^{\tau-1}\|\hm_t-\hu_t\|^2 + \left(1 + \frac{64G}{\lambda}\right)\sum_{t=1}^{\tau-1}\|\E_t[\hatphi(x_t,y_t^*;\Bar{\xi}_t)] - \gdphi(x_t)\|^2 \\
        &\quad\leq \frac{16G}{\eta\lambda}\left(\lambda\Delta_1 + 64\sigma_{\phi}^2\left(\frac{\eta}{\beta} + \eta\beta T\right) + 160\eta\sigma_{\phi}^2\sqrt{(1/\beta^2+T)\ln(4/\delta)}\right) \\
        &\quad\quad+ \frac{33G}{\lambda}\sum_{t=1}^{\tau-1}\|\hm_t-\hu_t\|^2 + \frac{65G}{\lambda}\sum_{t=1}^{\tau-1}\|\E_t[\hatphi(x_t,y_t^*;\Bar{\xi}_t)] - \gdphi(x_t)\|^2,
    \end{aligned}
\end{equation*}
where the last inequality uses $G\geq \lambda$. By \cref{lm:hm-hu-diff}, we further have
\begin{equation} \label{eq:contra-inter}
    \begin{aligned}
        &\frac{16G}{\eta}(\Phi(x_{\tau})-\Phi^*) + 3\sum_{t=1}^{\tau-1}\|\gdphi(x_t)\|^2 + \frac{64G}{\lambda}\sum_{t=1}^{\tau-1}\|\epsilon_t\|^2 \\
        &\leq \frac{16G}{\eta\lambda}\left(\lambda\Delta_1 + 64\sigma_{\phi}^2\left(\frac{\eta}{\beta} + \eta\beta T\right) + 160\eta\sigma_{\phi}^2\sqrt{(1/\beta^2+T)\ln(4/\delta)}\right) \\
        &\quad+ \frac{33L^2GT}{\lambda}\left(\left(1 + \frac{8\eta^2l_{g,1}^2L^2}{\lambda^2\mu^4\gamma^2}\right) \|y_1-y_1^*\|^2 + \left(\frac{8\gamma}{\mu}\ln\frac{4eT}{\delta} + \frac{32\eta^2l_{g,1}^2L^2}{\lambda^2\mu^5\gamma}\right)\sigma_{g,1}^2\right) \\
        &\quad+ \frac{33L^2G}{\lambda}\left(\frac{8\eta^2l_{g,1}^2}{\lambda^2\mu^4\gamma^2} + \frac{2048\eta^4l_{g,1}^4L^2}{\lambda^4\mu^8\gamma^4}\left(2+\ln\frac{1}{\beta}\right)\right)\sum_{t=1}^{\tau-1}\|\epsilon_t\|^2 + 2\|\E_t[\hatphi(x_t,y_t^*;\Bar{\xi}_t)] - \gdphi(x_t)\|^2 \\
        &\quad+ \frac{66L^2G}{\lambda}\left(\frac{8\eta^2l_{g,1}^2}{\lambda^2\mu^4\gamma^2} + \frac{2048\eta^4l_{g,1}^4L^2}{\lambda^4\mu^8\gamma^4}\left(2+\ln\frac{1}{\beta}\right)\right)\sum_{t=1}^{\tau-1}\|\gdphi(x_t)\|^2 \\
        &\quad+ \frac{65G}{\lambda}\sum_{t=1}^{\tau-1}\|\E_t[\hatphi(x_t,y_t^*;\Bar{\xi}_t)] - \gdphi(x_t)\|^2.
    \end{aligned}
\end{equation}
% Now we give bounds for the coefficients on right-hand side of the above inequality. According to the parameter choices for $\gamma$ and $\eta$, we have
By \cref{lm:para-verify}, we know that
\begin{equation*}
    \begin{aligned}
        \frac{66L^2G}{\lambda}\left(\frac{8\eta^2l_{g,1}^2}{\lambda^2\mu^4\gamma^2} + \frac{2048\eta^4l_{g,1}^4L^2}{\lambda^4\mu^8\gamma^4}\left(2+\ln\frac{1}{\beta}\right)\right) 
        % &\leq \frac{66L^2G}{\lambda}\left(\frac{2\eta^2l_{g,1}^2}{\lambda^2\mu^2\beta^2} + \frac{128\eta^4l_{g,1}^4L^2}{\lambda^4\mu^4\beta^4}\left(2+\ln\frac{1}{\beta}\right)\right) 
        % \leq \frac{66L^2G}{\lambda}\left(\frac{2c_2^2\sigma_{\phi}^2l_{g,1}^2}{\mu^2L^2G^2} + \frac{384c_2^4\sigma_{\phi}^4l_{g,1}^4}{\mu^4L^2G^4}\right) 
        \leq 2,
    \end{aligned}
\end{equation*}
% where in the last inequality we choose small enough $c_2$.
Then with $G\geq \lambda$, \eqref{eq:contra-inter} can be simplified as 
\begin{equation} \label{eq:contradiction}
    \begin{aligned}
        &\frac{16G}{\eta}(\Phi(x_{\tau})-\Phi^*) + \sum_{t=1}^{\tau-1}\|\gdphi(x_t)\|^2 \\
        &\quad\leq \frac{16G}{\eta\lambda}\left(\lambda\Delta_1 + 64\sigma_{\phi}^2\left(\frac{\eta}{\beta} + \eta\beta T\right) + 160\eta\sigma_{\phi}^2\sqrt{(1/\beta^2+T)\ln(4/\delta)}\right) \\
        &\quad\quad+ \frac{33L^2GT}{\lambda}\left(\left(1 + \frac{8\eta^2l_{g,1}^2L^2}{\lambda^2\mu^4\gamma^2}\right) \|y_1-y_1^*\|^2 + \left(\frac{8\gamma}{\mu}\ln\frac{4eT}{\delta} + \frac{32\eta^2l_{g,1}^2L^2}{\lambda^2\mu^5\gamma}\right)\sigma_{g,1}^2\right) \\
        &\quad\quad+ \frac{67GTl_{g,1}^2l_{f,0}^2}{\lambda\mu^2}\left(1-\frac{\mu}{l_{g,1}}\right)^{2Q} =: I_1.
    \end{aligned}
\end{equation}
By definition of $\tau$ in \eqref{eq:tau-def}, we have
\begin{equation*}
    \frac{16G}{\eta}(\Phi(x_{\tau})-\Phi^*) > \frac{16G\psi}{\eta} = \frac{8C_LG^3}{\eta L} =: I_2.
\end{equation*}
By \cref{lm:para-choice}, we have $I_1\leq I_2$, which leads to a contradiction. Thus, we must have $\tau=T+1$ conditioned on $\gE_0\cap\gE_y\cap\gE_x$. Therefore, combining \eqref{eq:contradiction} and \cref{lm:para-choice} finally yields that under event $\gE_0\cap\gE_y\cap\gE_x$,
\begin{equation*}
    \frac{1}{T}\sum_{t=1}^{T}\|\gdphi(x_t)\|^2 \leq \epsilon^2.
\end{equation*}
Moreover, since $\tau=T+1$, then by \cref{lm:gradient-bound} we can replace $\hl_T$ and $G_T$ with 
$L$ and $G$ respectively, in the additional requirement \eqref{eq:eta-additional} for $\eta$. Therefore, \eqref{eq:eta-additional} is now included in the parameter choices of \cref{thm:main-appendix}, which indicates that the current parameter choices are sufficient.
\end{proof}

%%%%%%%%%%%%%%%%%%%%%%%%%%%%%%%%%%%%%%%%%%%%%%%%%%%%%%%%%%%%%%%%%%%%%%%%%%%%%%%%%%%%%%%%%%%%%%%
\subsection{Parameter Choices for AdamBO (\cref{thm:main-appendix})}
The following two lemmas, \cref{lm:para-verify,lm:para-choice}, hide complicate calculations and will be useful in the contradiction argument and upper-level convergence analysis.

\begin{lemma} \label{lm:para-verify}
Under the parameter choices in \cref{thm:main-appendix}, we have the following facts:
\begin{equation} \label{eq:ln-Cbeta-verify}
    \ln\frac{4eT}{\delta} \leq \ln(C_{\beta}),
    \quad
    \|y_1-y_1^*\|^2 \leq \frac{17\beta\sigma_{g,1}^2}{\mu^2}\ln(C_{\beta}),
\end{equation}
\begin{equation} \label{eq:rho-less-r-verify}
    \varrho \leq \min\left\{r, \frac{1}{4L_1}\right\},
    \quad
    C_{u,0} + C_{u,1}\varrho + \lambda \leq 2G,
\end{equation}
\begin{equation} \label{eq:66-verify}
    \frac{66L^2G}{\lambda}\left(\frac{8\eta^2l_{g,1}^2}{\lambda^2\mu^4\gamma^2} + \frac{2048\eta^4l_{g,1}^4L^2}{\lambda^4\mu^8\gamma^4}\left(2+\ln\frac{1}{\beta}\right)\right) \leq 2
\end{equation}
\end{lemma}

\begin{proof}[Proof of \cref{lm:para-verify}]
We first list all the relevant parameter choices below for convenience:
\begin{equation*}
    \begin{aligned}
        G \geq \max\left\{4\lambda, 2\sigma_{\phi}, 4C_{\phi,0}, \frac{C_{\phi,1}}{L_1}, \sqrt{\frac{C_1\Delta_1L_0}{C_L}}, \frac{C_1\Delta_1L_1}{C_L}\right\},
    \end{aligned}
\end{equation*}
\begin{equation*}
    \beta \leq \min\left\{1, \frac{c_1\lambda\epsilon^2}{\sigma_{\phi}^2G\max\{1, \sqrt{\iota}, \ln(C_{\beta})\}}\right\},
    \quad
    \gamma = \frac{2\beta}{\mu},
\end{equation*}
\begin{equation*}
    \begin{aligned}
        \eta \leq c_2\min\left\{\frac{r\lambda}{G}, \frac{\sigma_{\phi}\lambda\beta}{LG\max\{1, \sqrt{\iota}, \ln(1/\beta), \ln(C_{\beta})\}}, \frac{\lambda^{3/2}\beta}{L\sqrt{G}}\right\}, 
        \quad
        % T = \max\left\{\frac{1}{\beta^2}, \frac{C_2\Delta_1G}{\eta\epsilon^2}\right\}.
        % T_0 = \ln\left(\frac{16\sigma_{g,1}^2\beta}{\mu^2\|y_0-y_0^*\|^2}\right) \Big/ \ln(1-\beta),
    \end{aligned}
\end{equation*}
\begin{equation*}
    Q \geq \frac{1}{2}\max\left\{\ln\beta \Big/ \ln\left(1-\frac{\mu}{l_{g,1}}\right), \ln\left(\frac{c_3\lambda\mu^2\epsilon^2}{Gl_{g,1}^2l_{f,0}^2}\right) \Big/ \ln\left(1-\frac{\mu}{l_{g,1}}\right)\right\},
\end{equation*}
\begin{equation*}
    T_0 = \ln\left(\frac{\sigma_{g,1}^2\beta}{\mu^2\|y_0-y_0^*\|^2}\right) \Big/ \ln(1-\beta),
    \quad
    T = \max\left\{\frac{1}{\beta^2}, \frac{C_2\Delta_1G}{\eta\epsilon^2}\right\},
\end{equation*}
where $C_{\beta}$ is defined in \eqref{eq:Cbeta-def}. Now we verify the above listed facts one by one.

\paragraph{Verification for \eqref{eq:ln-Cbeta-verify}: $\ln(4eT/\delta)\leq \ln(C_\beta)$.}
We focus on the dominant terms for each parameter choice when $\epsilon$ is sufficiently small. For the remaining cases, the result can be easily obtained by following the same procedure. Specifically, we consider the case where $\beta, \eta$ and $T$ are chosen as
\begin{equation*}
    \beta = \frac{c_1\lambda\epsilon^2}{\sigma_{\phi}^2G\max\{1, \sqrt{\iota}, \ln(C_{\beta})\}}, 
    \quad
    \eta = \frac{c_2\sigma_{\phi}\lambda\beta}{LG\max\{1, \sqrt{\iota}, \ln(1/\beta), \ln(C_{\beta})\}},
    \quad
    T = \max\left\{\frac{1}{\beta^2}, \frac{C_2\Delta_1G}{\eta\epsilon^2}\right\}.
\end{equation*}
(Case 1) If $T=1/\beta^2$, then we have
\begin{equation*}
    \begin{aligned}
        \ln\frac{4eT}{\delta}
        &= \ln\frac{4e}{\delta\beta^2} \\
        &= \ln\left(\frac{4e\sigma_{\phi}^4G^2\max\{1,\iota,\ln^2(C_{\beta})\}}{c_1^2\delta\lambda^2\epsilon^4}\right) 
        \leq \ln\left(\frac{4e\sigma_{\phi}^4G^2(\max\{1,\iota\}+\ln^2(C_{\beta}))}{c_1^2\delta\lambda^2\epsilon^4}\right) \\
        &\leq \ln\left(\frac{4e\sigma_{\phi}^4G^2(\max\{1,\iota\}+4C_{\beta}^{1/2})}{c_1^2\delta\lambda^2\epsilon^4}\right) 
        \leq \ln(C_{\beta}),
    \end{aligned}
\end{equation*}
where the second inequality uses $\ln x\leq 2x^{1/4}$ for $x>0$, and the last inequality is due to 
\begin{equation*}
    \begin{aligned}
        \frac{4e\sigma_{\phi}^4G^2\max\{1,\iota\}}{c_1^2\delta\lambda^2\epsilon^4} \leq \frac{C_{\beta}}{2}
        \quad\quad\text{and}\quad\quad
        \frac{16e\sigma_{\phi}^4G^2}{c_1^2\delta\lambda^2\epsilon^4}C_{\beta}^{1/2} \leq \frac{C_{\beta}}{2}
    \end{aligned}
\end{equation*}
since
\begin{equation*}
    \begin{aligned}
        C_{\beta} \geq \max\left\{\frac{8e\sigma_{\phi}^4G^2\max\{1,\iota\}}{c_1^2\delta\lambda^2\epsilon^4}, \left(\frac{32e\sigma_{\phi}^4G^2}{c_1^2\delta\lambda^2\epsilon^4}\right)^2\right\}.
    \end{aligned}
\end{equation*}
(Case 2) If $T=\frac{C_2\Delta_1G}{\eta\epsilon^2}$, then we have
\begin{equation} \label{eq:ln-eT-delta}
    \begin{aligned}
        \ln\frac{4eT}{\delta}
        &= \ln\left(\frac{4C_2e\Delta_1L\sigma_{\phi}G^3\max\{1,\sqrt{\iota},\ln(1/\beta), \ln(C_{\beta})\}\max\{1,\sqrt{\iota},\ln(C_{\beta})\}}{c_1c_2\delta\lambda^2\epsilon^4}\right) \\
        &= \ln\left(\frac{4C_2e\Delta_1L\sigma_{\phi}G^3(\max\{1,\sqrt{\iota}\}+\ln(1/\beta)+\ln(C_{\beta}))(\max\{1,\sqrt{\iota}\}+\ln(C_{\beta}))}{c_1c_2\delta\lambda^2\epsilon^4}\right). 
    \end{aligned}
\end{equation}
Also note that 
\begin{equation*}
    \begin{aligned}
        \ln\frac{1}{\beta}
        &= \ln\left(\frac{\sigma_{\phi}^2G\max\{1, \sqrt{\iota}, \ln(C_{\beta})\}}{c_1\lambda\epsilon^2}\right) 
        \leq \ln\left(\frac{\sigma_{\phi}^2G(\max\{1, \sqrt{\iota}\} + \ln(C_{\beta}))}{c_1\lambda\epsilon^2}\right) \\
        &\leq \frac{\sigma_{\phi}^2G(\max\{1, \sqrt{\iota}\} + \ln(C_{\beta}))}{c_1\lambda\epsilon^2}.
    \end{aligned}
\end{equation*}
Then we obtain
\begin{equation} \label{eq:max-times-max}
    \begin{aligned}
        &(\max\{1,\sqrt{\iota}\}+\ln(1/\beta)+\ln(C_{\beta}))(\max\{1,\sqrt{\iota}\}+\ln(C_{\beta})) \\
        &\leq \left(\max\{1,\sqrt{\iota}\} + \frac{\sigma_{\phi}^2G(\max\{1, \sqrt{\iota}\} + \ln(C_{\beta}))}{c_1\lambda\epsilon^2} + \ln(C_{\beta})\right)(\max\{1,\sqrt{\iota}\}+\ln(C_{\beta})) \\
        % &= \left(1 + \frac{\sigma_{\phi}^2G}{c_1\lambda\epsilon^2}\right)(\max\{1,\sqrt{\iota}\} + \ln(C_{\beta}))^2 \\
        &= \left(1 + \frac{\sigma_{\phi}^2G}{c_1\lambda\epsilon^2}\right)(\max\{1,\iota\} + 2\max\{1,\sqrt{\iota}\}\ln(C_{\beta}) + \ln^2(C_{\beta})) \\
        &\leq \left(1 + \frac{\sigma_{\phi}^2G}{c_1\lambda\epsilon^2}\right)(\max\{1,\iota\} + 2\max\{1,\sqrt{\iota}\}C_{\beta}^{1/2} + 4C_{\beta}^{1/2}) \\
        &\leq \left(1 + \frac{\sigma_{\phi}^2G}{c_1\lambda\epsilon^2}\right)(\max\{1,\iota\} + 6\max\{1,\sqrt{\iota},\iota\}C_{\beta}^{1/2}) \\
        &\leq \left(1 + \frac{\sigma_{\phi}^2G}{c_1\lambda\epsilon^2}\right)\max\{1,\sqrt{\iota},\iota\}\left(1+6C_{\beta}^{1/2}\right) \\
        % &= \left(1 + \frac{\sigma_{\phi}^2G}{c_1\lambda\epsilon^2}\right)\max\{1,\iota\} + \left(2 + \frac{2\sigma_{\phi}^2G}{c_1\lambda\epsilon^2}\right)\max\{1,\sqrt{\iota}\}\ln(C_{\beta}) + \frac{\sigma_{\phi}^2G}{c_1\lambda\epsilon^2}\ln^2(C_{\beta}) \\
        % &\leq \left(1 + \frac{\sigma_{\phi}^2G}{c_1\lambda\epsilon^2}\right)\max\{1,\iota\} + \left(1 + \frac{2\sigma_{\phi}^2G}{c_1\lambda\epsilon^2}\right)\max\{1,\sqrt{\iota}\}C_{\beta}^{1/2} + \frac{4\sigma_{\phi}^2G}{c_1\lambda\epsilon^2}C_{\beta}^{1/2} \\
        % &\leq \left(1 + \frac{\sigma_{\phi}^2G}{c_1\lambda\epsilon^2}\right)\max\{1,\iota\} + \left(1 + \frac{6\sigma_{\phi}^2G}{c_1\lambda\epsilon^2}\right)\max\{1,\sqrt{\iota}\}C_{\beta}^{1/2} \\
        % &\leq \left(1 + \frac{6\sigma_{\phi}^2G}{c_1\lambda\epsilon^2}\right)\max\{1,\sqrt{\iota},\iota\}\left(1+C_{\beta}^{1/2}\right),
    \end{aligned}
\end{equation}
where the second inequality uses $\ln x\leq x^{1/2}$ and $\ln x\leq 2x^{1/4}$ for $x>0$.
Thus, plugging \eqref{eq:max-times-max} back into \eqref{eq:ln-eT-delta} and we have
\begin{equation*}
    \begin{aligned}
        \ln\frac{4eT}{\delta}
        &= \ln\left(\frac{4C_2e\Delta_1L\sigma_{\phi}G^3(\max\{1,\sqrt{\iota}\}+\ln(1/\beta))(\max\{1,\sqrt{\iota}\}+\ln(C_{\beta}))}{c_1c_2\delta\lambda^2\epsilon^4}\right) \\
        &\leq \ln\left(\frac{4C_2e\Delta_1L\sigma_{\phi}G^3}{c_1c_2\delta\lambda^2\epsilon^4}\left(1 + \frac{\sigma_{\phi}^2G}{c_1\lambda\epsilon^2}\right)\max\{1,\sqrt{\iota},\iota\}\left(1+6C_{\beta}^{1/2}\right)\right) \\
        &\leq \ln(C_{\beta}),
    \end{aligned}
\end{equation*}
where the last inequality is due to 
\begin{equation*}
    \frac{4C_2e\Delta_1L\sigma_{\phi}G^3}{c_1c_2\delta\lambda^2\epsilon^4}\left(1 + \frac{\sigma_{\phi}^2G}{c_1\lambda\epsilon^2}\right)\max\{1,\sqrt{\iota},\iota\} \leq \frac{C_{\beta}}{2}
\end{equation*}
and
\begin{equation*}
    \frac{24C_2e\Delta_1L\sigma_{\phi}G^3}{c_1c_2\delta\lambda^2\epsilon^4}\left(1 + \frac{\sigma_{\phi}^2G}{c_1\lambda\epsilon^2}\right)\max\{1,\sqrt{\iota},\iota\}C_{\beta}^{1/2} \leq \frac{C_{\beta}}{2}
\end{equation*}
since
\begin{equation*}
    \begin{aligned}
        C_{\beta} \geq \max&\left\{\frac{8C_2e\Delta_1L\sigma_{\phi}G^3}{c_1c_2\delta\lambda^2\epsilon^4}\left(1 + \frac{\sigma_{\phi}^2G}{c_1\lambda\epsilon^2}\right)\max\{1,\sqrt{\iota},\iota\}, \right. \\
        &\left. \quad\quad\quad\quad \left(\frac{48C_2e\Delta_1L\sigma_{\phi}G^3}{c_1c_2\delta\lambda^2\epsilon^4}\left(1 + \frac{\sigma_{\phi}^2G}{c_1\lambda\epsilon^2}\right)\max\{1,\sqrt{\iota},\iota\}\right)^2\right\}.
    \end{aligned}
\end{equation*}

\paragraph{Verification for \eqref{eq:ln-Cbeta-verify}: $\|y_1-y_1^*\|^2 \leq 17\beta\sigma_{g,1}^2\ln(C_{\beta}) / \mu^2$.}
By choice of $T_0$ and $\gamma$, we have
\begin{equation*}
    \begin{aligned}
        \|y_1-y_1^*\|^2 
        &\leq \left(1-\frac{\mu\gamma}{2}\right)^{T_0}\|y_0-y_0^*\|^2 + \frac{8\gamma\sigma_{g,1}^2}{\mu}\ln\frac{4e}{\delta} \\
        &\leq \frac{\beta\sigma_{g,1}^2}{\mu^2} + \frac{16\beta\sigma_{g,1}^2}{\mu^2}\ln\frac{4e}{\delta} \\
        &\leq \frac{17\beta\sigma_{g,1}^2}{\mu^2}\ln(C_{\beta}),
    \end{aligned}
\end{equation*}
where the last inequality uses $T\geq 1/\beta^2 \geq 1$ and $\ln(4eT/\delta)\leq \ln(C_{\beta})$.

\paragraph{Verification for \eqref{eq:rho-less-r-verify}: $\varrho\leq \min\{r, 1/4L_1\}$.}
By \cref{lm:recursion-two-type} and choices of $\eta,\gamma$ and $\beta$, we have
\begin{equation*}
    \begin{aligned}
        \varrho^2
        &= \left(1 + \frac{8\eta^2l_{g,1}^2L^2}{\lambda^2\mu^4\gamma^2}\right) \|y_1-y_1^*\|^2 + \left(\frac{8\gamma}{\mu}\ln\frac{4eT}{\delta} + \frac{32\eta^2l_{g,1}^2L^2}{\lambda^2\mu^5\gamma}\right)\sigma_{g,1}^2 \\
        &\quad+ \frac{8\eta^2l_{g,1}^2C_{u,0}^2}{\lambda^2\mu^4\gamma^2} + \frac{1024\eta^4l_{g,1}^4L^2C_{u,0}^2}{\lambda^4\mu^8\gamma^4}\left(2+\ln\frac{1}{\beta}\right) \\
        % &\leq \left(1 + \frac{8\eta^2l_{g,1}^2L^2}{\lambda^2\mu^4\gamma^2}\right)\frac{17\beta\sigma_{g,1}^2}{\mu^2}\ln(C_{\beta}) + \left(\frac{8\gamma}{\mu}\ln\frac{4eT}{\delta} + \frac{32\eta^2l_{g,1}^2L^2}{\lambda^2\mu^5\gamma}\right)\sigma_{g,1}^2 \\
        % &\quad+ \frac{8\eta^2l_{g,1}^2C_{u,0}^2}{\lambda^2\mu^4\gamma^2} + \frac{1024\eta^4l_{g,1}^4L^2C_{u,0}^2}{\lambda^4\mu^8\gamma^4}\left(2+\ln\frac{1}{\beta}\right) \\
        &\leq \left(1 + \frac{2\eta^2l_{g,1}^2L^2}{\lambda^2\mu^2\beta^2}\right) \frac{17\beta\sigma_{g,1}^2}{\mu^2}\ln(C_{\beta}) + \left(\frac{16\beta}{\mu^2}\ln(C_{\beta}) + \frac{16\eta^2l_{g,1}^2L^2}{\lambda^2\mu^4\beta}\right)\sigma_{g,1}^2 \\
        &\quad+ \frac{2\eta^2l_{g,1}^2C_{u,0}^2}{\lambda^2\mu^2\beta^2} + \frac{64\eta^4l_{g,1}^4L^2}{\lambda^4\mu^4\beta^4}\left(2+\ln\frac{1}{\beta}\right) \\
        &\leq \left(1 + \frac{2c_2^2\sigma_{\phi}^2l_{g,1}^2}{\mu^2G^2}\right) \frac{17c_1\lambda\sigma_{g,1}^2\epsilon^2}{\mu^2\sigma_{\phi}^2G} + \left(\frac{16c_1\lambda\epsilon^2}{\mu^2\sigma_{\phi}^2G} + \frac{16c_1c_2^2\lambda l_{g,1}^2\epsilon^2}{\mu^4G^3}\right)\sigma_{g,1}^2 \\
        &\quad+ \frac{2c_2^2\sigma_{\phi}^2l_{g,1}^2C_{u,0}^2}{\mu^2L^2G^2} + \frac{192c_2^4\sigma_{\phi}^4l_{g,1}^4}{\mu^4L^2G^4}
        \leq \min\left\{r, \frac{1}{4L_1}\right\},
    \end{aligned}
\end{equation*}
where in the last inequality we choose small enough $c_1$ and $c_2$.

\paragraph{Verification of \eqref{eq:rho-less-r-verify}: $C_{u,0} + C_{u,1}\varrho + \lambda \leq 2G$.}
By definitions of $C_{u,0}, C_{u,1}$ in \eqref{eq:Cu-def} and choice of $G$, we have
\begin{equation*}
    \begin{aligned}
        C_{u,0} + C_{u,1}\varrho + \lambda
        &= C_{\phi,0} + G + (C_{\phi,1} + L_1G)\varrho + \lambda \\
        &\leq \frac{G}{4} + G + \frac{G}{2} + \frac{G}{4} = G.
    \end{aligned}
\end{equation*}

\paragraph{Verification for \eqref{eq:66-verify}.}
By choices of $\eta,\gamma$ and $\beta$, we have
\begin{equation*}
    \begin{aligned}
        &\frac{66L^2G}{\lambda}\left(\frac{8\eta^2l_{g,1}^2}{\lambda^2\mu^4\gamma^2} + \frac{2048\eta^4l_{g,1}^4L^2}{\lambda^4\mu^8\gamma^4}\left(2+\ln\frac{1}{\beta}\right)\right) \\
        &\leq \frac{66L^2G}{\lambda}\left(\frac{2\eta^2l_{g,1}^2}{\lambda^2\mu^2\beta^2} + \frac{128\eta^4l_{g,1}^4L^2}{\lambda^4\mu^4\beta^4}\left(2+\ln\frac{1}{\beta}\right)\right) 
        \leq \frac{66L^2G}{\lambda}\left(\frac{2c_2^2\sigma_{\phi}^2l_{g,1}^2}{\mu^2L^2G^2} + \frac{384c_2^4\sigma_{\phi}^4l_{g,1}^4}{\mu^4L^2G^4}\right) 
        \leq 2,
    \end{aligned}
\end{equation*}
where in the last inequality we choose small enough $c_2$.
\end{proof}

%%%%%%%%%%%%%%%%%%%%%%%%%%%%%%%%%%%%%%%%%%%%%%%%%%%%%%%%%%%%%%%%%%%%%%%%%%%%%%%%%%%%%%%%%%%%%%%

\begin{lemma} \label{lm:para-choice}
Denote $I_1$ and $I_2$ as 
\begin{equation} \label{eq:I1-I2-def}
    \begin{aligned}
        I_1 &\coloneqq \frac{16G}{\eta\lambda}\left(\lambda\Delta_1 + 64\sigma_{\phi}^2\left(\frac{\eta}{\beta} + \eta\beta T\right) + 160\eta\sigma_{\phi}^2\sqrt{(1/\beta^2+T)\ln(4/\delta)}\right) \\
        &\quad+ \frac{33L^2GT}{\lambda}\left(\left(1 + \frac{8\eta^2l_{g,1}^2L^2}{\lambda^2\mu^4\gamma^2}\right) \|y_1-y_1^*\|^2 + \left(\frac{8\gamma}{\mu}\ln\frac{4eT}{\delta} + \frac{32\eta^2l_{g,1}^2L^2}{\lambda^2\mu^5\gamma}\right)\sigma_{g,1}^2\right) \\
        &\quad+ \frac{67GTl_{g,1}^2l_{f,0}^2}{\lambda\mu^2}\left(1-\frac{\mu}{l_{g,1}}\right)^{2Q}, \\
        I_2 &\coloneqq \frac{8C_LG^3}{\eta L}.
    \end{aligned}
\end{equation}
For any given $\epsilon>0$, under the parameter choice in \cref{thm:main-appendix}, we have $I_1\leq I_2$ and $I_1/T \leq \epsilon^2$.
\end{lemma}

\begin{proof}[Proof of \cref{lm:para-choice}]
We first verify $I_1\leq I_2$ and then verify $I_1/T\leq \epsilon^2$.

\paragraph{Proof of $I_1\leq I_2$.}
We start to show $I_1/I_2\leq 1$. We have
{\allowdisplaybreaks
% \begin{equation*}
    \begin{align*}
        \frac{I_1}{I_2}
        &\leq \frac{2L}{\lambda C_LG^2}\left(\lambda\Delta_1 + 64\sigma_{\phi}^2\left(\frac{\eta}{\beta} + \eta\beta T\right) + 160\eta\sigma_{\phi}^2\sqrt{(1/\beta^2+T)\ln(4/\delta)}\right) \\
        &\quad+ \frac{5L^3\eta T}{\lambda C_LG^2}\left(\left(1 + \frac{8\eta^2l_{g,1}^2L^2}{\lambda^2\mu^4\gamma^2}\right) \|y_1-y_1^*\|^2 + \left(\frac{8\gamma}{\mu}\ln\frac{4eT}{\delta} + \frac{32\eta^2l_{g,1}^2L^2}{\lambda^2\mu^5\gamma}\right)\sigma_{g,1}^2\right) \\ 
        &\quad+ \frac{9l_{g,1}^2l_{f,0}^2L\eta T}{\lambda\mu^2 C_LG^2}\left(1-\frac{\mu}{l_{g,1}}\right)^{2Q} \\
        &\leq \frac{\lambda\Delta_1}{8\lambda\Delta_1} + \frac{2L}{\lambda C_LG^2}\left( 64\sigma_{\phi}^2\left(\frac{2\eta}{\beta}+\frac{C_2\Delta_1G\beta}{\epsilon^2}\right) + 160\sigma_{\phi}^2\sqrt{\iota}\sqrt{\frac{5\eta^2}{2\beta^2}+\frac{1}{2}\left(\frac{C_2\Delta_1G\beta}{\epsilon^2}\right)^2}\right) \\
        &\quad+ \frac{5L^3}{\lambda C_LG^2}\left(\frac{\eta}{\beta^2}+\frac{C_2\Delta_1G}{\epsilon^2}\right)\left(1 + \frac{2\eta^2l_{g,1}^2L^2}{\lambda^2\mu^2\beta^2}\right) \frac{17\beta\sigma_{g,1}^2}{\mu^2}\ln(C_{\beta}) \\ 
        &\quad+ \frac{5L^3}{\lambda C_LG^2}\left(\frac{\eta}{\beta^2}+\frac{C_2\Delta_1G}{\epsilon^2}\right)\left(\frac{16\beta}{\mu^2}\ln(C_{\beta}) + \frac{16\eta^2l_{g,1}^2L^2}{\lambda^2\mu^4\beta}\right)\sigma_{g,1}^2 \\
        &\quad+ \frac{9l_{g,1}^2l_{f,0}^2L}{\lambda\mu^2C_LG^2}\left(\frac{\eta}{\beta^2}+\frac{C_2\Delta_1G}{\epsilon^2}\right)\beta \\
        % &\leq \frac{1}{8} + \frac{16L}{\lambda C_LG^2}\left(48\sigma_{\phi}^2\max\{1,\sqrt{\iota}\}\frac{\eta}{\beta} + \frac{24C_2\sigma_{\phi}^2\Delta_1G\max\{1,\sqrt{\iota}\}\beta}{\epsilon^2}\right) \\
        % &\quad+ \frac{5L^3}{8\lambda C_LG^2}\left(\frac{\eta}{\beta^2}+\frac{C_2\Delta_1G}{\epsilon^2}\right)\left(1 + \frac{2\eta^2l_{g,1}^2L^2}{\lambda^2\mu^2\beta^2}\right) \frac{17\beta\sigma_{g,1}^2}{\mu^2}\ln(C_{\beta}) \\ 
        % &\quad+ \frac{5L^3}{\lambda C_LG^2}\left(\frac{\eta}{\beta^2}+\frac{C_2\Delta_1G}{\epsilon^2}\right)\left(\frac{16\beta}{\mu^2}\ln(C_{\beta}) + \frac{16\eta^2l_{g,1}^2L^2}{\lambda^2\mu^4\beta}\right)\sigma_{g,1}^2 \\
        % &\quad+ \frac{9l_{g,1}^2l_{f,0}^2L}{\lambda\mu^2C_LG^2}\left(\frac{\eta}{\beta^2}+\frac{C_2\Delta_1G}{\epsilon^2}\right)\beta \\
        &\leq  \frac{1}{8} + \frac{16L}{\lambda C_LG^2}\left(\frac{48c_2\lambda\sigma_{\phi}^3}{LG} + 24c_1C_2\lambda\Delta_1\right) \\
        &\quad+ \frac{5L^3}{\lambda C_LG^2}\left(\frac{c_2\sigma_{\phi}\lambda}{LG}+\frac{c_1C_2\lambda\Delta_1}{\sigma_{\phi}^2}\right)\left(1 + \frac{2c_2^2\sigma_{\phi}^2l_{g,1}^2}{\mu^2G^2}\right) \frac{17\sigma_{g,1}^2}{\mu^2} \\ 
        &\quad+ \frac{5L^3}{\lambda C_LG^2}\left(\frac{c_2\sigma_{\phi}\lambda}{LG}+\frac{c_1C_2\lambda\Delta_1}{\sigma_{\phi}^2}\right)\left(\frac{16}{\mu^2}+\frac{2c_2^2\sigma_{\phi}^2l_{g,1}^2}{\mu^4G^2}\right)\sigma_{g,1}^2 \\
        &\quad+ \frac{9l_{g,1}^2l_{f,0}^2L}{\lambda\mu^2C_LG^2}\left(\frac{c_2\sigma_{\phi}\lambda}{LG}+\frac{c_1C_2\lambda\Delta_1}{\sigma_{\phi}^2}\right) \\
        &= \frac{1}{8} + \frac{384L}{\lambda C_LG^2}\left(\frac{2c_2\lambda\sigma_{\phi}^3}{LG} + c_1C_2\lambda\Delta_1\right) \\
        &\quad+ \left(\frac{c_2\sigma_{\phi}\lambda}{LG}+\frac{c_1C_2\lambda\Delta_1}{\sigma_{\phi}^2}\right)\left(\frac{5L^3}{\lambda C_LG^2}\left(\frac{17}{\mu^2}+\frac{36c_2^2\sigma_{\phi}^2l_{g,1}^2}{\mu^4G^2}\right)\sigma_{g,1}^2 + \frac{9l_{g,1}^2l_{f,0}^2L}{\lambda\mu^2C_LG^2}\right) \\
        &\leq 1,
    \end{align*}
% \end{equation*}
}%
where the first inequality is due to \eqref{eq:I1-I2-def}; the second inequality uses large enough $C_1$ and \citep[Lemma C.5]{li2023convergence}, the fact that $\ln(4eT/\delta)\leq \ln(C_{\beta})$ and $\ln(4eT/\delta)\leq \ln(C_{\beta})$, and the choice of $\gamma, Q, T$ that 
\begin{equation*}
    \eta T \leq \frac{\eta}{\beta^2}+\frac{C_2\Delta_1G}{\epsilon^2},
    \quad
    \gamma = \frac{2\beta}{\mu},
    \quad
    Q \geq \frac{1}{2}\ln\beta \Big/ \ln\left(1-\frac{\mu}{l_{g,1}}\right);
\end{equation*}
the third inequality uses \citep[Lemma C.5]{li2023convergence}, the choice of $\eta,\beta$ that
\begin{equation*}
    \frac{\eta}{\beta} \leq \frac{c_2\sigma_{\phi}\lambda}{LG\max\{1, \sqrt{\iota}, \ln(C_{\beta})\}}, 
    \quad
    \frac{\beta}{\epsilon^2} \leq \frac{c_1\lambda}{\sigma_{\phi}^2G\max\{1, \sqrt{\iota}, \ln(C_{\beta})\}};
\end{equation*}
and in the last inequality we choose small enough $c_1$ and $c_2$.

\paragraph{Proof of $I_1/T\leq \epsilon^2$.}
Last, we show $I_1/T\leq \epsilon^2$. We have
\begin{equation*}
    \begin{aligned}
        \frac{I_1}{T}
        &= \frac{16G\Delta_1}{\eta T} + \frac{1024\sigma_{\phi}^2G}{\lambda\beta T} + \frac{1024\sigma_{\phi}^2G}{\lambda} + \frac{2560\sigma_{\phi}^2G\sqrt{\iota}}{\lambda}\sqrt{\frac{1}{\beta^2T^2}+\frac{1}{T}} \\
        &\quad+ \frac{33L^2G}{\lambda}\left(\left(1 + \frac{8\eta^2l_{g,1}^2L^2}{\lambda^2\mu^4\gamma^2}\right) \|y_1-y_1^*\|^2 + \left(\frac{8\gamma}{\mu}\ln\frac{4eT}{\delta} + \frac{32\eta^2l_{g,1}^2L^2}{\lambda^2\mu^5\gamma}\right)\sigma_{g,1}^2\right) \\
        &\quad+ \frac{67Gl_{g,1}^2l_{f,0}^2}{\lambda\mu^2}\left(1-\frac{\mu}{l_{g,1}}\right)^{2Q} \\
        &\leq \frac{16\epsilon^2}{C_2} + \frac{3584\sigma_{\phi}^2G\max\{1,\sqrt{\iota}\}}{\lambda\beta T} + \frac{1024\sigma_{\phi}^2G}{\lambda} + \frac{2560\sigma_{\phi}^2G\sqrt{\iota}}{\lambda\sqrt{T}} + 67c_3\epsilon^2 \\
        &\quad+ \frac{33L^2G}{\lambda}\left(\left(1 + \frac{2\eta^2l_{g,1}^2L^2}{\lambda^2\mu^2\beta^2}\right) \frac{17\beta\sigma_{g,1}^2}{\mu^2}\ln(C_{\beta}) + \left(\frac{16\beta}{\mu^2}\ln(C_{\beta}) + \frac{16\eta^2l_{g,1}^2L^2}{\lambda^2\mu^4\beta}\right)\sigma_{g,1}^2\right) \\
        &\leq \frac{16\epsilon^2}{C_2} + \frac{7168\sigma_{\phi}^2G\max\{1,\sqrt{\iota}\}\beta}{\lambda} + 67c_3\epsilon^2 \\
        &\quad+ \frac{33L^2G}{\lambda}\left(\left(1 + \frac{2c_2^2\sigma_{\phi}^2l_{g,1}^2}{\mu^2G^2}\right) \frac{17c_1\lambda\sigma_{g,1}^2\epsilon^2}{\mu^2\sigma_{\phi}^2G} + \left(\frac{16c_1\lambda\epsilon^2}{\mu^2\sigma_{\phi}^2G} + \frac{16c_1c_2^2\lambda l_{g,1}^2\epsilon^2}{\mu^4G^3}\right)\sigma_{g,1}^2\right) \\
        &= \left(\frac{16}{C_2}+7168c_1+67c_3+\frac{33c_1L^2G}{\lambda}\left(\left(1 + \frac{2c_2^2\sigma_{\phi}^2l_{g,1}^2}{\mu^2G^2}\right) \frac{17\lambda\sigma_{g,1}^2}{\mu^2\sigma_{\phi}^2G} + \left(\frac{16\lambda}{\mu^2\sigma_{\phi}^2G} + \frac{16c_2^2\lambda l_{g,1}^2}{\mu^4G^3}\right)\sigma_{g,1}^2\right)\right)\epsilon^2 \\
        &\leq \epsilon^2,
    \end{aligned}
\end{equation*}
where the first inequality uses $\sqrt{a+b}\leq \sqrt{a}+\sqrt{b}$ for $a,b\geq0$, the fact that $\ln(4eT/\delta)\leq \ln(C_{\beta})$ and $\|y_1-y_1^*\|^2 \leq 17\beta\sigma_{g,1}^2\ln(C_{\beta}) / \mu^2$, and the choice of $T,Q,\gamma$ that
\begin{equation*}
    T \geq \frac{C_2\Delta_1G}{\eta\epsilon^2},
    \quad
    Q\geq \frac{1}{2}\ln\left(\frac{c_3\lambda\mu^2\epsilon^2}{Gl_{g,1}^2l_{f,0}^2}\right) \Big/ \ln\left(1-\frac{\mu}{l_{g,1}}\right),
    \quad
    \gamma = \frac{2\beta}{\mu};
\end{equation*}
the second inequality uses the choice of $T,\eta,\beta$ that
\begin{equation*}
    T \geq \frac{1}{\beta^2},
    \quad
    \eta \leq \frac{c_2\sigma_{\phi}\lambda\beta}{LG},
    \quad
    \beta \leq \frac{c_1\lambda\epsilon^2}{\sigma_{\phi}^2G\max\{1,\sqrt{\iota},\ln(C_{\beta})\}};
\end{equation*}
and in the last inequality we choose small enough $c_1,c_2,c_3$ and large enough $C_2$.
\end{proof}

\section{More Experimental Details} 
\label{app:hyper_param}

\subsection{AUC Maximization on RNNs}\label{app:rnn_results}
The results with RNNs for both training and testing over 25 epochs are presented in \cref{fig:auc_rnn} (a) and (b), while the corresponding running times are shown in \cref{fig:auc_rnn} (c) and (d). Our proposed Adam-type algorithm, AdamBO, shows the faster convergence rate and significantly outperform other baselines during training process.

\begin{figure*}[!t]
\begin{center}
\subfigure[\scriptsize Train AUC]{\includegraphics[width=0.24\linewidth]{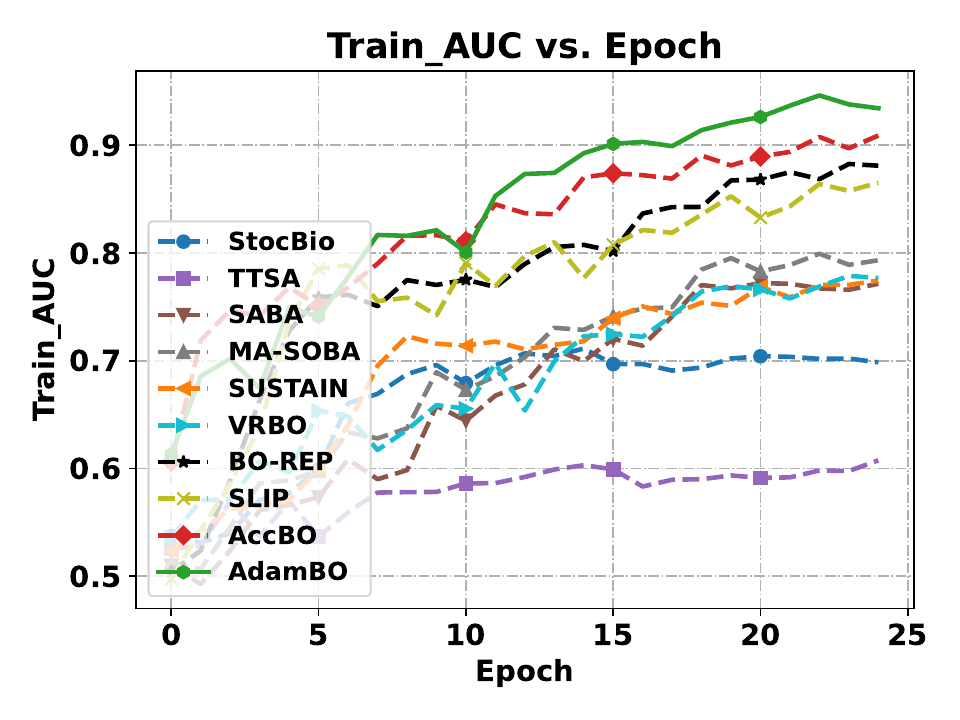}}   
\subfigure[\scriptsize Test AUC]{\includegraphics[width=0.24\linewidth]{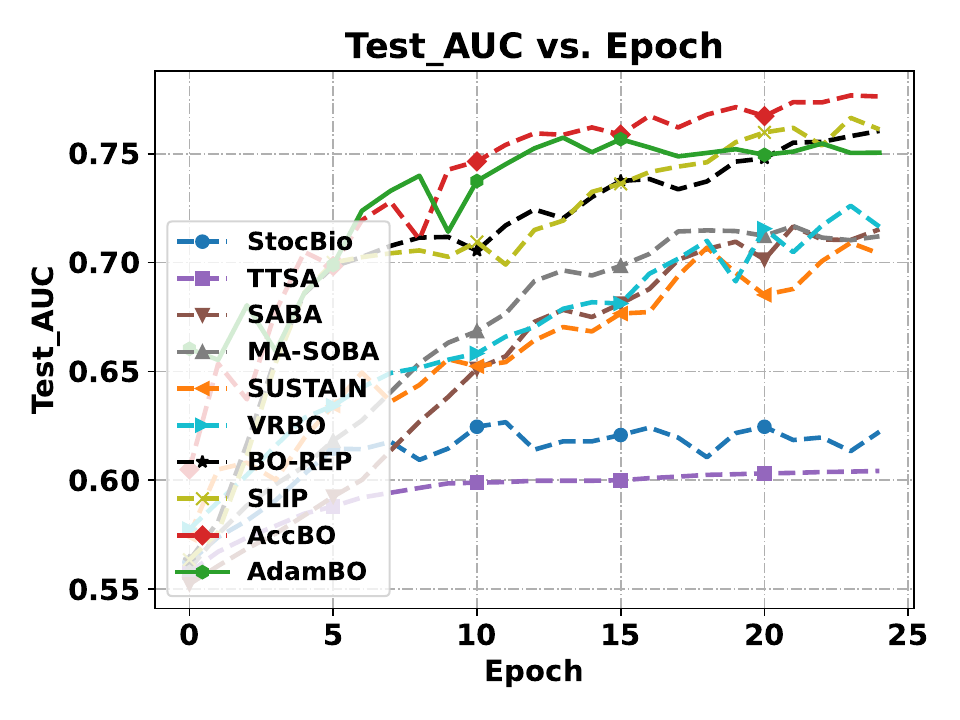}}   
\subfigure[\scriptsize Train AUC vs. Time]{\includegraphics[width=0.24\linewidth]{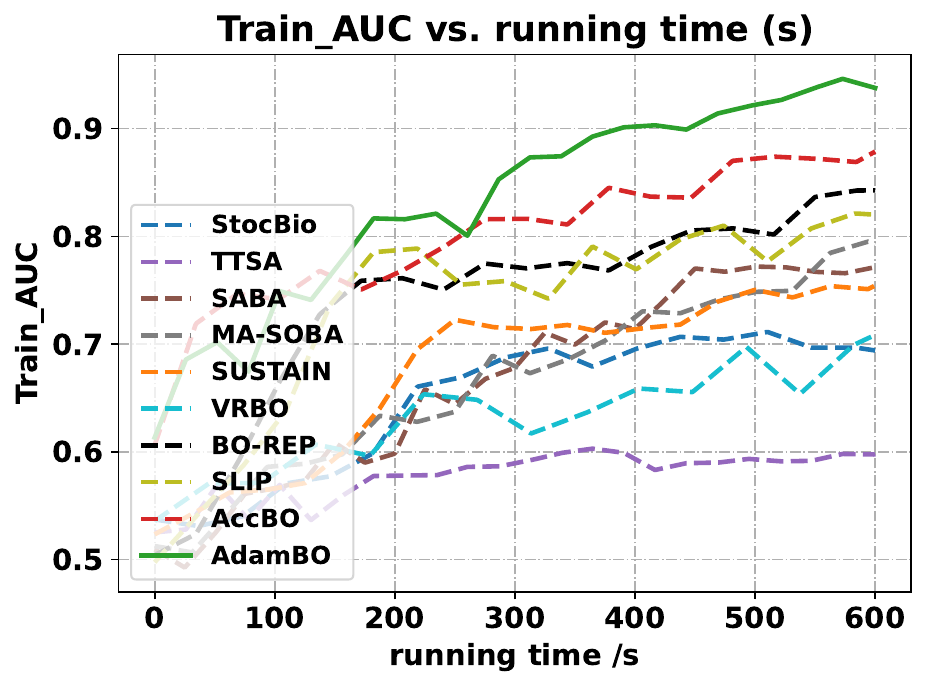}}  
\subfigure[\scriptsize Test AUC vs. Time]{\includegraphics[width=0.24\linewidth]{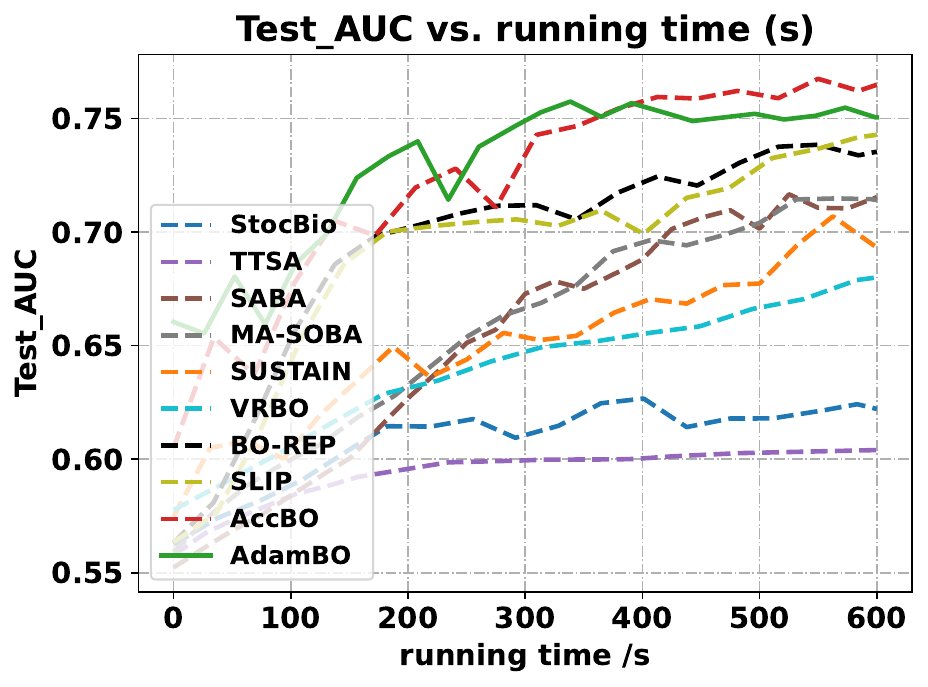}}
\end{center}
\vspace*{-0.15in}
\caption{RNN for AUC maximization on Sentiment140 dataset with imbalance ratio of 0.8. Figures (a), (b) are the results over epochs, Figures (c), (d) are the results over running time.}
\label{fig:auc_rnn}
\end{figure*}

\begin{figure*}[t]
\begin{center}
\subfigure[\scriptsize RNN on AUC maximization]{\includegraphics[width=0.3\linewidth]{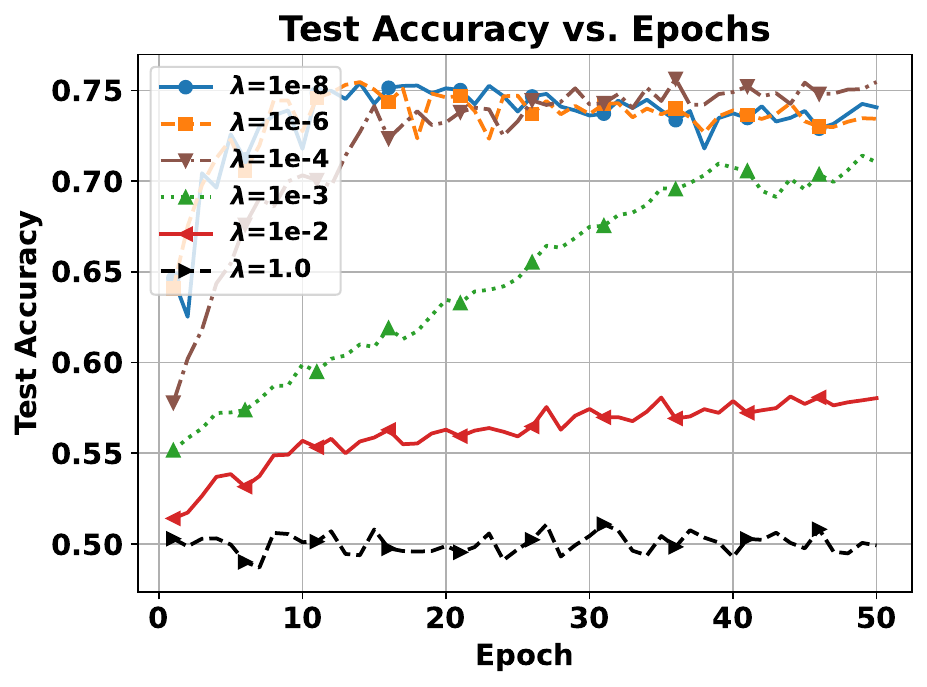}}   \   
\subfigure[\scriptsize Transformer on AUC maximization]{\includegraphics[width=0.3\linewidth]{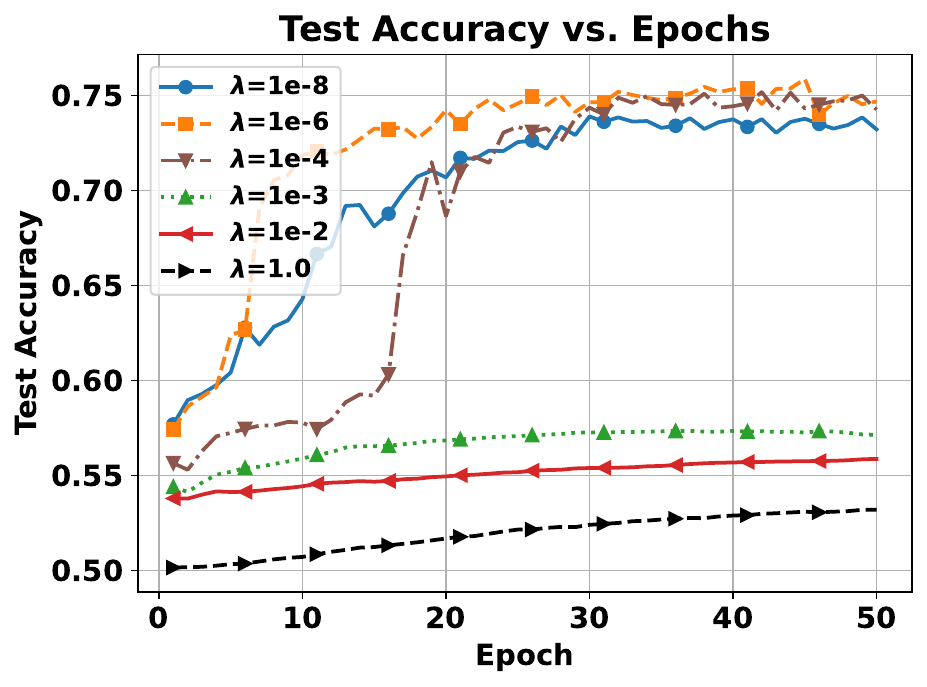}} \ 
% \subfigure[\scriptsize RNN on hyper-representation]{\includegraphics[width=0.24\linewidth]{figures/lambda_meta_lr_rnn.pdf}} \
\subfigure[\scriptsize BERT on hyper-representation]{\includegraphics[width=0.3\linewidth]{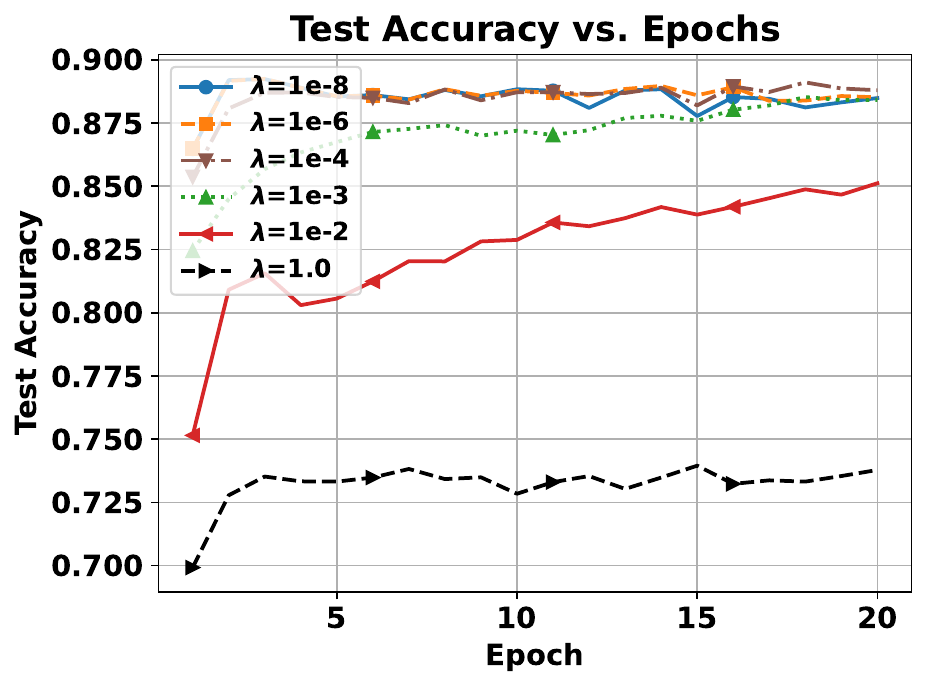}} 
\end{center}
\vspace*{-0.15in}
\caption{Test accuracy of different models on AUC maximization and hyper-representaion using AdamBO with $\beta=0.1, \betasq=0.001$ and different $\lambda$s. (a) a 2-layer RNN model on AUC maximization (data imbalanced ratio = 0.8);  (b) a 2-layer Transformer model on AUC maximization (data imbalanced ratio = 0.9); (c) an 8-layer BERT model on hyper-representation.}
\label{fig:lambda_sensitivity}
\end{figure*}

\subsection{Sensitivity to the Choice of $\lambda$}

% \begin{figure*}[t]
% \begin{center}
% \subfigure[\scriptsize RNN on AUC maximization]{\includegraphics[width=0.3\linewidth]{figures/lambda_rnn_auc.pdf}}   \   
% \subfigure[\scriptsize Transformer on AUC maximization]{\includegraphics[width=0.3\linewidth]{figures/lambda_trans_auc.pdf}} \ 
% % \subfigure[\scriptsize RNN on hyper-representation]{\includegraphics[width=0.24\linewidth]{figures/lambda_meta_lr_rnn.pdf}} \
% \subfigure[\scriptsize BERT on hyper-representation]{\includegraphics[width=0.3\linewidth]{figures/lambda_meta_lr_bert.pdf}} 
% \end{center}
% \vspace*{-0.15in}
% \caption{Test accuracy of different models on AUC maximization and hyper-representaion using AdamBO with $\beta=0.1, \betasq=0.001$ and different $\lambda$s. (a) a 2-layer RNN model on AUC maximization (data imbalanced ratio = 0.8);  (b) a 2-layer Transformer model on AUC maximization (data imbalanced ratio = 0.9); (c) an 8-layer BERT model on hyper-representation.}
% \label{fig:lambda_sensitivity}
% \end{figure*}

% We conduct additional experiments in \cref{fig:lambda_sensitivity} to show the empirical performance of our algorithm is not very sensitive to the choice of $\lambda$. Although the default choice of $\lambda$ is $10^{-8}$ \citep{kingma2014adam}, increasing it up to $10^{-4}$ only causes minor differences in AUC maximization, and increasing it up to $10^{-3}$ leads to minor changes in hyper-representation performance with BERT \citep{devlin2018bert}. 

\cref{fig:lambda_sensitivity} shows the empirical performance of our algorithm is not very sensitive to the choice of $\lambda$. Although the default choice of $\lambda$ is $10^{-8}$ \citep{kingma2014adam}, increasing it up to $10^{-4}$ only causes minor differences in AUC maximization, and increasing it up to $10^{-3}$ leads to minor changes in hyper-representation performance with BERT \citep{devlin2018bert}.

\subsection{Hyerparameter Settings for Hyper-representation}

\textbf{Hyper-representation on BERT.} 
The upper-level learning rate $\eta$ and the lower-level learning rate $\gamma$ are tuned in a range of $[1.0\times 10^{-4}, 0.1]$ for all the baselines. The optimal learning rate pairs $(\eta, \gamma)$ are, $(0.01, 0.001)$ for MAML, $(0.01, 0.02)$ for ANIL, $(0.01, 0.002)$ for StocBio, $(0.01, 0.001)$ for TTSA, $(0.01, 0.01)$ for SABA, $(0.01, 0.01)$ for MA-SOBA, and $(0.1, 0.05)$ for both BO-REP and SLIP, $(1.0\times 10^{-4}, 5.0\times 10^{-3})$ for AdamBO. 
% Please refer to our code \url{https://anonymous.4open.science/r/AdamBO} for more experimental details.

\textbf{Hyper-representation on RNN.} 
The upper-level learning rate $\eta$ and the lower-level learning rate $\gamma$ are tuned in a range of $[1.0\times 10^{-4}, 0.1]$ for all the baselines. The optimal learning rate pairs are listed as follows, $(0.01, 0.01)$ for MAML, $(0.01, 0.05)$ for ANIL, $(0.01, 0.01)$ for StocBio, $(0.02, 0.05)$ for TTSA, $(0.01, 0.05)$ for SABA, $(0.05, 0.05)$ for MA-SOBA, and $(0.1, 0.05)$ for both BO-REP and SLIP, $(1.0\times 10^{-4}, 1.0\times 10^{-3})$ for AdamBO.

Other hyper-parameter settings are summarized as follows. The steps for neumann series estimation in StocBiO, AdamBO is set to 3, while it is uniformly sampled from $\{1,2,3\}$ in TTSA. The momentum parameter $\beta=0.1$ is fixed in SLIP, MA-SOBA, BO-REP, AdamBO, and $\betasq=0.001$ in AdamBO. The warm start steps for the lower level variable in BO-REP, SLIP, AdamBO are set to $3$. The number of inner loops for StocBio is set to $3$. BO-REP uses the periodic update for the low-level variable, and sets the iterations $N=3$ and the update interval $I=2$. The hyperparameter $\lambda$ in the Adam update is fixed as $1.0\times 10^{-8}$ in AdamBO.

\subsection{Hyerparameter Settings for Deep AUC Maxmization}
We tune the best hyperparameters for each algorithm, including upper-/lower-level step size, the number of inner loops, momentum parameters, etc. The upper-level learning rate $\eta$ and the lower-level learning rate $\gamma$ are tuned in a wide range of $[1.0\times 10^{-6}, 0.1]$ for all the baselines on experiments of AUC maximization. 

\textbf{AUC maximization on Transformer}. 
The best learning rates $(\eta, \gamma)$ are summarized as follows: Stocbio: $(0.005, 0.0001)$, TTSA: $(0.0005, 0.001)$, SABA: $(0.001, 0.005)$, MA-SOBA: $(0.0005, 0.005)$, SUSTAIN: $(0.005, 0.001)$, VRBO: $(0.005, 0.0005)$, BO-REP: $(0.0001, 0.0001)$, SLIP: $(0.0001, 0.001)$, AccBO: $(0.0005, 0.0001)$,   AdamBO: $(5.0\times 10^{-6}, 0.005)$. Note that SUSTAIN decays its upper-/lower-level step size with epoch ($t$) by $\eta= \eta/(t+2)^{1/3}, \eta_{low}= \gamma /(t+2)^{1/3}$. Other algorithms use a constant learning rate. 

\textbf{AUC maximization on RNN}. 
The best learning rates $(\eta, \gamma)$ are summarized as follows: StocBio: ($0.01, 0.001$), TTSA: $(0.005, 0.01)$, SABA: $(0.01, 0.005)$, MA-SOBA: $(0.01, 0.005)$, SUSTAIN: $(0.03, 0.01)$, VRBO: $(0.05, 0.01)$, BO-REP: $(0.001, 0.001)$, SLIP: $(0.001, 0.001)$, AccBO: $(0.005, 0.005)$, AdamBO: $(1.0\times 10^{-5}, 0.001)$.

Other hyper-parameter settings are summarized as follows. The steps for neumann series estimation in StocBiO, VRBO, and AdamBO is set to 3, while it is uniformly sampled from $\{1,2,3\}$ in TTSA, SUSTAIN, and AccBO. AccBO uses the Nesterov accelerated gradient descent for the lower-level update, the momentum parameter $\alpha=0.5$ for AccBO, the averaging parameter $\nu=0.5$ for AccBO. The batch size is set to $32$ for all algorithms except VRBO, which uses a larger batch size of $64$ (tuned in the range of $\{32, 64, 128, 256, 512\}$) at the checkpoint (snapshot) step and $32$ otherwise. The momentum parameter $\beta=0.1$ is fixed in SLIP, AccBO, MA-SOBA, BO-REP, and AdamBO, and $\betasq=0.001$ in AdamBO. The warm start steps for the lower level variable in BO-REP, SLIP, AccBO, and AdamBO are set to $3$. The number of inner loops for StocBio is set to $3$. BO-REP uses the periodic updates for low-level variable, and sets the iterations $N=3$ and the update interval $I=2$. The hyperparameter $\lambda$ in the Adam update is fixed as $1.0\times 10^{-8}$ for AdamBO.

\section{Comparison Tables}
\label{app:table}

\begin{assumption} \label{ass:adam-smoothness}
Consider the following smoothness assumptions:
\begin{enumerate}[(A)]
    \item The objective function is $L$-smooth. \label{ass:adam-s1}
    \item The objective function is $(L_0,L_1)$-smooth \citep[Definition 1.1, Remark 2.3]{zhang2020improved}. \label{ass:adam-s2}
    \item The objective function is $(\rho, L_0, L_{\rho})$-smooth with $0\leq \rho < 2$ \citep[Definition 3.2]{li2023convergence}. \label{ass:adam-s3}
\end{enumerate}
The above assumptions satisfy:
\crefdefpart{ass:adam-smoothness}{ass:adam-s1} $\Longrightarrow$ \crefdefpart{ass:adam-smoothness}{ass:adam-s2} $\Longrightarrow$ \crefdefpart{ass:adam-smoothness}{ass:adam-s3}.
In other words, \crefdefpart{ass:adam-smoothness}{ass:adam-s1} is the strongest, and \crefdefpart{ass:adam-smoothness}{ass:adam-s3} is the weakest.
\end{assumption}

\begin{assumption} \label{ass:adam-bounded-gradient}
The (stochastic) gradient norm of the objective function is (almost surely) bounded.
\end{assumption}

\begin{assumption} \label{ass:prior-noise}
Suppose the following stochastic estimators are unbiased and satisfy:
% \begin{small}
\begin{equation*}
    \E_{\xi\sim\gD_f}[\|\gdx F(x,y;\xi)-\gdx f(x,y)\|^2] \leq \sigma_{f,1}^2, \quad
    \E_{\xi\sim\gD_f}[\|\gdy F(x,y;\xi)-\gdy f(x,y)\|^2] \leq \sigma_{f,1}^2,
\end{equation*}
\begin{equation*}
    \Pr\{\|\gdy G(x,y;\xi)-\gdy g(x,y)\| \geq \lambda\} \leq 2\exp(-2\lambda^2/\sigma_{g,1}^2) \quad \forall \lambda>0,
\end{equation*}
\begin{equation*}
    \E_{\zeta\sim\gD_g}[\|\gdxy G(x,y;\zeta)-\gdxy g(x,y)\|^2] \leq \sigma_{g,2}^2, \quad
    \E_{\zeta\sim\gD_g}[\|\gdyy G(x,y;\zeta)-\gdyy g(x,y)\|^2] \leq \sigma_{g,2}^2.
\end{equation*}
% \end{small}%
\end{assumption}

\textbf{Remark.} 
Existing convergence analyses of (single-level) Adam that do not need such choice of $\beta$ require other strong assumptions for the objective function, which is incompatible to our setting. They either rely on the bounded gradient assumption \citep{de2018convergence,defossez2020simple}, or they only prove convergence to some neighborhood of stationary points with a constant radius unless assuming the strong growth condition under the finite sum setting \citep{zhang2022adam,wang2022provable}. Please see \cref{tab:adam} in \cref{app:table} for more details.

\begin{table}[!h]
    \centering
    \caption{Comparison of Adam-related papers under different settings and assumptions. \halfcheck \ represents dropping the bias correction term for the first-order momentum while keeping it for the second-order momentum. $d$ denotes the dimension. Only the key assumptions are listed here.}
    \label{tab:adam}
    \renewcommand{\arraystretch}{0.7}
    \setlength{\tabcolsep}{6pt}
    \resizebox{\textwidth}{!}{
    \begin{tabular}{ccccccc}
    \toprule[2pt]
    Adam Paper & Problem & Stochastic Setting & Assumptions & Bias Correction & Complexity  \\
    \midrule[1pt]
    \cite{de2018convergence} & Single-Level & Deterministic & \refdefpart{ass:adam-smoothness}{ass:adam-s1} + \ref{ass:adam-bounded-gradient} & \ding{55} & $O(\epsilon^{-6})$  \\ \midrule
    \cite{defossez2020simple} & Single-Level & Stochastic (Expectation) & \refdefpart{ass:adam-smoothness}{ass:adam-s1} + \ref{ass:adam-bounded-gradient} & \halfcheck & $\widetilde{O}(d\epsilon^{-4})$ \\ \midrule
    \cite{guo2021novel} & Single-Level & Stochastic (Expectation) & \refdefpart{ass:adam-smoothness}{ass:adam-s1} + \ref{ass:adam-bounded-gradient} \tablefootnote{\citep[Assumption 2]{guo2021novel} can be implied by \cref{ass:adam-bounded-gradient}, although it is weaker.} & \ding{55} & $O(\epsilon^{-4})$ \\ \midrule
    \cite{zhang2022adam} & Single-Level & Stochastic (Finite Sum) & \refdefpart{ass:adam-smoothness}{ass:adam-s1} & \ding{51} (Randomly Reshuffled) & Not Converge \tablefootnote{Adam can converge with an additional strong growth condition \citep{zhang2022adam,wang2022provable}.} \\ \midrule
    \cite{wang2022provable} & Single-Level & Stochastic (Finite Sum) & \refdefpart{ass:adam-smoothness}{ass:adam-s2} & \ding{55} (Randomly Reshuffled) & Not Converge \\ \midrule
    \cite{li2023convergence} & Single-Level & Stochastic (Expectation) & \refdefpart{ass:adam-smoothness}{ass:adam-s3} & \ding{51} & $O(\epsilon^{-4})$ \\ \midrule
    \rowcolor[rgb]{ .741,  .843,  .933} \Gape[1pt][1pt]{\makecell{AdamBO \\ (This work, \cref{thm:main})}} & Bilevel & Stochastic (Expectation) & \refdefpart{ass:adam-smoothness}{ass:adam-s2} \tablefootnote{Under Assumption 3.2, the objective function $\Phi$ is $(L_0,L_1)$-smooth, see Lemma B.10 for details.} & \ding{51} & $\widetilde{O}(\epsilon^{-4})$ \\  
    % \bottomrule[2pt]
    % \\
    % \toprule[2pt]
    % Variance-Reduced Adam Paper & Problem & Stochastic Setting & Assumptions & Bias Correction & Complexity  \\
    % \midrule[1pt]
    % VR ADAM \citep{wang2022divergence} & Single-Level & Stochastic (Expectation) & \refdefpart{ass:adam-smoothness}{ass:adam-s1} + \ref{ass:adam-bounded-gradient} & \ding{51} (Resetting) & Asymptotic Convergence \\ \midrule
    % VRAdam \citep{li2023convergence} & Single-Level & Stochastic (Expectation) & \refdefpart{ass:adam-smoothness}{ass:adam-s3} & \halfcheck & $O(\epsilon^{-3})$ \\ \midrule
    % \rowcolor[rgb]{ .741,  .843,  .933} \Gape[1pt][1pt]{\makecell{VR-AdamBO \\ (This work, \cref{thm:vr-main})}} & Bilevel & Stochastic (Expectation) & \refdefpart{ass:adam-smoothness}{ass:adam-s2} & \halfcheck & $\widetilde{O}(\epsilon^{-3})$ \\ 
    \bottomrule[2pt]
    \end{tabular}}%
\end{table}%

\begin{table}[!h]
    \centering
    \caption{Comparison of bilevel optimization algorithms under the unbounded smoothness setting.}
    \label{tab:bilevel}
    \renewcommand{\arraystretch}{0.7}
    \setlength{\tabcolsep}{6pt}
    \resizebox{\textwidth}{!}{
    \begin{tabular}{cccccccc}
    \toprule[2pt]
    Method & Problem & Stochastic Setting & Loop Style & Assumptions & Adam-Type & Learning Rate $\eta$ & Complexity  \\
    \midrule[1pt]
    BO-REP \citep{hao2024bilevel} & Bilevel & Stochastic (Expectation) & Double & \cref{ass:bilevel-assumption,ass:prior-noise} & \ding{55} & $O(\epsilon^3)$ & $\widetilde{O}(\epsilon^{-4})$  \\ \midrule
    SLIP \citep{gong2024a} & Bilevel & Stochastic (Expectation) & Single & \cref{ass:bilevel-assumption,ass:prior-noise} & \ding{55} & $\widetilde{\Theta}(\epsilon^3)$ & $\widetilde{O}(\epsilon^{-4})$  \\ \midrule
    \rowcolor[rgb]{ .741,  .843,  .933} \Gape[1pt][1pt]{\makecell{AdamBO \\ (This work, \cref{thm:main})}} & Bilevel & Stochastic (Expectation) & Single & \cref{ass:bilevel-assumption,ass:noise,ass:bilevel-additional} & \ding{51} & $\widetilde{\Theta}(\epsilon^2)$ & $\widetilde{O}(\epsilon^{-4})$ \\  
    % \bottomrule[2pt]
    % \\
    % \toprule[2pt]
    % Method (Variance-Reduction) & Problem & Stochastic Setting & Loop Style & Assumptions & Adam-Type & Learning Rate $\eta$ & Complexity  \\
    % \midrule[1pt]
    % AccBO \citep{gong2024accelerated} & Bilevel & Stochastic (Expectation) & Double \tablefootnote{The single-loop version (Option I) of AccBO \citep{gong2024accelerated} only works for one-dimensional quadratic lower-level function.} & \cref{ass:bilevel-assumption,ass:prior-noise,ass:individual-noise} & \ding{55} & $\widetilde{\Theta}(\epsilon^2)$ & $\widetilde{O}(\epsilon^{-3})$  \\ \midrule
    % \rowcolor[rgb]{ .741,  .843,  .933} \Gape[1pt][1pt]{\makecell{VR-AdamBO \\ (This work, \cref{thm:vr-main})}} & Bilevel & Stochastic (Expectation) & Double & \cref{ass:bilevel-assumption,ass:noise,ass:individual-noise} & \ding{51} & $O(\epsilon)$ & $\widetilde{O}(\epsilon^{-3})$ \\ 
    \bottomrule[2pt]
    \end{tabular}}%
\end{table}%

% \section{Additional Experiments}
% \label{app:add_exp}
% \input{icml2025_appendix/experiments}

% \section{Proof Sketch for VR-AdamBO (\cref{thm:vr-main})}
% \label{app:proof_sketch_vr}
% \input{icml2025_appendix/proof_sketch_vr}

%%%%%%%%%%%%%%%%%%%%%%%%%%%%%%%%%%%%%%%%%%%%%%%%%%%%%%%%%%%%%%%%%%%%%%%%%%%%%%%
%%%%%%%%%%%%%%%%%%%%%%%%%%%%%%%%%%%%%%%%%%%%%%%%%%%%%%%%%%%%%%%%%%%%%%%%%%%%%%%

\end{document}